\documentclass{article}
\usepackage{ijcai26}

\usepackage{times}
\usepackage{soul}
\usepackage{url}
\usepackage[hidelinks]{hyperref}
\usepackage[utf8]{inputenc}
\usepackage[small]{caption}
\usepackage{graphicx}
\usepackage{amsmath}
\usepackage{amsthm}
\usepackage{booktabs}
\usepackage{algorithm}
\usepackage{algorithmic}
\usepackage[switch]{lineno}

\usepackage{algorithm}
\usepackage{algorithmic}
\usepackage{threeparttable}
\usepackage{tablefootnote}
\usepackage{amssymb}
\usepackage{amsthm} 
\usepackage{amsmath}
\usepackage{bm}
\usepackage{multirow}
\usepackage{xcolor} 
\definecolor{royalblue}{RGB}{65,105,225}
  
\newtheorem{theorem}{Theorem}
\newtheorem{corollary}{Corollary}
\newtheorem{lemma}{Lemma}

\newtheorem{definition}{Definition}
\newtheorem{remark}{Remark}
\newtheorem{assumption}{Assumption}
\newcommand{\A}{\mathcal{A}}

\newcommand{\mw}{\mathbf{w}}
\newcommand{\mv}{\mathbf{v}}

\newcommand{\ls}{\left\|}
\newcommand{\rs}{\right\|_2}
\newcommand{\lp}{\left(}
\newcommand{\rp}{\right)}

\newcommand{\lk}{\left[}
\newcommand{\rk}{\right]}

\newcommand{\citep}[1]{\citeauthor{#1}~[\citeyear{#1}]}

\definecolor{darkgrey}{rgb}{0.53,0.53,0.53}
\definecolor{middlegrey}{rgb}{0.75,0.75,0.75}
\definecolor{mygrey}{rgb}{0.9,0.9,0.9}
\definecolor{mydarkblue}{rgb}{0,0.08,0.45}
\definecolor{darkdarkblue}{rgb}{0.0,0.0,0.3}
\definecolor{darkblue}{rgb}{0.0,0.0,0.7}
\definecolor{darkred}{rgb}{0.4,0,0.3}
\definecolor{lightblue}{HTML}{F9FEFE}
\definecolor{verylightpurple}{HTML}{FBFAFF}
\definecolor{lightred}{HTML}{FFFAFA}
\definecolor{fancyblue}{HTML}{4771E3}
\definecolor{grey}{rgb}{0.95,0.95,0.95}
\definecolor{myred}{HTML}{7A1410}
\definecolor{lightpurple}{HTML}{ECE5F3}
\definecolor{myorange}{HTML}{FFDD81}
\definecolor{mypink}{HTML}{ffaec9}
\usepackage[table]{xcolor}
\usepackage{tcolorbox}
\usepackage{cleveref}

\title{Stability and Generalization for Decentralized Markov SGD}

\author{
	Jiahuan Wang$^1$
	\and
	Ziqing Wen$^1$
	\and
	Ping Luo$^1$
	\and
	Dongsheng Li$^1$
	\and
	Tao Sun$^1$\thanks{Corresponding author.} \\
	\affiliations
	$^1$National Key Laboratory of Parallel and Distributed Computing, \\
	National University of Defense Technology, Changsha, 410073, China \\
	\emails
	\{wangjiahuan, zqwen, luoping, dsli\}@ nudt.edu.cn, suntao.saltfish@outlook.com
}

\begin{document}
	
	\maketitle
	\begin{abstract}
		Stochastic gradient methods are central to large-scale learning, yet their generalization theory typically relies on independent sampling assumptions. In many practical applications, data are generated by Markov chains and learning is performed in a decentralized manner, which introduces significant analytical challenges.
		In this work, we investigate the stability and generalization of decentralized stochastic gradient descent (SGD) and stochastic gradient descent ascent (SGDA) under Markov chain sampling. Leveraging a stability-based framework, we characterize how Markovian dependence and decentralized communication jointly influence generalization behavior. Our analysis captures the effects of network topology, Markov chain mixing properties, and primal–dual dynamics. We establish non-asymptotic generalization bounds for both algorithms, extending existing results on Markov stochastic gradient methods to decentralized and minimax settings.
	\end{abstract}
	\section{Introduction}
	Stochastic gradient methods form the backbone of modern large-scale machine learning, owing to their simplicity and scalability \cite{li2014efficient,lin2018don}. In many practical applications, however, data samples are not independently drawn at each iteration. Instead, they are generated sequentially by a stochastic process, such as a Markov chain, which introduces temporal dependence into the training procedure. This phenomenon naturally arises in reinforcement learning \cite{wai2018multi,doan2020convergence}, recommendation systems \cite{li2010contextual}, and distributed data collection \cite{lopes2007incremental,duchi2011dual,chen2014dictionary,madry2017towards,shah2018linearly}, where i.i.d. sampling is either infeasible or prohibitively expensive.
	
	In parallel, the increasing scale of data has driven the adoption of decentralized optimization architectures, where multiple workers collaboratively train a shared model using local computations and limited communication \cite{nedic2009distributed,sundhar2010distributed}. Decentralized stochastic gradient descent (D-SGD) and its variants \cite{lian2017can,lian2018asynchronous} have attracted significant attention due to their robustness, communication efficiency, and suitability for large networks. While the optimization properties of decentralized algorithms are now relatively well understood \cite{shi2015extra,yuan2016convergence,lian2017can,koloskova2020unified,sun2021stability,yuan2023removing}, their statistical behavior—particularly generalization—remains less explored under realistic sampling assumptions.
	
	Recent progress has begun to address these challenges from different perspectives. Stability-based analyses have established rigorous generalization guarantees for stochastic gradient methods under Markovian sampling, highlighting the role of algorithmic stability in controlling the discrepancy between empirical and population risks \cite{wang-markov-2022}. Separately, decentralized SGD with Markov chain sampling has been studied primarily from an optimization viewpoint \cite{sun2022adaptive,sun2023decentralized,koloskova2020unified}, focusing on convergence rates and consensus errors. More recently, stability and generalization properties of decentralized stochastic gradient descent (ascent) (SGDA) algorithms have been investigated under standard i.i.d. sampling assumptions \cite{sun2021stability,zhu2022topology,bars2023improved,zhu2024stability,zeng2025stability}. Despite these advances, a unified theoretical understanding of decentralized stochastic gradient methods under Markovian data dependence is still lacking. In particular, existing works either (i) analyze Markov chain sampling in centralized settings, (ii) study decentralized algorithms assuming independent samples, or (iii) focus on optimization performance without addressing generalization. As a result, fundamental questions remain open:
	
	\begin{tcolorbox}[colback=cyan!05!white, colframe=pink!40!black]
		\centering
		\textit{Can one obtain sharp generalization guarantees for decentralized SGD and SGDA when data are generated by Markov chains?}
	\end{tcolorbox}
	Motivated by these gaps, this paper investigates the theoretical foundations of decentralized stochastic gradient descent and descent–ascent algorithms with Markov chain sampling, and aims to provide a rigorous characterization of their generalization behavior. Building on the stability framework, we develop novel stability bounds for both decentralized SGD and SGDA under Markovian data dependence. Our analysis explicitly captures the interplay between algorithmic stability, network consensus, and the mixing properties of the underlying Markov chains.  We generalize stability-based generalization theory for Markov chain stochastic gradient methods to decentralized settings, and complement prior studies on decentralized optimization by providing statistically meaningful guarantees. To the best of our knowledge, this is the first work to systematically study the stability and generalization of decentralized SGD and SGDA under Markovian sampling.
	\subsection{Differences and Technical Challenges}
	Unlike existing analyses of decentralized SGD under i.i.d. sampling \cite{richards2020graph,sun2021stability,zhu2022topology,bars2023improved,zeng2025stability}, this work allows each worker to access data through a Markov chain, which introduces temporal dependence and breaks the sample-wise independence typically used in stability arguments.
	Compared with centralized Mc-SGD \cite{wang-markov-2022}, decentralization further introduces consensus dynamics, leading to an additional error accumulation through network disagreement.
	A key technical challenge is to control the interaction between these two effects without imposing mixing-time or spectral assumptions on the Markov chain.
	This is resolved by exploiting aggregation identities that depend only on the update structure, rather than on the specific sampling mechanism.
	As a result, i.i.d.-type stability bounds can still be recovered in the decentralized Markovian setting.

	\subsection{Contributions}
	Our main contributions can be summarized as follows.
	
	\begin{itemize}
		\item \textbf{Stability analysis of decentralized SGD with Markovian sampling.} We establish on-average stability bounds for decentralized SGD under Markov chain sampling (See Theorem 1). 
		The resulting bounds match those of decentralized SGD under i.i.d. sampling up to the same consensus-dependent terms, showing that Markovian sampling does not incur additional stability degradation in the decentralized setting.
		
		\item \textbf{Excess risk and generalization guarantees.}
		Based on the stability analysis, we derive excess risk and generalization bounds for Decentralized Markov SGD (DMc-SGD) under both smooth and non-smooth losses (See Theorem 3-4). 
		Our results explicitly characterize the roles of network connectivity, Markov chain mixing, and stepsize selection, and recover known centralized and i.i.d.-based decentralized rates as special cases.
		
		\item \textbf{Stability and generalization of decentralized SGDA with Markovian sampling.}
		We extend the stability-based framework to decentralized stochastic gradient descent ascent (DMc-SGDA) for convex--concave minimax problems. 
		We obtain on-average argument stability bounds and corresponding generalization guarantees for both weak primal--dual risk and primal population risk (See Theorem 6-7), covering smooth setting.

	\end{itemize}
	\begin{table*}[h]
		\begin{center}
			\renewcommand\arraystretch{1.3}
			\begin{threeparttable}
				\begin{tabular}{c|cccc}
					\hline
					Update Type&Data sampling&Reference&Situation&Bounds\\
					\hline
					\multirow{2}{*}{GtC \eqref{I}}
					&{i.i.d}&\cite{richards2020graph}&$\eta\leq2/\beta$&$\mathcal{O}\lp1/\sqrt{mn}\rp$\\
					&Markov chain& Ours (Theorem 8)&$\eta\leq2/\beta$&$\mathcal{O}\lp1/\sqrt{mn}\rp$\\
					\hline
					\multirow{3}{*}{CtG \eqref{II}}&\multirow{2}{*}{i.i.d}&\cite{sun2021stability}&$\eta\leq2/\beta$&$\mathcal{O}\lp1/\sqrt{mn}\rp$\\
					&&\cite{bars2023improved}&$\eta\leq2\min_{k}P_{kk}/\beta$&$\mathcal{O}\lp1/\sqrt{mn}\rp$\\
					&Markov chain&Ours (Theorem 2)&$\eta\leq2/\beta$&$\mathcal{O}\lp1/\sqrt{mn}\rp$\\
					\hline
				\end{tabular}
			\end{threeparttable}
			\caption{Generalization Bounds for D-SGD Problem. (GtC: Gradient-then-Consensus, CtG: Consensus-then-Gradient. $P_{kk}$: the $k$-th diagonal element of the gossip matrix; $\eta$: stepsize; $\beta$: smoothness property.)}
			\label{Compare 0}
		\end{center}
	\end{table*}
	\section{Related Work}
	\textbf{Generalization Analysis of D-SGD.}
	Recent studies have investigated the generalization behavior of decentralized learning primarily through stability-based analyses, which relate the generalization gap to algorithmic sensitivity and network-induced perturbations.
	These works reveal how communication topology and consensus dynamics influence generalization, though the resulting bounds often grow with the number of iterations unless additional structural assumptions, such as strong convexity, are imposed. Representative results include stability- and complexity-based bounds for D-SGD in both smooth and non-smooth settings~\cite{richards2020graph,sun2021stability}, analyses of topology-dependent effects under on-average stability frameworks~\cite{zhu2022topology}, and improvements that recover rates comparable to centralized SGD~\cite{bars2023improved}.
	Extensions to heterogeneous data distributions and weaker regularity conditions have also been considered~\cite{ye2025generalization,zeng2025stability}.
	A detailed comparison of these results is provided in Table~1. Related generalization analyses have further been developed for several D-SGD variants, including asynchronous implementations~\cite{deng2023stability}, minibatch methods~\cite{wang2024towards}, decentralized SGDA~\cite{zhu2024stability}, and zeroth-order optimization schemes~\cite{wang2024towards,hu2025stability}.

	\noindent\textbf{Markov Chain Gradient Descent.} Beyond stability and generalization, Markov chain stochastic gradient methods have been extensively investigated from an optimization perspective, motivated by scenarios where independent sampling is infeasible and data are generated sequentially \cite{ram2009incremental,tadic2011asymptotic,duchi2012ergodic}.
	A central theme in this line of work is to understand how the temporal dependence induced by Markovian sampling affects convergence rates. Convergence properties of SGD with Markovian sampling for convex and non-convex problems were investigated in \citep{sun2021stability}.
	The same work also developed convergence guarantees for non-convex problems under Markovian sampling.
	Decentralized SGD with gradients sampled from non-reversible Markov chains was further studied in \citep{sun2023decentralized}, which characterized the impact of non-reversibility on convergence behavior. Acceleration techniques have also been considered in this setting. \citep{doan2020convergence} analyzed accelerated ergodic Markov chain SGD for both convex and non-convex objectives, while \citep{doan2022finite} further relaxed standard assumptions by deriving convergence rates without requiring bounded gradients.

	\section{Preliminaries}
	We study a decentralized learning setting with $m$ computing nodes, where each node stores $n$ training samples. The full dataset is denoted by $S=\{S_1,S_2,\cdots,S_m\}$, and each local dataset $S_r=\{Z_{1(1)},\cdots,Z_{k(r)},\cdots,Z_{n(m)}\}$ consists of $n$ samples drawn from an unknown distribution $\mathcal{D}$.
	A learning algorithm $\A$, such as (decentralized) SGD, maps the dataset $S$ to a model parameter in a hypothesis space $\mathcal{W}$. The performance of a model $\mw\in\mathcal{W}$ is evaluated through a loss function $f\lp\mw;Z\rp$, which induces the population risk (expected risk)
	$
		\min_{w\in\mathcal{W}} R(\mw):=\mathbb{E}_{Z\sim \mathcal{D}}\lk f\lp\mw;Z\rp\rk.
	$
	Since the data-generating distribution is inaccessible, the algorithm instead minimizes its empirical analogue computed over the observed samples. In the decentralized setting, this empirical objective can be written as
	\begin{align}
		R_{S}(\mw)=\frac{1}{m}\sum_{r=1}^{m}R_{S_{r}}=\frac{1}{mn}\sum_{r=1}^{m}\sum_{k=1}^{n}f\lp\mw;Z_{k(r)}\rp,
	\end{align}
	with $R_{S_{r}}$ denoting the corresponding local empirical risk at node/machine $r$.
	
	Although $\mathcal{A}(S)$ may fit the training data well, good empirical performance does not necessarily translate to good population performance. This motivates the study of the expected gap between the population and empirical risks evaluated at the algorithm output, namely
	\begin{equation}\label{gen}
		\epsilon_{\textrm{gen}}:=\mathbb{E}_{S,\A}\lk R(\mathcal{A}(S))-R_{S}(\mathcal{A}(S))\rk.  
	\end{equation}
	which we refer to as the generalization error. Beyond generalization, we are also interested in how far the learned model is from the optimal population solution.
	\begin{definition}
		The excess generalization error of the learned model, $\epsilon_{\textrm{excess}}:=\mathbb{E}_{S,\A}\lk R(\A(S))-R(\mw^*)\rk$, where $\mw^*$ is denote the minimizer of $R(\cdot)$. This term can then be expressed as
		\begin{align*}
			&\mathbb{E}_{S,\A}\lk R(\A(S))-R(\mw^*)\rk
			=\underbrace{\mathbb{E}_{S,\A}\lk R(\A(S))-R_S(\A(S))\rk}_{\emph{Generalization\quad Error}}\nonumber\\
			&+\underbrace{\mathbb{E}_{S,\A}\lk R_S(\A(S))-R_S({\mw}_{S}^{*})\rk}_{\epsilon_{\textrm{opt}}:~\emph{Optimization\quad Error}}
			+\underbrace{\mathbb{E}_{S,\A}\lk R_S({\mw_{S}^{*}})-R(\mw^*)\rk}_{\emph{Test \quad Error}}\label{exc},
		\end{align*}
		where $\mw^{\star}$ be the minimizer of $R_S(\cdot)$. 
	\end{definition}
	\begin{remark}
		Since $\mathbb{E}\lk R_S\lp \mw^*\rp\rk=\mathbb{E}\lk R\lp\mw^*\rp\rk$ and the empirical risk minimizer satisfies $R_S(\mw_{S}^{*})\leq R_S(\mw^*)$,  the last item is guaranteed to be non-positive. Consequently, most theoretical analyses can be reduced to concentrating on the first two terms, the generalization error and the optimization error.
	\end{remark}
	We now turn to decentralized minimax learning problems, which naturally arise in adversarial learning and robust optimization \cite{goodfellow2014generative,madry2017towards}. In this setting, the objective depends on a pair of variables: a primal variable $\mw\in\mathcal{W}\subset\mathbb{R}^d$ and a dual variable $\mv\in\mathcal{V}\subset\mathbb{R}^d$. Given a loss function $f : \mathcal{W}\times\mathcal{V}\times\mathcal{Z} \rightarrow [0,\infty),$ the population-level minimax objective is defined as
	\begin{align}
		\min_{\mw\in\mathcal{W}}\max_{\mv\in\mathcal{V}} R(\mw,\mv),~R(\mw,\mv):= \mathbb{E}_{z\sim \mathcal{D}}\lk f\lp\mw,\mv;Z\rp\rk.
	\end{align}
	As before, the underlying distribution $\mathcal{D}$ is unknown, and learning algorithms operate on a finite dataset distributed across $m$ workers. The corresponding empirical objective is therefore given by
	\begin{align}
		R_{S}(\mw,\mv)=\frac{1}{mn}\sum_{r=1}^{m}\sum_{k=1}^{n}f\lp\mw,\mv;Z    _{k(r)}\rp.
	\end{align}
	
	In contrast to standard minimization, minimax learning involves two coupled variables, which makes generalization more nuanced. To streamline notation, we write the (possibly randomized) output of algorithm $\mathcal{A}$ as $\mathcal{A}(S)= (\mathcal{A}_{\mw}(S), \mathcal{A}_{\mv}(S))$. We focus on two notions of performance that will be used throughout the paper \cite{farnia2021train,lei2021stability,zhu2024stability}.
	
	\begin{definition}[Weak Primal-Dual (PD) Risk]
		The weak Primal-Dual population risk $\Delta^{\mw}\left(\mathcal{A}_{\mathbf{w}}, \mathcal{A}_{\mathbf{v}}\right)$:
		\begin{align*}
			\max_{\mv \in \mathcal{V}} \mathbb{E}\left[R\left(\mathcal{A}_{\mathbf{w}}({S}), \mathbf{v}\right)\right]-\min _{\mathbf{w} \in \mathcal{W}} \mathbb{E}\left[R\left(\mathbf{w}, \mathcal{A}_{\mathbf{v}}({S})\right)\right].
		\end{align*}
		The weak PD empirical risk $\Delta_{\mathrm{emp}}^{\mw}\left(\mathcal{A}_{\mw}, \mathcal{A}_{\mv}\right)$:
		\begin{align*}
			\max_{\mv \in \mathcal{V}} \mathbb{E}\left[R_{{S}}\left(\mathcal{A}_{\mathbf{w}}(\mathcal{S}), \mathbf{v}\right)\right]-\min _{\mathbf{w} \in \mathcal{W}} \mathbb{E}\left[R_{{S}}\left(\mathbf{w}, \mathcal{A}_{\mathbf{v}}({S})\right)\right].
		\end{align*} 
		The weak PD generalization error of the model:
		\begin{align*}
			\epsilon_{\mathrm{gen}}^{\mw}=\Delta^{\mw}\left(\mathcal{A}_{\mathbf{w}}, \mathcal{A}_{\mathbf{v}}\right)-\Delta_{\mathrm{emp}}^{\mw}\left(\mathcal{A}_{\mw},\mathcal{A}_{\mv}\right).
		\end{align*}
	\end{definition}
	
	\begin{definition}[Primal Risk]
		The primal population and empirical risks of $\mathcal{A}(\mathcal{S})$:
		\begin{align*}
			&F\left(\mathcal{A}_{\mathbf{w}}({S})\right)=\max _{\mathbf{v} \in \mathcal{V}} R\left(\mathcal{A}_{\mathbf{w}}(\mathcal{S}), \mathbf{v}\right),\\
			&F_{\mathcal{S}}\left(\mathcal{A}_{\mathbf{w}}({S})\right)=\max _{\mathbf{v} \in \mathcal{V}} R_{{S}}\left(\mathcal{A}_{\mathbf{w}}(\mathcal{S}), \mathbf{v}\right).
		\end{align*}
		The primal generalization error is defined as
		\begin{align*}
			\epsilon_{\mathrm{gen}}^{\mathrm{P}}=\mathbb{E}_{S,\A}\lk F\left(\mathcal{A}_{\mathbf{w}}({S})\right)-F_S\left(\mathcal{A}_{\mathbf{w}}({S})\right)\rk.
		\end{align*}
		The excess primal population risk of the model:
		\begin{align*}
			\epsilon_{\mathrm{excess}}^{\mathrm{P}}=\mathbb{E}_{S,\A}\lk F\left(\mathcal{A}_{\mathbf{w}}({S})\right)-\min_{\mathbf{w}\in \mathcal{W}} F(\mathbf{w})\rk.
		\end{align*}
	\end{definition}

	\subsection{Decentralized SGD with Markov Sampling} 
	We consider a decentralized optimization framework based on stochastic gradient descent, originally proposed in \citep{lian2017can}, and adopt its projected variant to handle constrained parameter domains. In this setting, a network of 
	m nodes cooperatively solves a learning problem while maintaining local model copies. Communication between nodes is governed by a weighted graph, through which neighboring nodes exchange information at each iteration. At a high level, each iteration of the algorithm consists of two stages. First, nodes perform a consensus operation by aggregating model information from their neighbors. Subsequently, each node updates its local model using a stochastic gradient computed from a data sample generated by a local Markov chain, followed by a projection step to ensure feasibility. This procedure naturally extends decentralized SGD to scenarios where data are not independently sampled.
	
	The complete procedure, referred to as Decentralized Markov Stochastic Gradient Descent (Ascent), is summarized in Algorithm~\ref{D-SGD}. For each node $i$, the algorithm initializes local primal (and dual) variables and iteratively updates them using a combination of communication and stochastic gradient steps. After 
	$T$ iterations, the network outputs the averaged primal (and dual) iterates.
	\begin{algorithm}
		\caption{Decentralized Markov Stochastic Gradient Descent (Ascent)}
		\label{D-SGD}
		\begin{algorithmic}[1] 
			\REQUIRE Initialize $\forall i, \mathbf{w}^{0}(i)=\mathbf{w}^{0}, (\mathbf{v}^{0}(i)=\mathbf{v}^{0})$, stepsizes $\{\eta_t\}_{t=1}^{T}$,weight matrix $P$ and the iteration number $T$. 
			\FOR {$ t=1,2,\cdots,T$ }
			\FOR {$ i=1,2,\cdots,m$}
			\STATE {Sample $Z_{j_t(i)}$ via a Markov chain}
			\STATE{$\mathbf{w}_{t+\frac{1}{2}}^{(i)}=\sum_{l=1}^{m}P_{il}\mathbf{w}_{t}^{(l)},~(\mathbf{v}_{t+\frac{1}{2}}^{(i)}=\sum_{l=1}^{m}P_{il}\mathbf{v}_{t}^{(l)})$}  
			\STATE {$\mathbf{w}_{t+{1}}^{(i)}=\mathbf{P}_{\mathcal{W}}\lp\mathbf{w}_{t+\frac{1}{2}}^{(i)}-{\eta_t}\nabla_{\mw} f(\mathbf{w}_{t}^{(i)};Z_{j_{t}(i)})\rp$ }
			\STATE {($\mathbf{v}_{t+{1}}^{(i)}=\mathbf{P}_{\mathcal{V}}\lp\mathbf{v}_{t+\frac{1}{2}}^{(i)}+{\eta_t}\nabla_{\mv} f(\mathbf{v}_{t}^{(i)};Z_{j_{t}(i)})\rp$) }
			\ENDFOR
			\ENDFOR
			\ENSURE $\bar{\mw}_{T+1}=\frac{1}{m}\sum_{i=1}^{m}\mw_{T+1}^{(i)},~\bar{\mv}_{T+1}=\frac{1}{m}\sum_{i=1}^{m}\mv_{T+1}^{(i)}$
		\end{algorithmic} 
	\end{algorithm}
	The communication structure of the decentralized network is encoded by a matrix $P\in R^{m\times m}$
	, which specifies how information is mixed across nodes. The matrix $P$ is assumed to satisfy the following properties: (1) $P$ is symmetric; (2) all entries satisfy $P_{ij}\in[0,1]$; (3)  $P$ is doubly stochastic, i.e., $\mathbf{1}_m^TP=\mathbf{1}_m^T$ and $P\mathbf{1}_m=\mathbf{1}_m$.
		
	\noindent\textbf{Gossip Matrix.} Let $\lambda_1\geq\lambda_2\geq\cdots\geq\lambda_m$ denote the eigenvalues of $P$, and define $\gamma:=1-\lambda=1-\max\{|\lambda_2(P)|,|\lambda_m(P)|\}$ \cite{sun2021stability,zhu2022topology,deng2023stability,bars2023improved,wang2024towards,zeng2025stability}. The quantity $\lambda\in[0,1)$ characterizes the speed at which consensus is achieved across the network. In particular, smaller values of $\lambda$ correspond to faster information mixing. A fully connected network yields $\lambda=0$, in which case $P$ reduces to the uniform averaging matrix.
 	
	\noindent\textbf{Update Order.} Following existing decentralized SGD schemes \cite {lian2017can}, the update at each node can be organized according to two distinct patterns. In the Gradient-then-Consensus (GtC) variant, each node first performs a local stochastic gradient step and subsequently mixes the updated parameters with its neighbors:
	\begin{align}\label{I}
		\mathbf{w}^{t+{1}}{(i)}=\sum_{l=1}^{m}P_{il}\lk\mathbf{w}^{t}{(l)}-\eta_t\nabla f(\mathbf{w}^{t}{(l)};Z_{j_{t}(l)})\rk.
	\end{align}
	In contrast, the Consensus-then-Gradient (CtG) variant reverses this order. Nodes first aggregate information through the gossip matrix and then apply a stochastic gradient update using locally sampled data:
	\begin{align}
		\mathbf{w}^{t+{1}}{(i)}=\sum_{l=1}^{m}P_{il}\mathbf{w}^{t}{(l)}-\eta_t\nabla f(\mathbf{w}^{t}{(i)};Z_{j_{t}(i)}).\label{II}
	\end{align}
	From an optimization viewpoint, GtC and CtG exhibit similar convergence behavior, enabling communication and computation to be overlapped in practice. Their difference becomes pronounced in stability and generalization analysis: the CtG update relies solely on local gradients, which magnifies the effects of data heterogeneity and network disagreement. Consequently, establishing stability guarantees for CtG typically requires more refined analysis or stronger structural assumptions \cite{sun2021stability,deng2023stability,zeng2025stability,le2023refined}.
	
	\noindent\textbf{Markov Chain Sampling.} We now summarize the basic concepts of finite-state, time-homogeneous Markov chains that are relevant to our analysis. These notions formalize the dependence structure induced by Markovian data sampling.
	\begin{definition}
		A stochastic process $\{X_k\}_{k\geq1}$ taking values in a finite state space $\{1,2, \ldots, n\}$ is called a time-homogeneous Markov chain with transition matrix $H \in \mathbb{R}^{n \times n}$ if  
		$\Pr\left(X_{k+1}=j \mid X_0, \ldots, X_k=i\right)=\Pr\left(X_{k+1}=j \mid X_k=i\right)={H}_{i,j},$
		for all states $i, j \in \{1,2, \ldots, n\}$ and $k\geq0$.
	\end{definition}
	\noindent Let $\pi^k$ denote the distribution of $X_k$, viewed as a row vector.
	Then  $\pi^k=\pi^{k-1}{H}=\pi^0 {H}^k$, where ${H}^k$ denotes the $k$-step transition matrix. A Markov chain is said to be irreducible if every state can be reached from any other state in a finite number of steps. A state is aperiodic if it does not exhibit deterministic cyclic behavior, and the chain is aperiodic if all states are aperiodic. Under irreducibility and aperiodicity, the chain admits a unique stationary distribution $\pi^*$ satisfying $\pi^*=\pi^*{H}$ and $\min _i \pi_i^*>0$. Moreover, the transition matrix converges as
	$
		\lim_{k\rightarrow\infty} {H}^{k}=\left[\left(\pi^*\right)^{\top}, \cdots,\left(\pi^*\right)^{\top}\right]^{\top}=: \Pi^* .
	$
	
	\noindent\textbf{Mixing Behavior.} An important quantity governing the statistical behavior of Markov chain sampling is the mixing time, which measures how quickly the distribution of the chain approaches its stationary distribution. Mixing properties provide explicit bounds on the deviation between $\pi^k$ and $\pi^*$a s a function of $k$, and play a central role in controlling the bias introduced by Markovian dependence. Following \cite{sun-markov-18,sun2023decentralized}, we will characterize this deviation through matrix-based bounds that are well suited for stability analysis.	
	\begin{table*}[h]
		\begin{center}
			\renewcommand\arraystretch{1.35}
			\begin{threeparttable}
				\begin{tabular}{cc|ccc}
					\hline
					Type&Algorithm&Reference&Case&Generalization Bounds\\
					\hline
					\multirow{4}{*}{Centralized}&\multirow{2}{*}{SGD}&\cite{hardt2016train}&Smooth&${\eta T}/{n}$\\
					&&\cite{bassily2020non}&Non-smooth&$\sqrt{T}\eta+{\eta T}/{n}$\\ 
					\cline{3-5}
					&\multirow{2}{*}{Mc-SGD}&\multirow{2}{*}{\cite{wang-markov-2022}}&Smooth&${\eta T}/{n}$\\
					&&&Non-smooth&$\sqrt{T}\eta+{\eta T}/{n}$\\
					\hline
					\multirow{4}{*}{Decentralized}&\multirow{2}{*}{D-SGD}&\cite{sun2021stability}&Smooth&${\eta T}/{mn}+{\eta T}/{(1-\lambda)}$\\
					&&\cite{zeng2025stability}&Non-smooth&${\sqrt{T}\eta}/{(1-\lambda^2)}+{\eta T}/{mn}$\\
					\cline{3-5}
					&\multirow{2}{*}{DMc-SGD}&\multirow{2}{*}{Ours (Theorem 2)}&Smooth&${\eta T}/{mn}+{\eta T}/{(1-\lambda)}$\\
					&&&Non-smooth&${\sqrt{T}\eta}/{(\sqrt{1-\lambda})}+{\eta T}/{mn}$\\
					\hline
				\end{tabular}
			\end{threeparttable}
			\caption{Generalization Bounds for Convex Problem. ($T$: Iterate numbers, $\eta$: Stepsize, $n$: Number of (local) dataset, $m$: Number of worker, $\lambda$: Consensus rate)}
			\label{Compare 1}
		\end{center}
	\end{table*}
	
	\section{Main Results}
	This section presents our main theoretical results and the assumptions required for our analysis.
	\begin{assumption}[$L$-Lipschitz]
		\label{Lipschitz}
		For any $Z\sim\mathcal{D}$ and $\mw,\tilde{\mw}\in\mathcal{W}$, the loss function $f(\mw;Z)$ satisifies
		\begin{align}
			|f(\mw;Z)-f(\tilde{\mw};Z)| \leq L\|\mw-\tilde{\mw}\|_2.
		\end{align}
	\end{assumption}
	\begin{remark}
		This condition ensures that perturbations in the model parameters induce proportionally bounded changes in the loss. In particular, whenever the gradient exists, its norm is uniformly bounded by $L$.
	\end{remark}
	\begin{assumption}[$\beta$-Smoothness]
		\label{Smooth}
		For any $Z\sim\mathcal{D}$ and $\mw,\tilde{\mw}\in\mathcal{W}$, the loss function has Lipschitz-continuous gradient:
		\begin{align}
			\|\nabla f(\mw;Z)-\nabla f(\tilde{\mw};Z)\|_2 \leq \beta\|\mw-\tilde{\mw}\|_2.
		\end{align}
	\end{assumption}
	\begin{remark}
		Smoothness guarantees regular variation of the gradient field and plays a key role in controlling the propagation of perturbations along the optimization trajectory. We next formalize the assumptions on the sampling mechanism employed at each worker.
	\end{remark}
	\begin{assumption}
		Each worker accesses data through a finite-state, time-homogeneous Markov chain that is irreducible and aperiodic. All workers are assumed to employ Markov chains with a common transition matrix $H$ and the same sationary distrubtion.\footnote{This assumption is introduced primarily to simplify exposition. The analysis can be extended to heterogeneous Markov chains across workers with additional technical effort.} 
	\end{assumption}
	\begin{remark}
		Such Markovian sampling schemes have been widely studied in the literature \cite{sun-markov-18,mao2020walkman,doan2020convergence,sun2023decentralized}. For example, Markov chain–based SGD has been proposed for pairwise learning problems, including AUC maximization, bipartite ranking, and metric learning \cite{lei2020sharper,lei2021generalization,yang2021simple}. In these settings, model updates are driven by data pairs generated via an auxiliary Markov chain, which satisfies irreducibility and aperiodicity under mild conditions. Similar assumptions also arise in decentralized consensus optimization over multi-agent networks, where the Markov chain state space coincides with the finite set of agents and a common transition matrix is shared across nodes.
	\end{remark}
	\begin{assumption}[Reversibility]
		The Markov transition matrix satisfies $H=H^T$.
	\end{assumption}
	\noindent Reversibility enables a clean characterization of mixing behavior and will be used to derive explicit stability bounds.
	\subsection{Generalization Analysis for DMc-SGD}
	We adopt an on-average notion of argument stability \cite{lei2020fine,zhu2022topology,wang2024towards,zeng2025stability}, which measures the expected deviation between algorithm outputs when a single training example is replaced.
	\begin{definition}\label{on-average}
		(On-Average Argument Stability)
		Let $S=(S_1,\cdots,S_m)$ and $\tilde{S}=(\tilde{S}_1,\cdots,\tilde{S}_m)$ be two independent datasets drawn from $\mathcal{D}$. 
		Let $S_{rk}$ denotes denote the dataset obtained by replacing the $k$-th sample at the $r$-th worker in $S$ with the corresponding sample from $\tilde{S}$. A (possibly randomized) algorithm $\mathcal{A}$ is said to be $\ell_1$ on-average argument $\epsilon$-stable if 
		\begin{equation}\label{argument1}
			\frac{1}{mn}\sum_{r=1}^{m}\sum_{k=1}^{n}\mathbb{E}\left[\left\|\A(S)-\A(S_{rk})\right\|_2\right]\leq\epsilon.
		\end{equation}
	\end{definition}
	\begin{lemma}
		(Generalization via On-Average Stability. \cite{lei2020fine}). 
		Let $\mathcal{A}$ be $\ell_1$ on-average argument $\epsilon$-stable and suppose Assumption~\ref{Lipschitz} holds. Then,
		\begin{equation*}
			|\mathbb{E}_{S,\mathcal{A}}\left[R(\mathcal{A}(S))-R_S(\mathcal{A}(S))\right]|\leq L\epsilon.
		\end{equation*}
	\end{lemma}
	This result reduces the analysis of generalization error to that of stability. We now present explicit stability bounds for DMc-SGD under convexity.
	\begin{theorem}[Stability Bounds for DMc-SGD]\label{Stability Bounds for DMc-SGD}
		Assume that $f(\mw;Z)$ is convex and L-Lipshitz.
		\begin{itemize}
			\item (Smooth Case) If $f(\mw;Z)$ is $\beta$-smooth and $\eta\leq2/\beta$, then after $T$ iterations of DMc-SGD, 
			\begin{align*}
				\epsilon\leq4\beta L\sum_{t=1}^{T}\eta_t\sum_{q=1}^{t}\eta_q\lambda^{t-q}+\frac{2L}{mn}\sum_{t=1}^{T}\eta_t.
			\end{align*}
			\item (Non-smooth Case) Without smoothness, DMc-SGD remains on-average 
			$\epsilon$-stable with
			\begin{align*}
				\epsilon\leq 2L\sqrt{\sum_{t=1}^{T}\eta_{t}^2}+4L\sqrt{\sum_{t=1}^{T}\eta_{t}\sum_{q=1}^{t}\eta_q\lambda^{t-q}}+\frac{4L}{mn}\sum_{t=1}^{T}\eta_{t}.
			\end{align*}
		\end{itemize}
	\end{theorem}
	\begin{remark}
		Compared with stability bounds for centralized SGD with Markovian sampling \cite{wang-markov-2022}, 
		Theorem~\ref{Stability Bounds for DMc-SGD} contains an additional consensus error term that captures the effect of decentralized communication. 
		This term naturally arises in decentralized optimization and is absent in the centralized setting. More importantly, when compared with decentralized SGD under i.i.d. sampling \cite{sun2021stability,zeng2025stability}, 
		the stability bound obtained here has exactly the \textbf{same form and order}.
		In particular, the result is derived without imposing any assumptions on the underlying Markov chain, such as mixing-time or spectral conditions.
		This shows that Markovian sampling\textbf{ does not lead to worse stability} guarantees than i.i.d. sampling for decentralized SGD.
		
		The key technical ingredient behind this result is the identity
		$
		\sum_{k=1}^{n}\mathbb{I}_{\{j_t = k\}} = 1,
		$
		where $\mathbb{I}_{\{\cdot\}}$ denotes the indicator function.
		Since this property is independent of the specific sampling mechanism, 
		the argument is not tied to Markov chain sampling.
		As a consequence, the same stability analysis \textbf{applies to other non-i.i.d. sampling schemes}, 
		such as particle filtering or SGLD-based sampling. A detailed comparison with existing stability results is summarized in Table~2.
	\end{remark}
	The stability bounds in Theorem 2 explicitly depend on the stepsize sequence, revealing how the choice of learning rate mediates the effect of decentralization. In particular, smaller stepsizes mitigate the amplification of instability caused by imperfect consensus across the network. To make this dependence more transparent, we specialize the general bounds to two commonly adopted stepsize schedules.
	\begin{corollary}[Stability under Smooth Losses]
		Suppose $f(\mw;Z)$ is convex and satisfies both the Lipschitz and smoothness conditions.
		\begin{itemize}
			\item (Constant Stepsize) If the stepsize $\eta_t=\eta\leq2/\beta$, the average stability parameter satisfies
			$
				\epsilon\leq \frac{4\eta^2\beta LT}{1-\lambda}+\frac{2\eta LT}{mn}.
			$
			\item (Decreasing Stepsize) If the stepsize $\eta_t=\frac{1}{t+1}\leq2/\beta$, then
			$
				\epsilon\leq\frac{4\beta LC_{\lambda}T}{T+1}+\frac{2L\ln(T+1)}{mn}.
			$
		\end{itemize}
	\end{corollary}
	The next result addresses the non-smooth setting, where stability exhibits a qualitatively different dependence on the iteration horizon.
	\begin{corollary}[Stability under Non-smooth Losses]
		Suppose $f(\mw;Z)$ is convex and $L$-Lipschitz.
		\begin{itemize}
			\item (Constant Stepsize) If the stepsize $\eta_t=\eta\leq2/\beta$, the average stability we get
			$
				\epsilon\leq2L\eta\sqrt{T}+\frac{4\eta L\sqrt{T}}{\sqrt{1-\lambda}}+\frac{4\eta LT}{mn}.
			$
			\item (Decreasing Stepsize) If the stepsize $\eta_t=\frac{1}{t+1}\leq2/\beta$, the average stability we have
			$
				\epsilon\leq 2L+4L\sqrt{C_{\lambda}}+\frac{2L\ln T}{mn}.
			$   
		\end{itemize}
	\end{corollary}
	The above corollaries characterize the stability of the final iterate produced by DMc-SGD. However, in general convex optimization, performance guarantees are often stated for a weighted average of iterates rather than the last iterate. Following standard practice in decentralized optimization \cite{sun2023decentralized}, we therefore consider the stepsize-weighted average $\bar{\mw}^{T}={\sum_{t=1}^{T}\eta_t\mw^{t}}/{\sum_{t=1}^{T}\eta_t}$.
	The next theorem establishes generalization guarantees for this averaged solution.
	\begin{theorem}[Generalization of Averaged Iterates]
		Suppose $f(\mw;Z)$ is convex, satisfying Lipschitz property.
		
		\noindent(Smooth Case) If $f(\mw;Z)$ is $\beta$-smooth and the stepsize $\eta_t=\eta\leq2/\beta$, the generalization error we get
		\begin{align*}
			\left|\mathbb{E}_{S,\mathcal{A}}\left[R(\bar{\mw}^{T})-R_S(\bar{\mw}^{T})\right]\right|\leq \frac{2\eta^2\beta LT}{1-\lambda}+\frac{\eta LT}{mn}.
		\end{align*}
		(Non-smooth Case) If the stepsize $\eta_t=\eta$, then
		\begin{align*}
			\left|\mathbb{E}_{S,\mathcal{A}}\left[R(\bar{\mw}^{T})-R_S(\bar{\mw}^{T})\right]\right|\leq 2L\eta\sqrt{T}+\frac{4\eta L\sqrt{T}}{\sqrt{1-\lambda}}+\frac{4\eta LT}{mn}.
		\end{align*}
	\end{theorem}
	We are now ready to combine the optimization and generalization results to obtain explicit excess risk bounds for DMc-SGD. The following results characterize how decentralization and Markovian sampling jointly influence the statistical performance of the averaged iterate.
	
	\begin{theorem}[Excess Risk under Smooth Case] Assume that the loss function $f(\mw;Z)$ is convex, L-Lipschitz and $\beta$-Smooth, and that Assumptions 4 holds. 
	Consider DMc-SGD initialized at $\mw^0=0$, and let $\{w^t\}_{t=1}^{T}$ denote the iterates generated with the constant stepsize $\eta_t=\eta\leq2/\beta$. If we select $T=\mathcal{O}\lp mn\rp$ and $\eta=\frac{1}{\sqrt{T\log T}}$, then 
		$$
			\mathbb{E}\lk R(\bar{\mw}^{T})-R(\mw^*)\rk
			=\mathcal{O}\lp\frac{1}{\gamma\log T}+\frac{\sqrt{\log T}}{\sqrt{T}\log(1/\lambda(H))}\rp.
		$$
	\end{theorem}
	The next theorem addresses the non-smooth regime, where the interaction between stepsize selection and network effects leads to a different scaling behavior.
	\begin{theorem}[Excess Risk under Nonsmooth Case] Suppose that the loss function $f(\mw;Z)$ is convex, L-Lipschitz , and that Assumptions 4 holds. Let DMc-SGD run  for $T$ iterations with a constant stepsize $\eta_t=\eta$. If we choose $T=\mathcal{O}\lp\frac{m^2n^2}{1-\lambda}\rp$ and $\eta=\sqrt{1-\lambda}{T^{-3/4}}=\sqrt{\gamma}{T^{-3/4}}$, then
		\begin{align*}
			\mathbb{E}\lk R(\bar{\mw}^{T})-R(\mw^*)\rk
			=\mathcal{O}\lp\frac{\log(mn)}{(mn)^{3/2}(\gamma)^{1/4}\log(1/\lambda(H))}\rp.
		\end{align*}
	\end{theorem}
	\begin{remark}
	In the smooth case, the excess risk consists of two distinct components: a network-dependent term scaling as $1/(\gamma \log T)$, which reflects\textbf{ the effect of imperfect consensus}, and a Markovian sampling term governed by the mixing property of the underlying chain through $\log(1/\lambda(H))$. 
	When $\gamma=1$, the bound reduces to the known excess risk rate of centralized Mc-SGD \cite{wang-markov-2022}.
	
	In the non-smooth regime, the slower rate arises from the combined effect of non-smoothness and decentralized communication, leading to a stronger dependence on both the network \textbf{spectral gap $\gamma$} and the \textbf{sample size $mn$}. 
	Nevertheless, the resulting bound remains consistent with existing excess risk guarantees for Mc-SGD. We note, however, that the existing analysis \cite{wang-markov-2022} reports a $\mathcal{O}(1/\sqrt{n})$ rate in the non-smooth case, whereas a careful derivation \textbf{yields a sharper} $\mathcal{O}(\log n / n^{3/2})$ dependence (up to logarithmic factors involving the Markov chain mixing).
	A detailed discussion of this discrepancy is provided in the appendix F.
	Overall, these results indicate that Markovian sampling does not introduce an additional statistical penalty beyond what is already induced by decentralization and non-smoothness, and that its effect is explicitly captured through the mixing characteristics of the Markov chain.
	\end{remark}
	
	\subsection{Generalization Analysis for DMc-SGDA}
	We now extend the stability-based generalization analysis to decentralized stochastic gradient descent ascent applied to minimax optimization problems. Let $(\bar{\mw}_T,\bar{\mv}_T)$ denote the output of DMc-SGDA after T iterations, where the averaged primal and dual variables are defined a
	\begin{align*}
		\bar{\mw}^{T}=\frac{\sum_{t=1}^{T}\eta_t\mw^{t}}{\sum_{t=1}^{T}\eta_t},~\bar{\mv}^{T}=\frac{\sum_{t=1}^{T}\eta_t\mv^{t}}{\sum_{t=1}^{T}\eta_t}.
	\end{align*}
	Our goal is to characterize how decentralized communication and Markovian sampling affect the generalization behavior of DMc-SGDA in minimax settings. We first impose some tandard definitions and assumptions \cite{lei2021stability,farnia2021train,ozdaglar2022good,zhu2024stability}.
	\begin{definition}
		Let $\rho \geq 0$. A function $f(\mw,\mv): \mathcal{W} \times\mathcal{V} \mapsto \mathbb{R}$ is said to be $\rho$-strongly-convex-strongly-concave ($\rho$-SC-SC) if, for any $\mathbf{v} \in \mathcal{V}$, the mapping $\mathbf{w} \mapsto f(\mathbf{w}, \mathbf{v})$ is $\rho$-strongly-convex and, for any $\mathbf{w} \in \mathcal{W}$, the mapping $\mathbf{v} \mapsto f(\mathbf{w}, \mathbf{v})$ is $\rho$-strongly-concave. The special case $\rho=0$ corresponds to convex–concave objectives.
	\end{definition}
	\begin{assumption}
		For all $\mathbf{w}\in\mathcal{W}, \mathbf{v}\in\mathcal{V}$ and $Z \in \mathcal{Z}$,
		$
			\left\|\nabla_{\mathbf{w}} f(\mathbf{w},\mathbf{v}; Z)\right\|_2\leq L,~\left\|\nabla_{\mathbf{v}} f(\mathbf{w}, \mathbf{v} ; Z)\right\|_2 \leq L.
		$
	\end{assumption}
	\begin{assumption}
		For any $Z$, the mapping $(\mathbf{w}, \mathbf{v}) \mapsto f(\mathbf{w}, \mathbf{v} ; Z)$ is $\beta$-smooth in the sense that, for all $(\mathbf{w}, \mathbf{v})$ and $(\tilde{\mw}, \tilde{\mv})$,
		\begin{align*}
			\left\|\binom{\nabla_{\mathbf{w}} f(\mathbf{w}, \mathbf{v} ; Z)-\nabla{\mathbf{w}} f\left(\tilde{\mw}, \tilde{\mv} ; Z\right)}{\nabla_{\mathbf{v}} f(\mathbf{w}, \mathbf{v} ; Z)-\nabla_{\mathbf{v}} f\left(\tilde{\mw}, \tilde{\mv} ; Z\right)}\right\|_2 \leq \beta\left\|\binom{\mathbf{w}-\tilde{\mw}}{\mathbf{v}-\tilde{\mv}}\right\|_2 .
		\end{align*}
	\end{assumption}
	To quantify the sensitivity of DMc-SGDA, we adopt an on-average notion of argument stability that simultaneously accounts for perturbations in both the primal and dual variables.
	\begin{definition}[Dencentralized Argument Stability for Minmax Problems]
		Let $S$ and $S^{(rk)}$ be constructed as in Definition 5. A randomized algorithm $\mathcal{A}$ is said to be on-average $\epsilon$-argument-stable for minimax problems if
		$
			\frac{1}{mn} \sum_{r=1}^{m}\sum_{k=1}^n \mathbb{E}\Big[\left\|\mathcal{A}_{\mathrm{w}}(S)-\mathcal{A}_{\mathrm{w}}\left(S_{rk}\right)\right\|_2+\left\|\mathcal{A}_{\mathrm{v}}(S)-\mathcal{A}_{\mathrm{v}}\left(S_{rk}\right)\right\|_2\Big]\leq\epsilon.
		$
	\end{definition}
	We will establish explicit stability bounds for DMc-SGDA applied to convex–concave objectives. 
	\begin{theorem}(Stability of DMc-SGDA)
		Consider DMc-SGDA run for $T$ iterations on a convex--concave objective, i.e., for every $Z\in\mathcal{Z}$ the mapping
		$(\mathbf{w},\mathbf{v})\mapsto f(\mathbf{w},\mathbf{v};Z)$ is convex in $\mathbf{w}$ and concave in $\mathbf{v}$.
		Assume $\mathcal{W}=\mathbb{R}^d$ and Assumption~5 holds. Then the algorithm output is on-average
		$\epsilon$-argument stable, where $\epsilon$ can be bounded as follows.
		
		\noindent\textbf{Smooth regime.} If Assumption 6 holds and the stepsizes satisfy $\sum_{t=1}^T \eta_t \leq 1 /\left(2 \beta\right)$, then
		\begin{align*}
			\epsilon \leq 8\sqrt{2}\beta L\sum_{t=1}^{T}\eta_{t}\sum_{q=1}^{t-1}\eta_{q}\lambda^{t-q-1}+\frac{4\sqrt{2}L}{mn}\sum_{t=1}^{T}\eta_{t}.
		\end{align*}	
		\noindent\textbf{Non-smooth regime.} Without Assumption~6, a valid bound is
		\begin{align*}
			\epsilon=\mathcal{O}\lp\sqrt{\sum_{t=1}^{T}\eta_t^2}+\sqrt{\sum_{t=1}^{T}\eta_t\sum_{q=1}^{t-1}\eta_q\lambda^{t-q-1}}+\frac{1}{mn}\sum_{t=1}^{T}\eta_t\rp.
		\end{align*}
	\end{theorem}
	\begin{remark}
		This theorem exhibits the same stability structure as decentralized SGDA under i.i.d. sampling~\cite{zhu2024stability} in the smooth regime, with comparable dependence on the stepsize sequence and the network connectivity.
		Moreover, we provide stability guarantees for the \textbf{non-smooth setting}, which has not been explicitly addressed in prior work. Compared with centralized SGDA or Mc-SGDA \cite{lei2021stability,wang-markov-2022}, the above bounds contain an additional consensus-related term that arises from decentralized communication.
		
	\end{remark}
	The following results characterize the generalization behavior of DMc-SGDA under different performance metrics.
	\begin{theorem}[Generalization Error of DMc-SGDA] 
		Consider DMc-SGDA with a constant stepsize $\eta$ run for $T$ iterations, and define
		$\mathcal{A}_{\mathbf{w}}(S)=\bar{\mathbf{w}}_T$ and
		$\mathcal{A}_{\mathbf{v}}(S)=\bar{\mathbf{v}}_T$.
		Suppose Assumption~5 holds, $\mathcal{W}=\mathbb{R}^d$, and for every $Z$ the function
		$(\mathbf{w},\mathbf{v})\mapsto f(\mathbf{w},\mathbf{v};Z)$ is convex--concave. (In addition, suppose the mapping $\mathbf{v}\mapsto R(\mathbf{w},\mathbf{v})$
		is $\rho$-strongly concave.)
		\begin{itemize}
			\item (Smooth case) If Assumption 6 holds and the stepsizes satisfy $\sum_{t=1}^T \eta_t \leq 1 /\left(2\beta\right)$, then
			\begin{align*}
				\epsilon_{\mathrm{gen}}^{\mw} \leq& \frac{4\sqrt{2}\eta^2\beta L^2T}{1-\lambda}+\frac{2\sqrt{2}\eta L^2T}{mn},\\
				\epsilon_{\mathrm{gen}}^\mathrm{P}\leq&2\sqrt{2}L^2(1+\beta/\rho)\left(\frac{2\eta^2\beta T}{1-\lambda}+\frac{\eta T}{mn}\right).
			\end{align*}
			\item (Non-smooth case) In the absence of smoothness, the same quantity can be bounded as
			\begin{align*}
				\epsilon_{\mathrm{gen}}^{\mw} \leq&2\sqrt{2}L^2\eta\sqrt{T}+\frac{4\eta L^2\sqrt{T}}{\sqrt{1-\lambda}}+\frac{4\sqrt{2}\eta L^2T}{mn},\\
				\epsilon_{\mathrm{gen}}^\mathrm{P}\leq&2\sqrt{2}L^2(1+\beta/\rho)\Bigg(\eta\sqrt{T}+\frac{\eta \sqrt{2T}}{\sqrt{1-\lambda}}+\frac{2\eta T}{mn}\Bigg).
			\end{align*}
		\end{itemize}

	\end{theorem}
	We next characterize the population-level performance of DMc-SGDA in terms of the weak primal–dual (PD) risk. Throughout, let $D_{\mathbf{w}}$ and $D_{\mathbf{v}}$ denote the diameters of the feasible sets $\mathcal{W}$ and $\mathcal{V}$, respectively. 
	\begin{theorem}[Weak PD Population Risk]
		Consider DMc-SGDA run with a constant stepsize $\eta_t=\eta$ for $T$ iterations, and let
		$(\bar{\mathbf{w}}_T,\bar{\mathbf{v}}_T)$ denote the averaged iterates.
		Suppose that Assumptions~3-6 hold, and that for every $Z$ the function
		$(\mathbf{w},\mathbf{v})\mapsto f(\mathbf{w},\mathbf{v};Z)$ is convex--concave. If $T=mn$ and $\eta=(T \log(T))^{-{1}/{2}}$, then
		\begin{align*}
			\Delta^{\mw}\left(\overline{\mathbf{w}}_T, \overline{\mathbf{v}}_T\right)=\mathcal{O}\lp\frac{1}{(1-\lambda)\log T}+\frac{\sqrt{\log T}}{\sqrt{T}\log (1 / \lambda(H))}\rp.
		\end{align*}	
	\end{theorem}

	
	\section{Conclusion}
	
	We investigated the stability and generalization of decentralized stochastic gradient methods with Markovian sampling.
	Our analysis shows that, despite temporal dependence and decentralized communication, i.i.d.-type stability and generalization guarantees can still be obtained under appropriate conditions.
	The results extend naturally to decentralized minimax optimization and highlight the robustness of stability-based analyses beyond independent sampling.
	\section*{Acknowledgements}
	This work is supported in part by the  National Natural Science Foundation of China (Grant Nos. 62522610,  62376278), and NUDT Foundational Research Funding (JS25-02).
	\bibliographystyle{named}
	\bibliography{ijcai26}

@inproceedings{lei2021stability,
	title={Stability and generalization of stochastic gradient methods for minimax problems},
	author={Lei, Yunwen and Yang, Zhenhuan and Yang, Tianbao and Ying, Yiming},
	booktitle={International Conference on Machine Learning (ICML)},
	pages={6175--6186},
	year={2021},
	organization={PMLR}
}

@inproceedings{yang2021simple,
	title={Simple stochastic and online gradient descent algorithms for pairwise learning},
	author={Yang, Zhenhuan and Lei, Yunwen and Wang, Puyu and Yang, Tianbao and Ying, Yiming},
	booktitle={Advances in Neural Information Processing Systems (NeurIPS)},
	volume={34},
	pages={20160--20171},
	year={2021}
}

@misc{ye2025generalization,
	title={Generalization error analysis for Attack-Free and Byzantine-Resilient decentralized learning with data heterogeneity},
	author={Ye, Haoxiang and Sun, Tao and Ling, Qing},
	eprint={2506.09438},
	archivePrefix={arXiv},
	year={2025}
}

@inproceedings{li2010contextual,
	title={A contextual-bandit approach to personalized news article recommendation},
	author={Li, Lihong and Chu, Wei and Langford, John and Schapire, Robert E},
	booktitle={International Conference on World Wide Web (WWW)},
	pages={661--670},
	year={2010}
}

@inproceedings{goodfellow2014generative,
	title={Generative adversarial nets},
	author={Goodfellow, Ian J and Pouget-Abadie, Jean and Mirza, Mehdi and Xu, Bing and Warde-Farley, David and Ozair, Sherjil and Courville, Aaron and Bengio, Yoshua},
	booktitle={Advances in Neural Information Processing Systems},
	volume={27},
	year={2014}
}

@misc{madry2017towards,
	title={Towards deep learning models resistant to adversarial attacks},
	author={Madry, Aleksander and Makelov, Aleksandar and Schmidt, Ludwig and Tsipras, Dimitris and Vladu, Adrian},
	eprint={1706.06083},
	archivePrefix={arXiv},
	year={2017}
}

@misc{doan2020convergence,
	title={Convergence rates of accelerated markov gradient descent with applications in reinforcement learning},
	author={Doan, Thinh T and Nguyen, Lam M and Pham, Nhan H and Romberg, Justin},
	eprint={2002.02873},
	archivePrefix={arXiv},
	year={2020}
}

@article{mao2020walkman,
	title={Walkman: A communication-efficient random-walk algorithm for decentralized optimization},
	author={Mao, Xianghui and Yuan, Kun and Hu, Yubin and Gu, Yuantao and Sayed, Ali H and Yin, Wotao},
	journal={IEEE Transactions on Signal Processing},
	volume={68},
	pages={2513--2528},
	year={2020},
	publisher={IEEE}
}

@inproceedings{wai2018multi,
	title={Multi-agent reinforcement learning via double averaging primal-dual optimization},
	author={Wai, Hoi-To and Yang, Zhuoran and Wang, Zhaoran and Hong, Mingyi},
	booktitle={Advances in Neural Information Processing Systems (NeurIPS)},
	volume={31},
	year={2018}
}

@article{chen2014dictionary,
	title={Dictionary learning over distributed models},
	author={Chen, Jianshu and Towfic, Zaid J and Sayed, Ali H},
	journal={IEEE Transactions on Signal Processing},
	volume={63},
	number={4},
	pages={1001--1016},
	year={2014},
	publisher={IEEE}
}

@inproceedings{sun2022adaptive,
	title={Adaptive random walk gradient descent for decentralized optimization},
	author={Sun, Tao and Li, Dongsheng and Wang, Bao},
	booktitle={International Conference on Machine Learning},
	pages={20790--20809},
	year={2022},
	organization={PMLR}
}

@article{duchi2011dual,
	title={Dual averaging for distributed optimization: Convergence analysis and network scaling},
	author={Duchi, John C and Agarwal, Alekh and Wainwright, Martin J},
	journal={IEEE Transactions on Automatic control},
	volume={57},
	number={3},
	pages={592--606},
	year={2011},
	publisher={IEEE}
}

@article{ram2009incremental,
	title={Incremental stochastic subgradient algorithms for convex optimization},
	author={Ram, S Sundhar and Nedi{\'c}, A and Veeravalli, Venugopal V},
	journal={SIAM Journal on Optimization},
	volume={20},
	number={2},
	pages={691--717},
	year={2009},
	publisher={SIAM}
}

@inproceedings{tadic2011asymptotic,
	title={Asymptotic bias of stochastic gradient search},
	author={Tadi{\'c}, Vladislav B and Doucet, Arnaud},
	booktitle={IEEE Conference on Decision and Control and European Control Conference},
	pages={722--727},
	year={2011},
	organization={IEEE}
}

@article{duchi2012ergodic,
	title={Ergodic mirror descent},
	author={Duchi, John C and Agarwal, Alekh and Johansson, Mikael and Jordan, Michael I},
	journal={SIAM Journal on Optimization},
	volume={22},
	number={4},
	pages={1549--1578},
	year={2012},
	publisher={SIAM}
}

@article{doan2022finite,
	title={Finite-time analysis of markov gradient descent},
	author={Doan, Thinh T},
	journal={IEEE Transactions on Automatic Control},
	volume={68},
	number={4},
	pages={2140--2153},
	year={2022},
	publisher={IEEE}
}

@inproceedings{ozdaglar2022good,
	title={What is a good metric to study generalization of minimax learners?},
	author={Ozdaglar, Asuman and Pattathil, Sarath and Zhang, Jiawei and Zhang, Kaiqing},
	booktitle={Advances in Neural Information Processing Systems (NeuIPS)},
	volume={35},
	pages={38190--38203},
	year={2022}
}

@article{lopes2007incremental,
	title={Incremental adaptive strategies over distributed networks},
	author={Lopes, Cassio G and Sayed, Ali H},
	journal={IEEE Transactions on Signal Processing},
	volume={55},
	number={8},
	pages={4064--4077},
	year={2007},
	publisher={IEEE}
}

@misc{shah2018linearly,
	title={Linearly convergent asynchronous distributed admm via markov sampling},
	author={Shah, Suhail M and Avrachenkov, Konstantin E},
	eprint={1810.05067},
	archivePrefix={arXiv},
	year={2018}
}

@inproceedings{farnia2021train,
	title={Train simultaneously, generalize better: Stability of gradient-based minimax learners},
	author={Farnia, Farzan and Ozdaglar, Asuman},
	booktitle={International Conference on Machine Learning (ICML)},
	pages={3174--3185},
	year={2021},
	organization={PMLR}
}

@inproceedings{zeng2025stability,
	title={Stability and generalization analysis of decentralized SGD: sharper bounds beyond lipschitzness and smoothness},
	author={Zeng, Shuang and Lei, Yunwen},
	booktitle={International Conference on Machine Learning (ICML)},
	year={2025},
	organization={PMLR}
}

@inproceedings{hu2025stability,
	title={Stability and Generalization of Zeroth-Order Decentralized Stochastic Gradient Descent with Changing Topology},
	author={Hu, Xiaolin and Gong, Zixuan and Xu, Gengze and Liu, Wei and Luan, Jian and Wang, Bin and Liu, Yong},
	booktitle={Proceedings of the AAAI Conference on Artificial Intelligence},
	volume={39},
	number={16},
	pages={17342--17350},
	year={2025}
}

@article{lin2018don,
	title={Don't use large mini-batches, use local sgd},
	author={Lin, Tao and Stich, Sebastian U and Patel, Kumar Kshitij and Jaggi, Martin},
	journal={arXiv preprint arXiv:1808.07217},
	year={2018}
}

@inproceedings{koloskova2020unified,
	title={A unified theory of decentralized sgd with changing topology and local updates},
	author={Koloskova, Anastasia and Loizou, Nicolas and Boreiri, Sadra and Jaggi, Martin and Stich, Sebastian},
	booktitle={International Conference on Machine Learning (ICML)},
	pages={5381--5393},
	year={2020},
	organization={PMLR}
}

@inproceedings{le2023refined,
	title={Refined convergence and topology learning for decentralized sgd with heterogeneous data},
	author={Le Bars, Batiste and Bellet, Aur{\'e}lien and Tommasi, Marc and Lavoie, Erick and Kermarrec, Anne-Marie},
	booktitle={International Conference on Artificial Intelligence and Statistics},
	pages={1672--1702},
	year={2023},
	organization={PMLR}
}

@article{yuan2023removing,
	title={Removing data heterogeneity influence enhances network topology dependence of decentralized sgd},
	author={Yuan, Kun and Alghunaim, Sulaiman A and Huang, Xinmeng},
	journal={Journal of Machine Learning Research},
	volume={24},
	number={280},
	pages={1--53},
	year={2023}
}

@article{shi2015extra,
	title={Extra: An exact first-order algorithm for decentralized consensus optimization},
	author={Shi, Wei and Ling, Qing and Wu, Gang and Yin, Wotao},
	journal={SIAM Journal on Optimization},
	volume={25},
	number={2},
	pages={944--966},
	year={2015},
	publisher={SIAM}
}

@article{yuan2016convergence,
	title={On the convergence of decentralized gradient descent},
	author={Yuan, Kun and Ling, Qing and Yin, Wotao},
	journal={SIAM Journal on Optimization},
	volume={26},
	number={3},
	pages={1835--1854},
	year={2016},
	publisher={SIAM}
}

@inproceedings{wang2024towards,
	title={Towards stability and generalization bounds in decentralized minibatch stochastic gradient descent},
	author={Wang, Jiahuan and Chen, Hong},
	booktitle={Proceedings of the AAAI Conference on Artificial Intelligence},
	volume={38},
	number={14},
	pages={15511--15519},
	year={2024}
}

@inproceedings{zhu2024stability,
	title={Stability and generalization of the decentralized stochastic gradient descent ascent algorithm},
	author={Zhu, Miaoxi and Shen, Li and Du, Bo and Tao, Dacheng},
	booktitle={Advances in Neural Information Processing Systems (NeuIPS)},
	volume={36},
	year={2023}
}

@inproceedings{li2014efficient,
	title={Efficient mini-batch training for stochastic optimization},
	author={Li, Mu and Zhang, Tong and Chen, Yuqiang and Smola, Alexander J},
	booktitle={International Conference on Knowledge Discovery and Data Mining},
	pages={661--670},
	year={2014}
}

@article{nedic2009distributed,
	title={Distributed subgradient methods for multi-agent optimization},
	author={Nedic, Angelia and Ozdaglar, Asuman},
	journal={IEEE Transactions on Automatic Control},
	volume={54},
	pages={48--61},
	year={2009}
}

@article{sundhar2010distributed,
	title={Distributed stochastic subgradient projection algorithms for convex optimization},
	author={Sundhar Ram, S and Nedi{\'c}, Angelia and Veeravalli, Venugopal V},
	journal={Journal of Optimization Theory and Applications},
	volume={147},
	pages={516--545},
	year={2010}
}

@article{richards2020graph,
	title={Graph-dependent implicit regularisation for distributed stochastic subgradient descent},
	author={Richards, Dominic and Rebeschini, Patrick},
	journal={Journal of Machine Learning Research},
	volume={21},
	pages={1-44},
	year={2020}
}

@inproceedings{bassily2020non,
	author = {Bassily,  Raef and Feldman, Vitaly and Guzm\'{a}n, Crist\'{o}bal and Talwar, Kunal},
	booktitle = {Advances in Neural Information Processing Systems (NeurIPS)},
	pages = {4381--4391},
	title = {Stability of stochastic gradient descent on nonsmooth convex losses},
	year = {2020}
}

@inproceedings{bars2023improved,
	title={Improved Stability and Generalization Analysis of the Decentralized SGD Algorithm},
	author={Bars, Batiste Le and Bellet, Aur{\'e}lien and Tommasi, Marc},
	booktitle={International Conference on Machine Learning (ICML)},
	year={2024}
}

@inproceedings{lian2017can,
	title={Can decentralized algorithms outperform centralized algorithms? a case study for decentralized parallel stochastic gradient descent},
	author={Lian, Xiangru and Zhang, Ce and Zhang, Huan and Hsieh, Cho-Jui and Zhang, Wei and Liu, Ji},
	booktitle={Advances in Neural Information Processing Systems (NeurIPS)},
	year={2017}
}

@inproceedings{lian2018asynchronous,
	title={Asynchronous decentralized parallel stochastic gradient descent},
	author={Lian, Xiangru and Zhang, Wei and Zhang, Ce and Liu, Ji},
	booktitle={International Conference on Machine Learning (ICML)},
	pages={3043--3052},
	year={2018}
}

@inproceedings{deng2023stability,
	title={Stability-Based Generalization Analysis of the Asynchronous Decentralized SGD},
	author={Deng, Xiaoge and Sun, Tao and Li, Shengwei and Li, Dongsheng},
	booktitle={AAAI Conference on Artificial Intelligence},
	pages={7340--7348},
	year={2023}
}

@inproceedings{sun2021stability,
	title={Stability and generalization of decentralized stochastic gradient descent},
	author={Sun, Tao and Li, Dongsheng and Wang, Bao},
	booktitle={AAAI Conference on Artificial Intelligence},
	pages={9756--9764},
	year={2021}
}

@inproceedings{zhu2022topology,
	title={Topology-aware generalization of decentralized sgd},
	author={Zhu, Tongtian and He, Fengxiang and Zhang, Lan and Niu, Zhengyang and Song, Mingli and Tao, Dacheng},
	booktitle={International Conference on Machine Learning (ICML)},
	pages={27479--27503},
	year={2022}
}

@inproceedings{sun-markov-18,
	author    = {Tao Sun and
	Yuejiao Sun and
	Wotao Yin},
	title     = {On Markov chain gradient descent},
	booktitle = {Advances in Neural Information Processing Systems (NeurIPS)},
	year      = {2018}
}

@inproceedings{wang-markov-2022,
	author    = {Puyu Wang and
	Yuewen Lei and Yiming Ying and Ding-Xuan Zhou
	},
	title     = {Stability and generalization for markov chain stochastic gradient methods},
	pages = {37735--37748},
	booktitle = {Advances in Neural Information Processing Systems (NeurIPS)},
	year      = {2022}
}

@article{sun2023decentralized,
	title={On the decentralized stochastic gradient descent with markov chain sampling},
	author={Sun, Tao and Li, Dongsheng and Wang, Bao},
	journal={IEEE Transactions on Signal Processing},
	pages={1-14},
	year={2023}
}

@inproceedings{hardt2016train,
	title={Train faster, generalize better: Stability of stochastic gradient descent},
	author={Hardt, Moritz and Recht, Ben and Singer, Yoram},
	booktitle={International Conference on Machine Learning (ICML)},
	pages={1225--1234},
	year={2016},
}

@inproceedings{lei2020sharper,
	title={Sharper generalization bounds for pairwise learning},
	author={Lei, Yunwen and Ledent, Antoine and Kloft, Marius},
	booktitle={Advances in Neural Information Processing Systems (NeurIPS)},
	pages={21236--21246},
	year={2020}
}

@inproceedings{lei2020fine,
	title={Fine-grained analysis of stability and generalization for stochastic gradient descent},
	author={Lei, Yunwen and Ying, Yiming},
	booktitle={International Conference on Machine Learning (ICML)},
	pages={5809--5819},
	year={2020}
}

@inproceedings{lei2021generalization,
	title={Generalization guarantee of SGD for pairwise learning},
	author={Lei, Yunwen and Liu, Mingrui and Ying, Yiming},
	booktitle={Advances in Neural Information Processing Systems (NeurIPS)},
	pages={21216--21228},
	year={2021}
}
	\newpage
	\onecolumn
	\appendix
	
	\begin{center}
		{\Large \bf Appendix for 
			``Stability and Generalization for Decentralized Markov SGD''}
	\end{center}
	
	\section{Centralized vs Decentralized}
	Centralized and decentralized optimization represent two fundamentally different communication paradigms. In centralized SGD, all workers communicate with a central server that aggregates local updates and broadcasts the global model. This design enables fast information aggregation but may suffer from scalability limitations, communication bottlenecks, and single-point failures.
	Decentralized SGD removes the central coordinator and relies on peer-to-peer communication over a network. Although information propagation is inherently more gradual and depends on the connectivity of the underlying graph, decentralized methods offer improved scalability, robustness, and fault tolerance, making them well suited for large-scale and distributed environments. The convergence behavior of decentralized methods is therefore closely tied to the spectral properties of the communication graph, with the spectral gap serving as a key measure of information mixing efficiency. In terms of the spectral gap $\gamma=1-\lambda$, larger values correspond to faster mixing and more efficient information propagation. From best to worst (in order of magnitude), the spectral gaps of several common network topologies satisfy:
	\begin{align*}
		\mathrm{complete~graphs / expanders / star~graphs}~(\Theta(1))\succ \mathrm{grids}~(\Theta(1/m))\succ \mathrm{rings}~(\Theta(1/m^2))
	\end{align*}
	\begin{figure}[h]
		\centering
		\includegraphics[width=15cm]{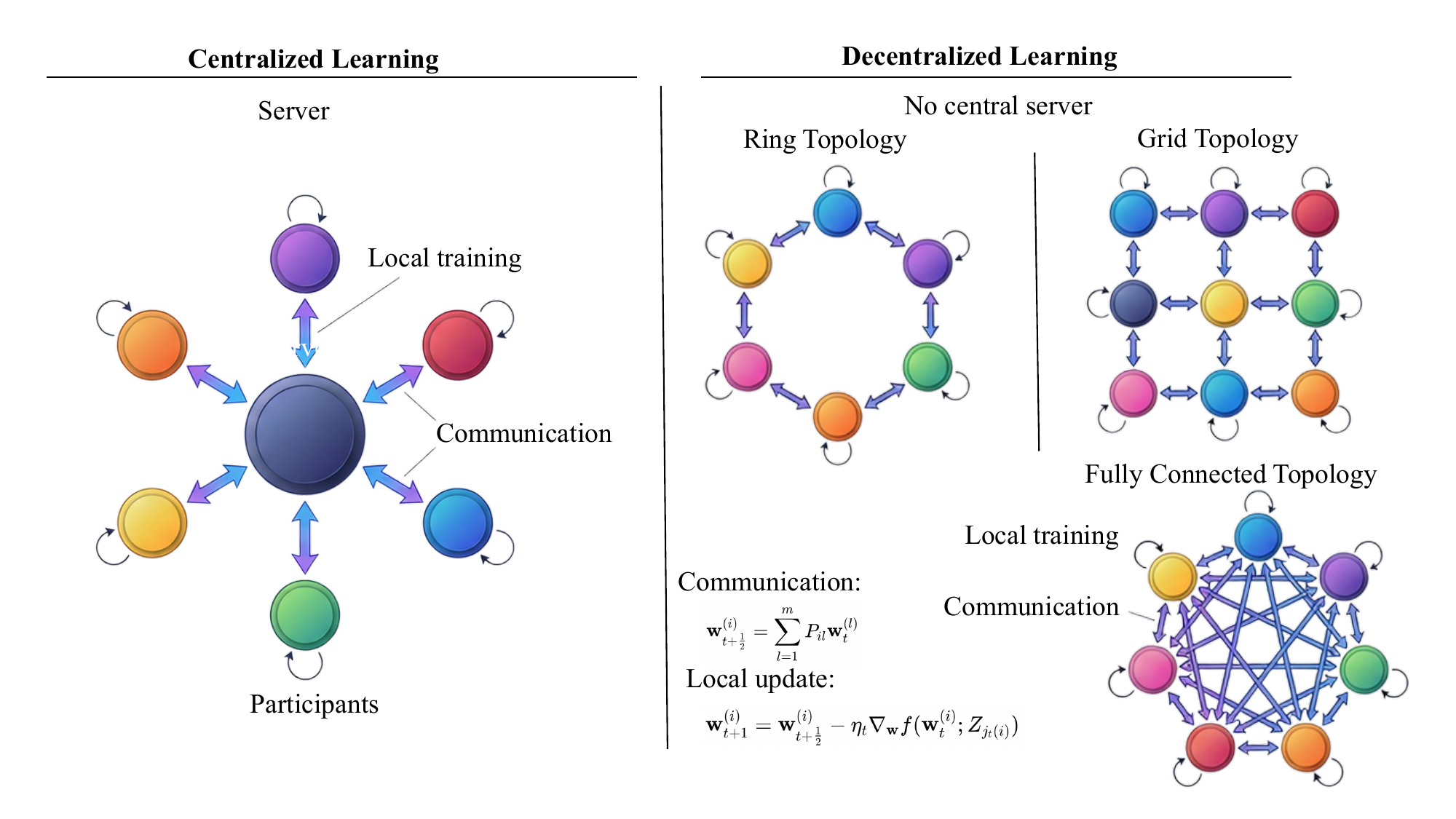}
		\caption{Illustration of the communication architectures in centralized and decentralized SGD.}
	\end{figure}
	\section{Technical Tools}
	\begin{lemma}[Lemma 3.6, \cite{hardt2016train}]
		\label{non}
		Assume the loss function $f\left(\mw ; Z\right)$ is convex and $\beta$-smooth with respect to $\mw$ for all $Z$. If $\eta\leq2/\beta$, then
		\begin{equation*}
			\left\|\mathbf{P}_{\mathcal{W}}\lp\mathbf{w}-\eta \nabla f\left(\mathbf{w} ; Z\right)-\mathbf{w}^{\prime}+\eta \nabla f\left(\mathbf{w}^{\prime} ; Z\right)\rp\right\|_{2} \leq\left\|\mathbf{w}-\mathbf{w}^{\prime}\right\|_{2} .
		\end{equation*}
	\end{lemma}
	\begin{lemma}[Lemma 5, \cite{sun2021stability,deng2023stability}]\label{sum}
		For any $0<\lambda<1$ and $t\in Z^+$, it holds that
		\begin{equation*}
			\sum_{q=1}^{t-1}\frac{\lambda^{t-1-q}}{q+1}\leq\frac{C_{\lambda}}{t}, \quad\quad C_{\lambda}:=\frac{1}{\lambda\log\frac{1}{\lambda}}\lp\frac{8}{e^{2}\log\frac{1}{\lambda}}+2\rp.
		\end{equation*}
	\end{lemma}
	\begin{lemma}[Lemma 8, \cite{sun2021stability}]
		Suppose that Assumption \ref{Lipschitz} holds. Let $\{\mw^t(i)\}_{i=1}^m$ and $\mw^t$ denote the local and averaged iterates of D-SGD at iteration $t$.
		Then
		\begin{equation*}
			\lk\sum_{i=1}^{m}\ls{\mw}^{t}-\mw^{t}(i)\rs^2\rk^\frac{1}{2}\leq2\sqrt{m}L\sum_{q=1}^{t}\eta_q\lambda^{t-q}.
		\end{equation*}
	\end{lemma}
	\begin{lemma}[Lemma 1,\cite{sun-markov-18}]\label{markov}
		Suppose Assumption 3 holds. Let $\lambda_i(H)$ denote the $i$-th largest eigenvalue of $H$, and define $\lambda(H)= \frac{\max \left\{\left|\lambda_2(H)\right|,\left|\lambda_n(H)\right|\right\}+1}{2} \in[1 / 2,1), C_H=\left(\sum_{i=2}^m d_i^2\right)^{1 / 2}\|U\|_F\left\|U^{-1}\right\|_F$ and
		$$
		K_H=\max \left\{\max _{1 \leq i \leq m}\left\{\left\lceil\frac{2 d_i\left(d_i-1\right)\left(\log \left(\frac{2 d_i}{\left|\lambda_2(H)\right| \cdot \log \left(\lambda(H) /\left|\lambda_2(H)\right|\right)}\right)-1\right)}{\left(d_i+1\right) \log \left(\lambda(H) /\left|\lambda_2(H)\right|\right)}\right\rceil\right\}, 0\right\}.
		$$
		There exist constants $C_H>0$ and $K_H \ge 0$ such that, for all $t \ge K_H$,
		$$
		\left\|\Pi^*-H^t\right\|_{\infty} \leq C_H \cdot(\lambda(H))^t.
		$$
		Moreover, if $H$ is symmetric, then $K_H=0$ and
		$$
		\left\|\Pi^*-H^t\right\|_{\infty} \leq n^{3 / 2} \cdot(\lambda(H))^t, \forall t \geq 0.     
		$$
	\end{lemma}
	\begin{lemma}[\textnormal{[Vershynin, 2018]}]
		\label{mag}
		Let $\{z_i\}_{i=1}^m$ be a sequence of (possibly dependent) random variables, and let
		$\xi_i=\xi_i(z_1,\ldots,z_i)$ satisfy
		$
		|\xi_i-\mathbb{E}_{z_i}[\xi_i]|\le b_i
		$.
		Then for any $\alpha\in(0,1)$, with probability at least $1-\alpha$,
		\begin{align*}
			\sum_{i=1}^m \xi_i-\sum_{i=1}^m\mathbb{E}_{z_i}\left[\xi_i\right] \leq\left(2 \sum_{i=1}^m b_i^2 \log ({1}/{\alpha}) \right)^{\frac{1}{2}}.
		\end{align*}
	\end{lemma}
	\begin{lemma}[\textnormal{[Schmidt \textit{et al}., 2011]}]
		\label{sequel}
		Let $\{u_t\}_{t\ge0}$ be a non-negative sequence satisfying
		$$
		u_t^2 \leq S_t+\sum_{\tau=1}^{t-1} \alpha_\tau u_\tau,
		$$
		where $\left\{S_\tau: \tau \in \mathbb{N}\right\}$ denotes a non-decreasing sequence satisfying $S_0 \geq u_0^2$, and $\alpha_\tau \geq 0$ for all $\tau \in \mathbb{N}$. Then
		$$
		u_t \leq \sqrt{S_t}+\sum_{\tau=1}^{t-1} \alpha_\tau.
		$$
	\end{lemma}
	\section{Stability Bound : Consensus-then-Gradient (CtG)}
	Before proceeding with the proof, let's first review the update details of DMc-SGD and try to derive some properties.
	Consider the output for the algorithm,
	\begin{align*}
		{\mw}^{t}=\frac{1}{m}\sum_{i=1}^{m}\mw^{t}(i)=&\frac{1}{m}\sum_{i=1}^{m}\lk\sum_{l=1}^{m}P_{i,l}\mw^{t-1}(l)-\eta_{t}\nabla f\lp\mw^{t-1}(i);Z_{j_{t}(i)}\rp\rk\\
		\overset{\textbf{(a)}}{=}&\frac{1}{m}\sum_{l=1}^{m}\mw^{t-1}(l)-\frac{\eta_{t}}{m}\sum_{i=1}^{m}\nabla f\lp\mw^{t-1}(i);Z_{j_{t}(i)}\rp\\
		=&{\mw}^{t-1}-\frac{\eta_{t}}{m}\sum_{i=1}^{m}\nabla f\lp\mw^{t-1}(i);Z_{j_{t}(i)}\rp,
	\end{align*}
	where Equation $\textbf{(a)}$ uses the fact of gossip matrix property.
	
	Recall the definition of on-average stability, we need to measure the gap or distance between two sets of weights in two datasets that differs by only one sample. Let ${\mw}^{t+1}$ and ${\mw}_{(rk)}^{t+1}$ be produced by DMc-SGD based on $S$ and $S_{rk}$ respectively. Specifically, these sets are constructed as follows:
	\begin{align*}
		S=&\{Z_{1(1)},\cdots,Z_{n(1)},\cdots,Z_{1(r)},\cdots,Z_{k(r)},\cdots,Z_{n(r)},\cdots,Z_{1(m)},\cdots,Z_{n(m)}\},\\
		S_{rk}=&\{Z_{1(1)},\cdots,Z_{n(1)},\cdots,Z_{1(r)},\cdots,\tilde{Z}_{k(r)},\cdots,Z_{n(r)},\cdots,Z_{1(m)},\cdots,Z_{n(m)}\}.
	\end{align*}
	Here, $\tilde{Z}_{k(r)}$ denotes the altered data in the $r$-th subset.
	
	\begin{proof}
		If $j_t\neq k$, then 
		\begin{align*}
			&\ls{\mw}^{t}-{\mw}_{(rk)}^{t}\rs\\
			=&\ls\frac{1}{m}\sum_{i=1}^{m}\mw^{t}(i)-\frac{1}{m}\sum_{i=1}^{m}{\mw}_{(rk)}^{t}(i)\rs\\
			=&\frac{1}{m}\sum_{i=1}^{m}\ls\sum_{l=1}^{m}P_{i,l}\mw^{t-1}(l)-\eta_t\nabla f\lp\mw^{t-1}(i);Z_{j_t(i)}\rp-\sum_{l=1}^{m}P_{i,l}{\mw}_{(rk)}^{t-1}(l)+\eta_t\nabla f\lp{\mw}_{(rk)}^{t-1}(i);\tilde{Z}_{j_t(i)}\rp\rs\\
			=&\frac{1}{m}\sum_{i=1}^{m}\ls\sum_{l=1}^{m}P_{i,l}\mw^{t-1}(l)-\eta_t\nabla f\lp\mw^{t-1}(i);Z_{j_t(i)}\rp-\sum_{l=1}^{m}P_{i,l}{\mw}_{(rk)}^{t-1}(l)+\eta_t\nabla f\lp{\mw}_{(rk)}^{t-1}(i);{Z}_{j_t(i)}\rp\rs\\
			=&\frac{1}{m}\sum_{i=1}^{m}\Bigg\|\sum_{l=1}^{m}P_{i,l}\mw^{t-1}(l)-\eta_t\nabla f\lp\mw^{t-1}(i);Z_{j_t(i)}\rp+\eta_t\nabla f\lp\mw^{t-1};Z_{j_t(i)}\rp-\eta_t\nabla f\lp\mw^{t-1};Z_{j_t(i)}\rp\\
			&\quad\quad\quad\quad-\sum_{l=1}^{m}P_{i,l}{\mw}_{(rk)}^{t-1}(l)+\eta_t\nabla f\lp{\mw}_{(rk)}^{t-1}(i);Z_{j_t(i)}\rp+\eta_t\nabla f\lp{\mw}_{(rk)}^{t-1};Z_{j_t(i)}\rp-\eta_t\nabla f\lp{\mw}_{(rk)}^{t-1};Z_{j_t(i)}\rp\Bigg\|_2\\
			\leq&\frac{1}{m}\sum_{i=1}^{m}\ls\sum_{l=1}^{m}P_{i,l}\mw^{t-1}(l)-\eta_t\nabla f\lp\mw^{t};Z_{j_t(i)}\rp-\sum_{l=1}^{m}P_{i,l}{\mw}_{(rk)}^{t-1}(l)+\eta_t\nabla f\lp{\mw}_{(rk)}^{t-1};Z_{j_t(i)}\rp\rs\\
			&\quad\quad+\frac{1}{m}\sum_{i=1}^{m}\ls\eta_t\nabla f\lp\mw^{t-1};Z_{j_t(i)}\rp-\eta_t\nabla f\lp\mw^{t-1}(i);Z_{j_t(i)}\rp\rs\\
			&\quad\quad+\frac{1}{m}\sum_{i=1}^{m}\ls\eta_t\nabla f\lp{\mw}_{(rk)}^{t-1}(i);Z_{j_t(i)}\rp-\eta_t\nabla f\lp{\mw}_{(rk)}^{t-1};Z_{j_t(i)}\rp\rs\\
			\leq
			&\frac{1}{m}\ls\sum_{i=1}^{m}\sum_{l=1}^{m}P_{i,l}\mw^{t-1}(l)-\sum_{i=1}^{m}\eta_t\nabla f\lp\mw^{t-1};Z_{j_t(i)}\rp-\sum_{i=1}^{m}\sum_{l=1}^{m}P_{i,l}{\mw}_{(rk)}^{t-1}(l)+\sum_{i=1}^{m}\eta_t\nabla f\lp{\mw}_{(rk)}^{t-1};Z_{j_t(i)}\rp\rs\\
			&\quad\quad+\frac{\eta_t\beta}{m}\sum_{i=1}^{m}\ls\mw^{t-1}-\mw^{t-1}(i)\rs+\frac{\eta_t\beta}{m}\sum_{i=1}^{m}\ls{\mw}_{(rk)}^{t-1}-{\mw}_{(rk)}^{t-1}(i)\rs.
		\end{align*}
		Moreover,
		\begin{align*}
			\ls{\mw}^{t}-{\mw}_{(rk)}^{t}\rs
			\overset{\textbf{(a)}}{\leq}&\frac{1}{m}\ls\sum_{l=1}^{m}\mw^{t-1}(l)-\sum_{i=1}^{m}\eta_t\nabla f\lp\mw^{t-1};Z_{j_t(i)}\rp-\sum_{l=1}^{m}{\mw}_{(rk)}^{t-1}(l)+\sum_{i=1}^{m}\eta_t\nabla f\lp{\mw}_{(rk)}^{t-1};Z_{j_t(i)}\rp\rs\\
			&\quad\quad+\frac{\eta_t\beta}{m}\sum_{i=1}^{m}\ls\mw^{t-1}-\mw^{t-1}(i)\rs+\frac{\eta_t\beta}{m}\sum_{i=1}^{m}\ls{\mw}_{(rk)}^{t-1}-{\mw}_{(rk)}^{t-1}(i)\rs\\
			\overset{\textbf{(b)}}{\leq}&\ls\mw^{t-1}-\frac{1}{m}\sum_{i=1}^{m}\eta_t\nabla f\lp\mw^{t-1};Z_{j_t(i)}\rp-{\mw}_{(rk)}^{t-1}+\frac{1}{m}\sum_{i=1}^{m}\eta_t\nabla f\lp{\mw}_{(rk)}^{t-1};Z_{j_t(i)}\rp\rs\\
			&\quad\quad+\frac{\eta_t\beta}{\sqrt{m}}\lk\sum_{i=1}^{m}\ls{\mw}^{t-1}-\mw^{t-1}(i)\rs^2\rk^\frac{1}{2}+\frac{\eta_t\beta}{\sqrt{m}}\lk\sum_{i=1}^{m}\ls{\mw}_{(rk)}^{t-1}-{\mw}_{(rk)}^{t-1}(i)\rs^2\rk^\frac{1}{2}\\
			\overset{\textbf{(c)}}{\leq}&\ls{\mw}^{t-1}-{\mw}_{(rk)}^{t-1}\rs+4\eta_t\beta L\sum_{q=1}^{t-1}\eta_q\lambda^{t-q-1},
		\end{align*}
		where the inequality $\textbf{(a)}$ the $\beta$-smoothness and the doubly random matrix property, the inequality $\textbf{(b)}$ relies on the basic inequality, and the inequality $\textbf{(c)}$ employs the non-expansive operator (see \cref{non}).
		
		If $j_t=k$, then 
		\begin{align*}
			&\ls{\mw}^{t}-{\mw}_{(rk)}^{t}\rs\\
			=&\ls\frac{1}{m}\sum_{i=1}^{m}\mw^{t}(i)-\frac{1}{m}\sum_{i=1}^{m}{\mw}_{(rk)}^{t}(i)\rs\\
			=&\frac{1}{m}\sum_{i=1}^{m}\ls\sum_{l=1}^{m}P_{i,l}\mw^{t-1}(l)-\eta_t\nabla f\lp\mw^{t-1}(i);Z_{j_t(i)}\rp-\sum_{l=1}^{m}P_{i,l}{\mw}_{(rk)}^{t-1}(l)+\eta_t\nabla f\lp{\mw}_{(rk)}^{t-1}(i);\tilde{Z}_{j_t(i)}\rp\rs\\
			\leq&\frac{1}{m}\sum_{i=1}^{m}\ls\sum_{l=1}^{m}P_{i,l}\mw^{t-1}(l)-\eta_t\nabla f\lp\mw^{t-1};Z_{j_t(i)}\rp-\sum_{l=1}^{m}P_{i,l}{\mw}_{(rk)}^{t-1}(l)+\eta_t\nabla f\lp{\mw}_{(rk)}^{t-1};\tilde{Z}_{j_t(i)}\rp\rs\\
			&+\frac{1}{m}\sum_{i=1}^{m}\ls\eta_t\nabla f\lp\mw^{t-1};Z_{j_t(i)}\rp-\eta_t\nabla f\lp\mw^{t-1}(i);Z_{j_t(i)}\rp\rs\\
			&+\frac{1}{m}\sum_{i=1}^{m}\ls\eta_t\nabla f\lp{\mw}_{(rk)}^{t-1}(i);\tilde{Z}_{j_t(i)}\rp-\eta_t\nabla f\lp{\mw}_{(rk)}^{t-1};\tilde{Z}_{j_t(i)}\rp\rs\\
			\leq&\frac{1}{m}\sum_{i=1}^{m}\ls\sum_{l=1}^{m}P_{i,l}\mw^{t-1}(l)-\eta_t\nabla f\lp\mw^{t-1};Z_{j_t(i)}\rp-\sum_{l=1}^{m}P_{i,l}{\mw}_{(rk)}^{t-1}(l)+\eta_t\nabla f\lp{\mw}_{(rk)}^{t-1};\tilde{Z}_{j_t(i)}\rp\rs+4\eta_t\beta L\sum_{q=1}^{t-1}\eta_q\lambda^{t-q-1}\\
			\overset{\textbf{(d)}}{\leq}&\ls\mw^{t-1}-\frac{\eta_t}{m}\sum_{i=1,i\neq r}^{m}\nabla f\lp\mw^{t-1};Z_{j_t(i)}\rp-{\mw}_{(rk)}^{t-1}+\frac{\eta_t}{m}\sum_{i=1,i\neq r}^{m}\nabla f\lp{\mw}_{(rk)}^{t-1};Z_{j_t(i)}\rp\rs\\
			&+\frac{\eta_t}{m}\ls\nabla f\lp\mw^{t-1};Z_{k(r)}\rp-\nabla f\lp{\mw}_{(rk)}^{t-1};\tilde{Z}_{k(r)}\rp\rs+4\eta_t\beta L\sum_{q=1}^{t-1}\eta_q\lambda^{t-q-1}\\
			\overset{\textbf{(e)}}{\leq}&\ls{\mw}^{t-1}-{\mw}_{(rk)}^{t-1}\rs+4\eta_t\beta L\sum_{q=1}^{t-1}\eta_q\lambda^{t-q-1}+\frac{2\eta_tL}{m}.
		\end{align*}
		For the above inequality $\textbf{(d)}$, we separate the machine $r$ to determine whether anomalous samples exist. This ensures that the former can continue leveraging the non-expansive operator property (See \cref{non}). Inequality $\textbf{(e)}$, on the other hand, is derived based on the L-Lipschitz assumption.
		
		Suppose that $\delta_{(rk)}^{t}=\ls{\mw}^{t}-{\mw}_{(rk)}^{t}\rs$, we can get
		\begin{align*}
			\delta_{(rk)}^{t}\leq\delta_{(rk)}^{t-1}+4\eta_t\beta L\sum_{q=1}^{t-1}\eta_q\lambda^{t-q-1}+\frac{2\eta_tL}{m}\mathbb{I}_{[j_t=k]}.
		\end{align*}
		where $\mathbb{I}_{[j_t=k]}$ denotes the event that machine $r$ selects the $k$-th sample at step $t$.
		
		Apply the above inequality recursively,
		\begin{align*}
			\delta_{(rk)}^{T}\leq4\beta L\sum_{t=1}^{T}\eta_t\sum_{q=1}^{t-1}\eta_q\lambda^{t-q-1}+\frac{2L}{m}\sum_{t=1}^{T}\eta_t\mathbb{I}_{[j_t=k]}.\label{convex recursion}
		\end{align*}
		Taking average about $k$ and $r$,
		\begin{align}
			\frac{1}{mn}\sum_{r=1}^{m}\sum_{k=1}^{n}\delta_{(rk)}^{T}\leq4\beta L\sum_{t=1}^{T}\eta_t\sum_{q=1}^{t-1}\eta_q\lambda^{t-q-1}+\frac{2L}{mn}\sum_{t=1}^{T}\eta_t.
		\end{align}
		where the inequality uses $\sum_{k=1}^{n}\mathbb{I}_{[j_t=k]}=1$. Taking expectation about the algorithm $\A$, we have
		\begin{align*}
			\mathbb{E}_{\A}\lk\frac{1}{mn}\sum_{r=1}^{m}\sum_{k=1}^{n}\ls{\mw}^{T}-{\mw}_{(rk)}^{T}\rs\rk\leq4\beta L\sum_{t=1}^{T}\eta_t\sum_{q=1}^{t-1}\eta_q\lambda^{t-q-1}+\frac{2L}{mn}\sum_{t=1}^{T}\eta_t.
		\end{align*}
		
		\vspace{0.8cm}
		
		Now, we turn to the non-smooth case. We first consider the weight square term.
		\begin{align*}
			\ls{\mw}^{t}-{\mw}_{(rk)}^{t}\rs^2
			=&\ls{\mw}^{t-1}-\frac{\eta_{t}}{m}\sum_{i=1}^{m}\nabla f\lp\mw^{t-1}(i);Z_{j_{t}(i)}\rp-{\mw}_{(rk)}^{t-1}+\frac{\eta_{t}}{m}\sum_{i=1}^{m}\nabla f\lp\mw_{(rk)}^{t}(i);\tilde{Z}_{j_{t}(i)}\rp\rs^2\\
			=&\ls{\mw}^{t-1}-{\mw}_{(rk)}^{t-1}\rs^2+\frac{\eta_{t}^2}{m^2}\ls\sum_{i=1}^{m}\lp\nabla f\lp\mw^{t-1}(i);Z_{j_{t}(i)}\rp-\nabla f\lp\mw_{(rk)}^{t-1}(i);\tilde{Z}_{j_{t}(i)}\rp\rp\rs^2\\
			&-\frac{2\eta_{t}}{m}\Big\langle{\mw}^{t-1}-{\mw}_{(rk)}^{t-1},\sum_{i=1}^{m}\lp\nabla f\lp\mw^{t-1}(i);Z_{j_{t}(i)}\rp-\nabla f\lp\mw_{(rk)}^{t-1}(i);\tilde{Z}_{j_{t}(i)}\rp\rp\Big\rangle.
		\end{align*}
		
		If $i=r$ and $j_t=k$, then we have
		\begin{align*}
			\Big\langle{\mw}^{t-1}-{\mw}_{(rk)}^{t-1},\nabla f\lp\mw^{t-1}(i);Z_{j_{t}(i)}\rp-\nabla f\lp\mw_{(rk)}^{t-1}(i);\tilde{Z}_{j_{t}(i)}\rp\Big\rangle\geq-2L\ls{\mw}^{t-1}-{\mw}_{(rk)}^{t-1}\rs.
		\end{align*}
		\par If ($i=r$ and $j_t\neq k$) or $i\neq r$, we get
		\begin{align*}
			&\Big\langle{\mw}^{t-1}-{\mw}_{(rk)}^{t-1},\nabla f\lp\mw^{t-1}(i);Z_{j_{t}(i)}\rp-\nabla f\lp\mw_{(rk)}^{t-1}(i);{Z}_{j_{t}(i)}\rp\Big\rangle\\
			=&\Big\langle{\mw}^{t-1}(i)-{\mw}_{(rk)}^{t-1}(i),\nabla f\lp\mw^{t-1}(i);Z_{j_{t}(i)}\rp-\nabla f\lp\mw_{(rk)}^{t-1}(i);{Z}_{j_{t}(i)}\rp\Big\rangle\\
			&+\Big\langle{\mw}^{t-1}-{\mw}^{t-1}(i),\nabla f\lp\mw^{t-1}(i);Z_{j_{t}(i)}\rp-\nabla f\lp\mw_{(rk)}^{t-1}(i);{Z}_{j_{t}(i)}\rp\Big\rangle\\
			&+\Big\langle{\mw}_{(rk)}^{t-1}(i)-
			{\mw}_{(rk)}^{t-1},\nabla f\lp\mw^{t-1}(i);Z_{j_{t}(i)}\rp-\nabla f\lp\mw_{(rk)}^{t-1}(i);{Z}_{j_{t}(i)}\rp\Big\rangle\\	
			\geq&-2L\|{\mw}^{t-1}-{\mw}^{t-1}(i)\|-2L\|{\mw}_{(rk)}^{t-1}-{\mw}_{(rk)}^{t-1}(i)\|\\
			\geq&-4L\|{\mw}^{t-1}-{\mw}^{t-1}(i)\|.
		\end{align*}
		Back to the sum term,
		\begin{align*}
			\sharp:=-\frac{2\eta_{t}}{m}\Big\langle{\mw}^{t-1}-{\mw}_{(rk)}^{t-1},\sum_{i=1}^{m}\lp\nabla f\lp\mw^{t-1}(i);Z_{j_{t}(i)}\rp-\nabla f\lp\mw_{(rk)}^{t-1}(i);\tilde{Z}_{j_{t}(i)}\rp\rp\Big\rangle.
		\end{align*}
		If $j_t\neq k$, we obtain
		\begin{align*}
			\sharp\leq\frac{8L\eta_{t}}{m}\sum_{i=1}^{m}\|{\mw}^{t-1}-{\mw}^{t-1}(i)\|\leq16L^2\eta_{t}\sum_{q=1}^{t-1}\eta_q\lambda^{t-q-1}.
		\end{align*}
		If $j_t= k$, we know
		\begin{align*}
			\sharp\leq&\frac{8L\eta_{t}}{m}\sum_{i=1}^{m}\|{\mw}^{t-1}-{\mw}^{t-1}(i)\|+\frac{4L\eta_{t}}{m}\ls{\mw}^{t-1}-{\mw}_{(rk)}^{t-1}\rs\\
			\leq&\frac{4L\eta_{t}}{m}\ls{\mw}^{t-1}-{\mw}_{(rk)}^{t-1}\rs+16L^2\eta_{t}\sum_{q=1}^{t-1}\eta_q\lambda^{t-q-1}.
		\end{align*}
		Combine the above two case,
		\begin{align*}
			\ls{\mw}^{t}-{\mw}_{(rk)}^{t}\rs^2\leq\ls{\mw}^{t-1}-{\mw}_{(rk)}^{t-1}\rs^2+4\eta_{t}^2L^2+16L^2\eta_{t}\sum_{q=1}^{t-1}\eta_q\lambda^{t-q-1}+\frac{4L\eta_{t}}{m}\ls{\mw}^{t-1}-{\mw}_{(rk)}^{t-1}\rs\mathbb{I}_{[j_t=k]}.
		\end{align*}
		Recursiving the above inequality, we get the following result
		\begin{align*}
			\ls{\mw}^{t}-{\mw}_{(rk)}^{t}\rs^2\leq&4L^2\sum_{s=1}^{t}\eta_{s}^2+16L^2\sum_{s=1}^{t}\eta_{s}\sum_{q=1}^{s-1}\eta_q\lambda^{s-q-1}+\frac{4L}{m}\sum_{s=1}^{t}\eta_{s}\ls{\mw}^{s-1}-{\mw}_{(rk)}^{s-1}\rs\mathbb{I}_{[j_s=k]}\\
			&\leq4L^2\sum_{s=1}^{t}\eta_{s}^2+16L^2\sum_{s=1}^{t}\eta_{s}\sum_{q=1}^{s-1}\eta_q\lambda^{s-q-1}+\frac{4L}{m}\sum_{s=1}^{t-1}\eta_{s+1}\ls{\mw}^{s}-{\mw}_{(rk)}^{s}\rs\mathbb{I}_{[j_{s+1}=k]}.
		\end{align*}
		According to Lemma \ref{sequel}, we shows
		\begin{align*}
			\ls{\mw}^{T}-{\mw}_{(rk)}^{T}\rs\leq\sqrt{4L^2\sum_{t=1}^{T}\eta_{t}^2+16L^2\sum_{t=1}^{T}\eta_{t}\sum_{q=1}^{t-1}\eta_q\lambda^{t-q-1}}+\frac{4L}{m}\sum_{t=1}^{T-1}\eta_{t+1}\mathbb{I}_{[j_{t+1}=k]}.
		\end{align*}
		Taking average about $k$ and $r$,
		\begin{align}
			\frac{1}{mn}\sum_{r=1}^{m}\sum_{k=1}^{n}\delta_{(rk)}^{T}\leq&\sqrt{4L^2\sum_{t=1}^{T}\eta_{t}^2+16L^2\sum_{t=1}^{T}\eta_{t}\sum_{q=1}^{t-1}\eta_q\lambda^{t-q-1}}+\frac{4L}{mn}\sum_{t=1}^{T-1}\eta_{t+1}\sum_{k=1}^{n}\mathbb{I}_{[j_{t+1}=k]}\nonumber\\
			\leq&\sqrt{4L^2\sum_{t=1}^{T}\eta_{t}^2+16L^2\sum_{t=1}^{T}\eta_{t}\sum_{q=1}^{t-1}\eta_q\lambda^{t-q-1}}+\frac{4L}{mn}\sum_{t=1}^{T}\eta_{t}.\label{recu}
		\end{align}
		where the first inequality uses $\sum_{k=1}^{n}\mathbb{I}_{[j_{t+1}=k]}=1$. Taking expectation about the algorithm $\A$, we have
		\begin{align*}
			\mathbb{E}_{\A}\lk\frac{1}{mn}\sum_{r=1}^{m}\sum_{k=1}^{n}\delta_{(rk)}^{T}\rk
			\leq&\sqrt{4L^2\sum_{t=1}^{T}\eta_{t}^2+16L^2\sum_{t=1}^{T}\eta_{t}\sum_{q=1}^{t-1}\eta_q\lambda^{t-q-1}}+\frac{4L}{mn}\sum_{t=1}^{T}\eta_{t}\\
			\leq&2L\sqrt{\sum_{t=1}^{T}\eta_{t}^2}+4L\sqrt{\sum_{t=1}^{T}\eta_{t}\sum_{q=1}^{t-1}\eta_q\lambda^{t-q-1}}+\frac{4L}{mn}\sum_{t=1}^{T}\eta_{t}.
		\end{align*}
	\end{proof}
	\subsection{Discussion about learning rate}
	\textbf{Smooth Case}--[Constant stepsize]:
	\begin{proof}
		If $\eta_t=\eta\leq2/\beta$, then we get
		\begin{align*}
			\delta_{(rk)}^{t}\leq\delta_{(rk)}^{t-1}+\frac{4\eta^2\beta L}{1-\lambda}+\frac{2\eta L}{m}\mathbb{I}_{[j_t=k]}.
		\end{align*}
		By iteratively employing the aforementioned inequality, we can derive
		\begin{align*}
			\delta_{(rk)}^{T}\leq\sum_{t=1}^{T}\lp\frac{4\eta^2\beta L}{1-\lambda}+\frac{2\eta L}{m}\mathbb{I}_{[j_t=k]}\rp.
		\end{align*}
		Taking average about $k$ and $r$,
		\begin{align*}
			\frac{1}{mn}\sum_{r=1}^{m}\sum_{k=1}^{n}\delta_{(rk)}^{T}\leq&\sum_{t=1}^{T}\lp\frac{4\eta^2\beta L}{1-\lambda}+\frac{2\eta L}{mn}\sum_{k=1}^{n}\mathbb{I}_{[j_t=k]}\rp\\
			\leq&\frac{4\eta^2\beta LT}{1-\lambda}+\frac{2\eta LT}{mn}.
		\end{align*}
	\end{proof}
	\noindent[Decreasing stepsize]:
	\begin{proof}
		If $\eta=\frac{1}{t+1}$, then we have
		\begin{align*}
			\delta_{(rk)}^{t}\leq&\delta_{(rk)}^{t-1}+\frac{4\beta L}{t+1}\sum_{q=1}^{t-1}\frac{\lambda^{t-q-1}}{q+1}+\frac{2 L}{m(t+1)}\mathbb{I}_{[j_t=k]}\\
			\leq&\delta_{(rk)}^{t-1}+\frac{4\beta LC_{\lambda}}{(t+1)t}+\frac{2 L}{m(t+1)}\mathbb{I}_{[j_t=k]},
		\end{align*}
		where the last inequality relies on \cref{sum}. 
		
		We can repeatedly use the aforementioned inequality to deduce
		\begin{align*}
			\delta_{(rk)}^{T}\leq&\sum_{t=1}^{T}\lp\frac{4\beta LC_{\lambda}}{(t+1)t}+\frac{2L}{(t+1)m}\mathbb{I}_{[j_t=k]}\rp\\
			=&4\beta LC_{\lambda}\sum_{t=1}^{T}\lp\frac{1}{t}-\frac{1}{t+1}\rp+\frac{2L}{m}\sum_{t=1}^{T}\frac{1}{t+1}\mathbb{I}_{[j_t=k]}\\
			\leq&\frac{4\beta LC_{\lambda}T}{T+1}+\frac{2L}{m}\sum_{t=1}^{T}\frac{1}{t+1}\mathbb{I}_{[j_t=k]}.
		\end{align*}
		Taking average about $k$ and $r$,
		\begin{align*}
			\frac{1}{mn}\sum_{r=1}^{m}\sum_{k=1}^{n}\delta_{(rk)}^{T}
			\leq&\frac{4\beta LC_{\lambda}T}{T+1}+\frac{2L}{mn}\sum_{t=1}^{T}\frac{1}{t+1}\sum_{k=1}^{n}\mathbb{I}_{[j_t=k]}\\
			\leq&\frac{4\beta LC_{\lambda}T}{T+1}+\frac{2L\ln(T+1)}{mn}.
		\end{align*}
	\end{proof}

	\noindent\textbf{Non-smooth Case}--[Constant stepsize]:
	\begin{proof}
		If $\eta_t=\eta\leq2/\beta$, then we get
		\begin{align*}
			\lp\delta_{(rk)}^{t}\rp^2\leq4L^2\eta^2t+\frac{16L^2\eta^2t}{1-\lambda}+\frac{4L\eta}{m}\sum_{s=1}^{t-1}\delta_{(rk)}^{s}\mathbb{I}_{[j_{s+1}=k]}.
		\end{align*}
		Furthermore, we can derive
		\begin{align*}
			\lp\delta_{(rk)}^{T}\rp^2\leq4L^2\eta^2T+\frac{16L^2\eta^2T}{1-\lambda}+\frac{4L\eta}{m}\sum_{t=1}^{T-1}\delta_{(rk)}^{t}\mathbb{I}_{[j_{t+1}=k]}.
		\end{align*}
		According to the \cref{sequel}, we have
		\begin{align*}
			\delta_{(rk)}^{T}\leq2L\eta\sqrt{T}+\frac{4L\eta \sqrt{T}}{1-\lambda}+\frac{4L\eta}{m}\sum_{t=1}^{T-1}\mathbb{I}_{[j_{t+1}=k]}.
		\end{align*}
		Taking average about $k$ and $r$,
		\begin{align*}
			\frac{1}{mn}\sum_{r=1}^{m}\sum_{k=1}^{n}\delta_{(rk)}^{T}\leq&2L\eta\sqrt{T}+\frac{4L\eta \sqrt{T}}{1-\lambda}+\frac{4L\eta}{mn}\sum_{t=1}^{T-1}\sum_{k=1}^{n}\mathbb{I}_{[j_{t+1}=k]}\\
			\leq&2L\eta\sqrt{T}+\frac{4L\eta \sqrt{T}}{1-\lambda}+\frac{4LT\eta}{mn}.
		\end{align*}
	\end{proof}
	\noindent[Decreasing stepsize]:
	\begin{proof}
		If $\eta=\frac{1}{t+1}$, then we have
		\begin{align*}
			\lp\delta_{(rk)}^{t}\rp^2\leq&\lp\delta_{(rk)}^{t-1}\rp^2+\frac{4L^2}{(t+1)^2}+\frac{16 L^2}{t+1}\sum_{q=1}^{t-1}\frac{\lambda^{t-q-1}}{q+1}+\frac{4 L}{m(t+1)}\delta_{(rk)}^{t-1}\mathbb{I}_{[j_t=k]}\\
			\leq&\lp\delta_{(rk)}^{t-1}\rp^2+\frac{4L^2}{(t+1)^2}+\frac{16L^2C_{\lambda}}{(t+1)t}+\frac{4 L}{m(t+1)}\delta_{(rk)}^{t-1}\mathbb{I}_{[j_t=k]},
		\end{align*}
		where the last inequality relies on \cref{sum}.
		
		\noindent We can repeatedly use the aforementioned inequality to deduce
		\begin{align*}
			\lp\delta_{(rk)}^{t}\rp^2\leq&\sum_{s=1}^{t}\lp\frac{4L^2}{(s+1)^2}+\frac{16L^2C_{\lambda}}{(s+1)s}+\frac{4L}{(s+1)m}\delta_{(rk)}^{s-1}\mathbb{I}_{[j_s=k]}\rp\\
			=&4L^2\sum_{s=1}^{t}\frac{1}{(s+1)^2}+16 L^2C_{\lambda}\sum_{s=1}^{t}\lp\frac{1}{s}-\frac{1}{s+1}\rp+\frac{4L}{m}\sum_{s=1}^{t}\frac{1}{s+1}\delta_{(rk)}^{s-1}\mathbb{I}_{[j_s=k]}\\
			\leq&4L^2+\frac{16 L^2C_{\lambda}t}{t+1}+\frac{2L}{m}\sum_{s=1}^{t}\frac{1}{s+1}\delta_{(rk)}^{s-1}\mathbb{I}_{[j_s=k]}\\
			=&4L^2+\frac{16 L^2C_{\lambda}t}{t+1}+\frac{2L}{m}\sum_{s=1}^{t-1}\frac{\delta_{(rk)}^{s}}{s}\mathbb{I}_{[j_{s+1}=k]}.
		\end{align*}
		According to the \cref{sequel}, we get
		\begin{align*}
			\delta_{(rk)}^{t}
			\leq2L+4L\sqrt{C_{\lambda}}+\frac{2L}{m}\sum_{s=1}^{t-1}\frac{1}{s} \mathbb{I}_{[j_{s+1}=k]}.
		\end{align*}
		Furthermore, we derive that
		\begin{align*}
			\delta_{(rk)}^{T}
			\leq2L+4L\sqrt{C_{\lambda}}+\frac{2L}{m}\sum_{t=1}^{T-1}\frac{1}{t} \mathbb{I}_{[j_{t+1}=k]}.
		\end{align*}
		Taking average about $k$ and $r$,
		\begin{align*}
			\frac{1}{mn}\sum_{r=1}^{m}\sum_{k=1}^{n}\delta_{(rk)}^{T}
			\leq&2L+4L\sqrt{C_{\lambda}}+\frac{2L}{mn}\sum_{t=1}^{T}\frac{1}{t}\sum_{k=1}^{n}\mathbb{I}_{[j_{t+1}=k]}\\
			\leq&2L+4L\sqrt{C_{\lambda}}+\frac{2L\ln T}{mn}.
		\end{align*}
	\end{proof}
	\subsection{Average weight}\label{Average weight}
	It is easy to know that
	\begin{align*}
		\frac{1}{mn}\sum_{r=1}^{m}\sum_{k=1}^{n}\ls\bar{\mw}^{T}-\bar{\mw}_{(rk)}^{T}\rs=\frac{1}{mn}\sum_{r=1}^{m}\sum_{k=1}^{n}\ls\frac{\sum_{t=1}^{T}\eta_t\lp\mw^{t}-\mw_{(rk)}^{t}\rp}{\sum_{t=1}^{T}\eta_t}\rs
		\leq\frac{1}{mn}\sum_{r=1}^{m}\sum_{k=1}^{n}\frac{\sum_{t=1}^{T}\eta_t\ls\mw^{t}-\mw_{(rk)}^{t}\rs}{\sum_{t=1}^{T}\eta_t}.
	\end{align*}
	\textbf{Smooth Case:}
	\par According to the \cref{convex recursion}, we can get 
	\begin{equation*}
		\frac{1}{mn}\sum_{r=1}^{m}\sum_{k=1}^{n}\ls{\mw}^{t}-{\mw}_{(rk)}^{t}\rs\leq4\beta L\sum_{{s}=1}^{t}\eta_{s}\sum_{q=1}^{s-1}\eta_q\lambda^{s-q-1}+\frac{2L}{mn}\sum_{s=1}^{t}\eta_{s}.
	\end{equation*}
	Furthermore, 
	\begin{align*}
		\frac{1}{mn}\sum_{r=1}^{m}\sum_{k=1}^{n}\ls\bar{\mw}^{T}-\bar{\mw}_{(rk)}
		^{T}\rs\leq&\frac{1}{mn}\sum_{r=1}^{m}\sum_{k=1}^{n}\frac{\sum_{t=1}^{T}\eta_t\ls\mw^{t}-\mw_{(rk)}^{t}\rs}{\sum_{t=1}^{T}\eta_t}\\
		\leq&\frac{\sum_{t=1}^{T}\eta_t\lp4\beta L\sum_{s=1}^{t}\eta_{s-1}\sum_{q=1}^{s}\eta_q\lambda^{s-q-1}+\frac{2L}{mn}\sum_{s=1}^{t}\eta_{s}\rp}{\sum_{t=1}^{T}\eta_t}.
	\end{align*}
	If $\eta_t\equiv \eta\leq\frac{2}{L}$,
	we know 
	\begin{align*}
		\frac{1}{mn}\sum_{r=1}^{m}\sum_{k=1}^{n}\ls{\mw}^{t}-{\mw}_{(rk)}^{t}\rs\leq
		\frac{4\eta^2\beta Lt}{1-\lambda}+\frac{2\eta Lt}{mn}.
	\end{align*}
	Then,
	\begin{align*}
		\mathbb{E}_{\A}\lk\frac{1}{mn}\sum_{r=1}^{m}\sum_{k=1}^{n}
		\ls\bar{\mw}^{T}-\bar{\mw}_{(rk)}^{T}\rs\rk\leq&\frac{\eta\sum_{t=1}^{T}\lp\frac{4\eta^2\beta Lt}{1-\lambda}+\frac{2\eta Lt}{mn}\rp}{T\eta}\\
		\leq&\frac{2\eta^2\beta LT}{1-\lambda}+\frac{\eta LT}{mn}.
	\end{align*}
	Furthermore, we can get the generlization error
	\begin{align*}
		\mathbb{E}_{S,\A}\lk R(\bar{\mw}^{T})-R_S(\bar{\mw}^{T})\rk\leq\frac{2\eta^2\beta LT}{1-\lambda}+\frac{\eta LT}{mn}.
	\end{align*}
	\vspace{0.8cm}
	
	\noindent\textbf{Nonsmooth Case:}
	\par According to the \cref{recu}, we have
	\begin{align}
		\frac{1}{mn}\sum_{r=1}^{m}\sum_{k=1}^{n}\ls{\mw}^{t}-{\mw}_{(rk)}^{t}\rs\leq\sqrt{4L^2\sum_{s=1}^{t}\eta_{s}^2+16L^2\sum_{s=1}^{t}\eta_{s}\sum_{q=1}^{s-1}\eta_q\lambda^{s-q-1}}+\frac{4L}{mn}\sum_{s=1}^{t-1}\eta_{s+1}
	\end{align}
	Furthermore, 
	\begin{align*}
		\frac{1}{mn}\sum_{r=1}^{m}\sum_{k=1}^{n}\ls\bar{\mw}^{T}-\bar{\mw}_{(rk)}
		^{T}\rs\leq&\frac{1}{mn}\sum_{r=1}^{m}\sum_{k=1}^{n}\frac{\sum_{t=1}^{T}\eta_t\ls\mw^{t}-\mw_{(rk)}^{t}\rs}{\sum_{t=1}^{T}\eta_t}\\
		\leq&\frac{\sum_{t=1}^{T}\eta_t\lp\sqrt{4L^2\sum_{s=1}^{t}\eta_{s}^2+16L^2\sum_{s=1}^{t}\eta_{s}\sum_{q=1}^{s-1}\eta_q\lambda^{s-q-1}}+\frac{4L}{mn}\sum_{s=1}^{t-1}\eta_{s+1}\rp}{\sum_{t=1}^{T}\eta_t}.
	\end{align*}
	If $\eta_t\equiv \eta$, we know 
	\begin{align*}
		\mathbb{E}_{\A}\lk\frac{1}{mn}\sum_{r=1}^{m}\sum_{k=1}^{n}
		\ls\bar{\mw}^{T}-\bar{\mw}_{(rk)}^{T}\rs\rk\leq&\frac{\eta\sum_{t=1}^{T}\lp2L\eta\sqrt{t}+\frac{4\eta L\sqrt{t}}{\sqrt{1-\lambda}}+\frac{4\eta Lt}{mn}\rp}{T\eta}\\
		\leq&2\eta L\sqrt{T}+\frac{4\eta L\sqrt{T}}{\sqrt{1-\lambda}}+\frac{4\eta LT}{mn}.
	\end{align*}
	Moreover, we obtain the generlization error in nonsmooth case
	\begin{align*}
		\mathbb{E}_{S,\A}\lk R(\bar{\mw}^{T})-R_S(\bar{\mw}^{T})\rk\leq2\eta L\sqrt{T}+\frac{4\eta L\sqrt{T}}{\sqrt{1-\lambda}}+\frac{4\eta LT}{mn}.
	\end{align*}
	\section{Stability Bound : Gradient-then-Consensus (GtC)}
	Consider the output for the GtC algorithm,
	\begin{align*}
		{\mw}^{t}=\frac{1}{m}\sum_{i=1}^{m}\mw^{t}(i)=&\frac{1}{m}\sum_{i=1}^{m}\lk\sum_{l=1}^{m}P_{i,l}\lp\mw^{t-1}(l)-\eta_{t}\nabla f\lp\mw^{t-1}(l);Z_{j_{t}(l)}\rp\rp\rk\\
		\overset{\textbf{(a)}}{=}&\frac{1}{m}\sum_{l=1}^{m}\mw^{t-1}(l)-\frac{\eta_{t}}{m}\sum_{l=1}^{m}\nabla f\lp\mw^{t-1}(l);Z_{j_{t}(l)}\rp\\
		=&{\mw}^{t-1}-\frac{\eta_{t}}{m}\sum_{l=1}^{m}\nabla f\lp\mw^{t-1}(l);Z_{j_{t}(l)}\rp,
	\end{align*}
	where Equation $\textbf{(a)}$ uses the fact of gossip matrix property.
	\begin{theorem}[Stability Bound]
		Suppose that $f(\mw;Z)$ is convex, L-Lipshitz and $\beta$-smooth. If 
		$\eta\leq2/\beta$, then for DMc-SGD (GtC) with $T$ iterations, we have
			\begin{align*}
				\frac{1}{mn}\sum_{r=1}^{m}\sum_{k=1}^{n}\mathbb{E}\left[\left\|\A(S)-\A(S_{rk})\right\|_2\right]\leq\frac{2L}{mn}\sum_{t=1}^{T}\eta_t.
			\end{align*}
	\end{theorem}
	\begin{proof}
		\begin{align*}
			\ls{\mw}^{t}-{\mw}_{(rk)}^{t}\rs
			=&\ls\frac{1}{m}\sum_{i=1}^{m}\mw^{t}(i)-\frac{1}{m}\sum_{i=1}^{m}{\mw}_{(rk)}^{t}(i)\rs\\
			=&\frac{1}{m}\sum_{i=1}^{m}\ls\sum_{l=1}^{m}P_{i,l}\lp\mw^{t-1}(l)-\eta_t\nabla f\lp\mw^{t-1}(l);Z_{j_t(l)}\rp\rp-\sum_{l=1}^{m}P_{i,l}\lp{\mw}_{(rk)}^{t-1}(l)-\eta_t\nabla f\lp{\mw}_{(rk)}^{t-1}(l);\tilde{Z}_{j_t(l)}\rp\rp\rs\\
			=&\frac{1}{m}\sum_{l=1}^{m}\ls\lp\mw^{t-1}(l)-\eta_t\nabla f\lp\mw^{t-1}(l);Z_{j_t(l)}\rp\rp-\lp{\mw}_{(rk)}^{t-1}(l)-\eta_t\nabla f\lp{\mw}_{(rk)}^{t-1}(l);\tilde{Z}_{j_t(l)}\rp\rp\rs.
		\end{align*}
		If $j_t\neq k$, then 
		\begin{align*}
			\ls{\mw}^{t}-{\mw}_{(rk)}^{t}\rs=&\frac{1}{m}\sum_{l=1}^{m}\ls\lp\mw^{t-1}(l)-\eta_t\nabla f\lp\mw^{t-1}(l);Z_{j_t(l)}\rp\rp-\lp{\mw}_{(rk)}^{t-1}(l)-\eta_t\nabla f\lp{\mw}_{(rk)}^{t-1}(l);{Z}_{j_t(l)}\rp\rp\rs\\
			\leq&\frac{1}{m}\sum_{l=1}^{m}\ls\mw^{t-1}(l)-\mw_{(rk)}^{t-1}(l)\rs=\ls{\mw}^{t-1}-{\mw}_{(rk)}^{t-1}\rs.
		\end{align*}
		If $j_t=k$, then
		\begin{align*}
			\ls{\mw}^{t}-{\mw}_{(rk)}^{t}\rs=&\frac{1}{m}\sum_{l=1}^{m}\ls\lp\mw^{t-1}(l)-\eta_t\nabla f\lp\mw^{t-1}(l);Z_{j_t(l)}\rp\rp-\lp{\mw}_{(rk)}^{t-1}(l)-\eta_t\nabla f\lp{\mw}_{(rk)}^{t-1}(l);\tilde{Z}_{j_t(l)}\rp\rp\rs\\
			\leq&\frac{1}{m}\sum_{l=1}^{m}\ls\mw^{t-1}(l)-\mw_{(rk)}^{t-1}(l)\rs
			+\frac{\eta_{t}}{m}\ls\nabla f\lp\mw^{t-1}(r);Z_{k(r)}\rp-\nabla f\lp{\mw}_{(rk)}^{t-1}(r);\tilde{Z}_{k(r)}\rp\rs\\
			=&\ls{\mw}^{t-1}-{\mw}_{(rk)}^{t-1}\rs+\frac{2\eta_tL}{m}.
		\end{align*}
		Combine the two case, we have
		\begin{align*}
			\ls{\mw}^{t}-{\mw}_{(rk)}^{t}\rs\leq\ls{\mw}^{t-1}-{\mw}_{(rk)}^{t-1}\rs+\frac{2\eta_tL}{m}\mathbb{I}_{[j_t=k]}.
		\end{align*} 
		Apply the above inequality recursively,
		\begin{align*}
			\ls{\mw}^{T}-{\mw}_{(rk)}^{T}\rs\leq\frac{2L}{m}\sum_{t=1}^{T}\eta_t\mathbb{I}_{[j_t=k]}.
		\end{align*}
		Taking average about $k$ and $r$,
		\begin{align*}
			\frac{1}{mn}\sum_{r=1}^{m}\sum_{k=1}^{n}\ls{\mw}^{T}-{\mw}_{(rk)}^{T}\rs\leq\frac{2L}{mn}\sum_{t=1}^{T}\sum_{k=1}^{n}\eta_t\mathbb{I}_{[j_t=k]}\leq\frac{2L}{mn}\sum_{t=1}^{T}\eta_t,
		\end{align*}
		where the last inequality uses $\sum_{k=1}^{n}\mathbb{I}_{[j_t=k]}=1$.\\
	\end{proof}
	\section{Optimization Error of DMc-SGD}
	\begin{theorem}[Convex Case]
		Consider DMc-SGD applied to a convex loss $f(w;Z)$ under Assumptions 1–3. Let the algorithm run for T iterations from $\mw^0=0$ with the stepsize $\eta_{t}\leq2/\beta$. Denote $D_0=\ls\mw_S^*\rs$ and $D=\lk\sum_{s=1}^{T}\eta_{s}\lp L^2+2L^2\sum_{q=1}^{s-1}\eta_q\lambda^{t-q-1}+2\sup_{Z\in\mathcal{Z}}f(0,Z)\rp\rk^{\frac{1}{2}}+D_0$, and define the truncation window
		\begin{align*}
			\mathcal{T}_t=\min\left\{\max\left\{\lceil\frac{\log\lp2C_HDnt\rp}{\log\lp1/\lambda(H)\rp}\rceil,K_H\right\},t\right\},~t\in[T].
		\end{align*}
		Then we can obtain that
		\begin{align*}
			\mathbb{E}_{\A}\lk R_S(\bar{\mw}^{T})-R_S(\mw_S^*)\rk\leq&\frac{L^2\sum_{t=1}^{T}\eta_t\lp\sum_{q=t-\mathcal{T}_t+1}^{t}\eta_q+\sum_{q=t-\mathcal{T}_t+1}^{t-1}\eta_q\rp}{\sum_{t=1}^{T}\eta_t}+\frac{\ls\mw_S^*\rs^2+4LD\sum_{t=1}^{K_H-1}\eta_t}{2\sum_{t=1}^{T}\eta_t}\\
			&+\frac{2D\beta L\sum_{t=1}^{T}\eta_{t}\sum_{q=1}^{t-1}\eta_q\lambda^{t-q-1}+\sum_{t=K_H}^{T}\frac{L\eta_t}{2t}}{\sum_{t=1}^{T}\eta_t}+\frac{L^2\sum_{t=1}^{T}\eta_{t}^2}{2\sum_{t=1}^{T}\eta_t}.
		\end{align*}
		Moreover, when Assumption 4 holds. If $\eta_t=\eta=1/\sqrt{T\log T}$, the bound simplifies to
		\begin{align*}
			\mathbb{E}_{\A}\lk R_S(\bar{\mw}^{T})-R_S(\mw_S^*)\rk=\mathcal{O}\lp\frac{\sqrt{\log T}}{\sqrt{T}\log(1/\lambda(H))}+\frac{1}{(1-\lambda)\sqrt{T\log T}}\rp.
		\end{align*}
	\end{theorem}
	\begin{proof}
		We start from the standard convexity argument for weighted averaging. Since $R_S(\cdot)$ is convex, 
		\begin{align}
			\lp\sum_{t=1}^{T}\eta_t\rp\mathbb{E}_{\A}\lk R_S(\bar{\mw}^{T})-R_S(\mw_S^*)\rk
			\leq&\sum_{t=1}^{T}\eta_t\mathbb{E}_{\A}\lk R_S(\mw^t)-R_S(\mw_S^*)\rk\nonumber\\
			\leq&\sum_{t=1}^{T}\eta_t\mathbb{E}_{\A}\lk R_S(\mw^t)-R_S(\mw^{t-\mathcal{T}_t})\rk+\sum_{t=1}^{T}\eta_t\mathbb{E}_{\A}\lk R_S(\mw^{t-\mathcal{T}_t})-R_S(\mw_S^{*})\rk,
		\end{align}
		where we introduce a time shift by $\mathcal{T}_t=\min\left\{\max\left\{\lceil\frac{\log\lp2C_HDnt\rp}{\log\lp1/\lambda(H)\rp}\rceil,K_H\right\},t\right\}$.
		
		\noindent\textbf{Step 1: controlling the drift term}
		
		For the first term, we have
		\begin{align*}
			\sum_{t=1}^{T}\eta_t\mathbb{E}_{\A}\lk R_S(\mw^t)-R_S(\mw^{t-\mathcal{T}_t})\rk=&L	\sum_{t=1}^{T}\eta_t\mathbb{E}_{\A}\lk\ls \mw^t-\mw^{t-\mathcal{T}_t}\rs\rk\\
			\leq&L	\sum_{t=1}^{T}\eta_t\mathbb{E}_{\A}\lk\sum_{q=t-\mathcal{T}_t+1}^{t}\eta_q\ls\frac{1}{m}\sum_{i=1}^{m}\nabla f\lp\mw^q(i);Z_{j_{q}(i)}\rp\rs\rk\\
			\leq&L^2\sum_{t=1}^{T}\eta_t\sum_{q=t-\mathcal{T}_t+1}^{t}\eta_q.
		\end{align*}
		\textbf{Step 2: a stationary–deviation decomposition for the Markov term}
		
		For the second term, we can estimate that
		\begin{align*}
			&\mathbb{E}_{j_t}\lk\frac{1}{m}\sum_{r=1}^{m}\lp f\lp\mw^{t-\mathcal{T}_t};Z_{j_{t}(r)}\rp-f\lp\mw_S^{*};Z_{j_{t}(r)}\rp\rp\mid\mw^0,\cdots,\mw^{t-\mathcal{T}_t},Z_{j_{1}(r)},\cdots,Z_{j_{t-\mathcal{T}_t}(r)}\rk\\
			=&\frac{1}{m}\sum_{k=1}^{n}\lk\sum_{r=1}^{m}\lp f\lp\mw^{t-\mathcal{T}_t};Z_{j_{t}(r)}\rp-f\lp\mw_S^{*};Z_{j_{t}(r)}\rp\rp\Pr\lp j_{t}(r)=k\mid j_{t-\mathcal{T}_t}(r)=k\rp\rk\\
			=&\frac{1}{m}\sum_{k=1}^{n}\lk\sum_{r=1}^{m}\lp f\lp\mw^{t-\mathcal{T}_t};Z_{j_{t}(r)}\rp-f\lp\mw_S^{*};Z_{j_{t}(r)}\rp\rp\lk H^{\mathcal{T}_t}\rk_{j_{t-\mathcal{T}_t}(r),k}\rk\\
			=&\frac{1}{m}\sum_{k=1}^{n}\lk\sum_{r=1}^{m}\lp f\lp\mw^{t-\mathcal{T}_t};Z_{j_{t}(r)}\rp-f\lp\mw_S^{*};Z_{j_{t}(r)}\rp\rp\lp\lk H^{\mathcal{T}_t}\rk_{j_{t-\mathcal{T}_t}(r),k}-\frac{1}{n}\rp\rk\\
			&+\frac{1}{mn}\sum_{k=1}^{n}\lk\sum_{r=1}^{m}\lp f\lp\mw^{t-\mathcal{T}_t};Z_{j_{t}(r)}\rp-f\lp\mw_S^{*};Z_{j_{t}(r)}\rp\rp\rk\\
			=&\frac{1}{m}\sum_{k=1}^{n}\lk\sum_{r=1}^{m}\lp f\lp\mw^{t-\mathcal{T}_t};Z_{j_{t}(r)}\rp-f\lp\mw_S^{*};Z_{j_{t}(r)}\rp\rp\lp\lk H^{\mathcal{T}_t}\rk_{j_{t-\mathcal{T}_t}(r),k}-\frac{1}{n}\rp\rk+R_S\lp \mw^{t-\mathcal{T}_t}\rp-R_S(\mw_S^{*}).
		\end{align*}
		Rearranging the above equality,
		\begin{align*}
			\mathbb{E}_{\A}\lk R_S\lp \mw^{t-\mathcal{T}_t}\rp-R_S(\mw_S^{*})\rk=&\mathbb{E}_{\A}\lk\frac{1}{m}\sum_{r=1}^{m}\lp f\lp\mw^{t-\mathcal{T}_t};Z_{j_{t}(r)}\rp-f\lp\mw_S^{*};Z_{j_{t}(r)}\rp\rp\rk\\
			&+\mathbb{E}_{\A}\lk\frac{1}{m}\sum_{k=1}^{n}\lk\sum_{r=1}^{m}\lp f\lp\mw^{t-\mathcal{T}_t};Z_{j_{t}(r)}\rp-f\lp\mw_S^{*};Z_{j_{t}(r)}\rp\rp\lp\frac{1}{n}-\lk H^{\mathcal{T}_t}\rk_{j_{t-\mathcal{T}_t}(r),k}\rp\rk\rk.
		\end{align*}
		Summing over $t$ yields
		\begin{align*}
			&\sum_{t=1}^{T}\eta_t\mathbb{E}_{\A}\lk R_S\lp \mw^{t-\mathcal{T}_t}\rp-R_S(\mw_S^{*})\rk\\
			=&\underbrace{\sum_{t=1}^{T}\eta_t\mathbb{E}_{\A}\lk\frac{1}{m}\sum_{r=1}^{m}\lp f\lp\mw^{t-\mathcal{T}_t};Z_{j_{t}(r)}\rp-f\lp\mw_S^{*};Z_{j_{t}(r)}\rp\rp\rk}_{\Im}\\
			&+\underbrace{\sum_{t=1}^{T}\eta_t\mathbb{E}_{\A}\lk\frac{1}{m}\sum_{k=1}^{n}\lk\sum_{r=1}^{m}\lp f\lp\mw^{t-\mathcal{T}_t};Z_{j_{t}(r)}\rp-f\lp\mw_S^{*};Z_{j_{t}(r)}\rp\rp\lp\frac{1}{n}-\lk H^{\mathcal{T}_t}\rk_{j_{t-\mathcal{T}_t}(r),k}\rp\rk\rk}_{\wp}.
		\end{align*}
		\textbf{Step 3: bounding the “stationary part” via a distance recursion}
		
		Estimate the $\Im$ term,
		\begin{align*}
			\ls\mw^t-\mw_S^{*}\rs^2
			\leq&\ls\mw^{t-1}-\frac{\eta_{t}}{m}\sum_{i=1}^{m}\nabla f(\mw^{t-1}(i);Z_{j_{t}(i)})-\mw_S^{*}\rs^2\\
			\leq&\ls\mw^{t-1}-\mw_S^{*}\rs^2-\frac{2\eta_{t}}{m}\sum_{i=1}^{m}\left\langle\mw^{t-1}-\mw_S^{*},\nabla f(\mw^{t-1}(i);Z_{j_{t}(i)})\right\rangle+\frac{\eta_{t}^2}{m^2}\ls\sum_{i=1}^{m}\nabla f(\mw^{t-1}(i);Z_{j_{t}(i)})\rs^2\\
			\leq&\ls\mw^{t-1}-\mw_S^{*}\rs^2-\frac{2\eta_{t}}{m}\sum_{i=1}^{m}\left\langle\mw^{t-1}-\mw_S^{*},\nabla f(\mw^{t-1}(i);Z_{j_{t}(i)})-\nabla f\lp\mw^{t-1};Z_{j_{t}(i)}\rp\right\rangle\\
			&-\frac{2\eta_{t}}{m}\sum_{i=1}^{m}\left\langle\mw^{t-1}-\mw_S^{*},\nabla f\lp\mw^{t-1};Z_{j_{t}(i)}\rp\right\rangle+\frac{\eta_{t}^2}{m^2}\ls\sum_{i=1}^{m}\nabla f(\mw^{t-1}(i);Z_{j_{t}(i)})\rs^2\\
			\leq&\ls\mw^{t-1}-\mw_S^{*}\rs^2+\frac{2D\beta\eta_{t}}{m}\sum_{i=1}^{m}\ls\mw^{t-1}(i)-\mw^{t-1}\rs-\frac{2\eta_{t}}{m}\sum_{i=1}^{m}\lp f\lp \mw^{t-1};Z_{j_{t}(i)}\rp-f\lp \mw_S^{*};Z_{j_{t}(i)}\rp\rp\\
			&+{\eta_{t}^2}L^2\\
			\leq&\ls\mw^{t-1}-\mw_S^{*}\rs^2+\frac{2D\beta\eta_{t}}{\sqrt{m}}\sum_{i=1}^{m}\lk\ls\mw^{t-1}(i)-\mw^{t-1}\rs^2\rk^{\frac{1}{2}}-\frac{2\eta_{t}}{m}\sum_{i=1}^{m}\lp f\lp \mw^{t-\mathcal{T}_t};Z_{j_{t}(i)}\rp-f\lp \mw_S^{*};Z_{j_{t}(i)}\rp\rp\\
			&+\frac{2\eta_{t}}{m}\sum_{i=1}^{m}\lp f\lp \mw^{t-\mathcal{T}_t};Z_{j_{t}(i)}\rp-f\lp \mw^{t-1};Z_{j_{t}(i)}\rp\rp+{\eta_{t}^2}L^2\\
			\leq&\ls\mw^{t-1}-\mw_S^{*}\rs^2+{4D\beta L\eta_{t}}\sum_{q=1}^{t-1}\eta_q\lambda^{t-q-1}-\frac{2\eta_{t}}{m}\sum_{i=1}^{m}\lp f\lp \mw^{t-\mathcal{T}_t};Z_{j_{t}(i)}\rp-f\lp \mw_S^{*};Z_{j_{t}(i)}\rp\rp\\
			&+{2L\eta_{t}}\ls\mw^{t-\mathcal{T}_t}-\mw^{t-1}\rs+{\eta_{t}^2}L^2\\
			\leq&\ls\mw^{t-1}-\mw_S^{*}\rs^2+{4D\beta L\eta_{t}}\sum_{q=1}^{t-1}\eta_q\lambda^{t-q-1}-\frac{2\eta_{t}}{m}\sum_{i=1}^{m}\lp f\lp \mw^{t-\mathcal{T}_t};Z_{j_{t}(i)}\rp-f\lp \mw_S^{*};Z_{j_{t}(i)}\rp\rp\\
			&+{2L^2\eta_{t}}\sum_{q=t-\mathcal{T}_t+1}^{t-1}\eta_q+{\eta_{t}^2}L^2.
		\end{align*}
		Taking a summation of the both sides over $t$, we can obtain that
		\begin{align*}
			\frac{1}{m}\sum_{i=1}^{m}\sum_{t=1}^{T}\eta_{t}\lp f\lp \mw^{t-\mathcal{T}_t};Z_{j_{t}(i)}\rp-f\lp \mw_S^{*};Z_{j_{t}(i)}\rp\rp
			\leq\frac{1}{2}\ls\mw_S^*\rs^2+2D\beta L\sum_{t=1}^{T}\eta_{t}\sum_{q=1}^{t-1}\eta_q\lambda^{t-q-1}+L^2\sum_{t=1}^{T}\eta_{t}\sum_{q=t-\mathcal{T}_t+1}^{t-1}\eta_q+\frac{L^2}{2}\sum_{t=1}^{T}\eta_{t}^2.
		\end{align*}
		Back to the matrix difference term, according to the \cref{markov},
		\begin{align*}
			\Bigg|\frac{1}{n}-\lk H^{\mathcal{T}_t}\rk_{k,k^{\prime}}\Bigg|\leq\frac{1}{2Dnt}.
		\end{align*}
		Based on the update rule,
		\begin{align*}
			\ls\mw^t\rs^2\leq&\ls\mw^{t-1}-\frac{\eta_{t}}{m}\sum_{i=1}^{m}\nabla f\lp\mw^{t-1}(i);Z_{j_t(i)}\rp\rs^2\\
			\leq&\ls\mw^{t-1}\rs^2+\underbrace{\frac{\eta_{t}^2}{m^2}\ls\sum_{i=1}^{m}\nabla f\lp\mw^{t-1}(i);Z_{j_t(i)}\rp\rs^2-\frac{2\eta_{t}}{m}\left\langle\sum_{i=1}^{m}\nabla f\lp\mw^{t-1}(i);Z_{j_t(i)}\rp,\mw^{t-1}\right\rangle}_{\aleph}.
		\end{align*}
		By the convexity of $f$, we arrive at
		\begin{align*}
			\frac{\aleph}{\eta_{t}}\leq&\eta_{t}L^2-\frac{2}{m}\left\langle\sum_{i=1}^{m}\nabla f\lp\mw^{t-1}(i);Z_{j_t(i)}\rp,\mw^{t-1}-\mw^{t-1}(i)\right\rangle-\frac{2}{m}\left\langle\sum_{i=1}^{m}\nabla f\lp\mw^{t-1}(i);Z_{j_t(i)}\rp,\mw^{t-1}(i)\right\rangle\\
			\leq&\eta_{t}L^2+\frac{2L}{m}\sum_{i=1}^{m}\ls\mw^{t-1}-\mw^{t-1}(i)\rs+\frac{2}{m}\sum_{i=1}^{m}\lp f(0;Z_{j_t(i)})-f(\mw^{t-1}(i);Z_{j_t(i)})\rp\\
			\leq&L^2+2L^2\sum_{q=1}^{t-1}\eta_q\lambda^{t-q-1}+2\sup_{Z\in\mathcal{Z}}f(0,Z).
		\end{align*}
		Hence, one obtains
		\begin{align*}
			\ls\mw^t\rs^2\leq\ls\mw^{t-1}\rs^2+L^2\eta_{t}+2L^2\eta_{t}\sum_{q=1}^{t-1}\eta_q\lambda^{t-q-1}+2\sup_{Z\in\mathcal{Z}}f(0,Z)\eta_{t}.
		\end{align*}
		Applying the above inequality recursively,
		\begin{align*}
			\ls\mw^t\rs^2\leq& 
			L^2\sum_{s=1}^{t}\eta_{s}+2L^2\sum_{s=1}^{t}\eta_{s}\sum_{q=1}^{s-1}\eta_q\lambda^{t-q-1}+2\sup_{Z\in\mathcal{Z}}f(0,Z)\sum_{s=1}^{t}\eta_{s}\\
			\leq&\sum_{s=1}^{T}\eta_{s}\lp L^2+2L^2\sum_{q=1}^{s-1}\eta_q\lambda^{t-q-1}+2\sup_{Z\in\mathcal{Z}}f(0,Z)\rp.
		\end{align*}
		Recall that $D_0=\ls\mw_S^*\rs$ and $D=\lk\sum_{s=1}^{T}\eta_{s}\lp L^2+2L^2\sum_{q=1}^{s-1}\eta_q\lambda^{t-q-1}+2\sup_{Z\in\mathcal{Z}}f(0,Z)\rp\rk^{\frac{1}{2}}+D_0$,
		\begin{align*}
			&\sum_{t=K_p}^{T}\eta_t\lk\frac{1}{m}\sum_{k=1}^{n}\lk\sum_{r=1}^{m}\lp f\lp\mw^{t-\mathcal{T}_t};Z_{j_{t}(r)}\rp-f\lp\mw_S^{*};Z_{j_{t}(r)}\rp\rp\lp\frac{1}{n}-\lk H^{\mathcal{T}_t}\rk_{j_{t-\mathcal{T}_t}(r),k}\rp\rk\rk\\
			\leq&LD\sum_{t=K_{H}}^{T}\eta_t\sum_{k=1}^{n}\lp\frac{1}{n}-\lk H^{\mathcal{T}_t}\rk_{j_{t-\mathcal{T}_t}(r),k}\rp
			\leq\sum_{t=K_{H}}^{T}\frac{L\eta_t}{2t}.
		\end{align*}
		In addition,
		\begin{align*}
			&\frac{1}{m}\sum_{k=1}^{n}\lk\sum_{r=1}^{m}\lp f\lp\mw^{t-\mathcal{T}_t};Z_{j_{t}(r)}\rp-f\lp\mw_S^{*};Z_{j_{t}(r)}\rp\rp\lp\frac{1}{n}-\lk H^{\mathcal{T}_t}\rk_{j_{t-\mathcal{T}_t}(r),k}\rp\rk\\
			\leq&\frac{1}{m}\sum_{k=1}^{n}\lk\sum_{r=1}^{m}\lp f\lp\mw^{t-\mathcal{T}_t};Z_{j_{t}(r)}\rp-f\lp\mw_S^{*};Z_{j_{t}(r)}\rp\rp\lp\frac{1}{n}+\lk H^{\mathcal{T}_t}\rk_{j_{t-\mathcal{T}_t}(r),k}\rp\rk\\
			\leq&2L\ls\mw^{t-\mathcal{T}_t}-\mw^*\rs\leq2LD.
		\end{align*}
		Then we get
		\begin{align*}
			\sum_{t=1}^{K_H-1}\eta_t\lk\frac{1}{m}\sum_{k=1}^{n}\lk\sum_{r=1}^{m}\lp f\lp\mw^{t-\mathcal{T}_t};Z_{j_{t}(r)}\rp-f\lp\mw_S^{*};Z_{j_{t}(r)}\rp\rp\lp\frac{1}{n}-\lk H^{\mathcal{T}_t}\rk_{j_{t-\mathcal{T}_t}(r),k}\rp\rk\rk
			\leq2LD\sum_{t=1}^{K_H-1}\eta_t.
		\end{align*}
		Combine the above inequality, we have
		\begin{align*}
			\sum_{t=1}^{T}\eta_t\lk\frac{1}{m}\sum_{k=1}^{n}\lk\sum_{r=1}^{m}\lp f\lp\mw^{t-\mathcal{T}_t};Z_{j_{t}(r)}\rp-f\lp\mw_S^{*};Z_{j_{t}(r)}\rp\rp\lp\frac{1}{n}-\lk H^{\mathcal{T}_t}\rk_{j_{t-\mathcal{T}_t}(r),k}\rp\rk\rk
			\leq&2LD\sum_{t=1}^{K_H-1}\eta_t+\sum_{t=K_H}^{T}\frac{L\eta_t}{2t}.
		\end{align*}
		Back to the initial term,
		\begin{align*}
			&\sum_{t=1}^{T}\eta_t\mathbb{E}_{\A}\lk R_S\lp \mw^{t-\mathcal{T}_t}\rp-R_S(\mw_S^{*})\rk\\
			\leq&\frac{1}{2}\ls\mw^*\rs^2+4D\beta L\sum_{t=1}^{T}\eta_{t}\sum_{q=1}^{t-1}\eta_q\lambda^{t-q-1}+L^2\sum_{t=1}^{T}\eta_{t}\sum_{q=t-\mathcal{T}_t+1}^{t-1}\eta_q+\frac{L^2}{2}\sum_{t=1}^{T}\eta_{t}^2+2LD\sum_{t=1}^{K_H-1}\eta_t+\sum_{t=K_H}^{T}\frac{L\eta_t}{2t}.
		\end{align*}
		Combine the two term,
		\begin{align*}
			\lp\sum_{t=1}^{T}\eta_t\rp\mathbb{E}_{\A}\lk R_S(\bar{\mw}^{T})-R_S(\mw^*)\rk
			\leq&\sum_{t=1}^{T}\eta_t\mathbb{E}_{\A}\lk R_S(\mw^t)-R_S(\mw^{t-\mathcal{T}_t})\rk+\sum_{t=1}^{T}\eta_t\mathbb{E}_{\A}\lk R_S(\mw^{t-\mathcal{T}_t})-R_S(\mw_S^{*})\rk\\
			\leq&L^2\sum_{t=1}^{T}\eta_t\sum_{q=t-\mathcal{T}_t+1}^{t}\eta_q+\frac{1}{2}\ls\mw^*\rs^2+2D\beta L\sum_{t=1}^{T}\eta_{t}\sum_{q=1}^{t-1}\eta_q\lambda^{t-q-1}\\
			&+L^2\sum_{t=1}^{T}\eta_{t}\sum_{q=t-\mathcal{T}_t+1}^{t-1}\eta_q+\frac{L^2}{2}\sum_{t=1}^{T}\eta_{t}^2+2LD\sum_{t=1}^{K_H-1}\eta_t+\sum_{t=K_H}^{T}\frac{L\eta_t}{2t}.
		\end{align*}
		Thus, we have
		\begin{align*}
			\mathbb{E}_{\A}\lk R_S(\bar{\mw}^{T})-R_S(\mw_S^*)\rk\leq&\frac{L^2\sum_{t=1}^{T}\eta_t\lp\sum_{q=t-\mathcal{T}_t+1}^{t}\eta_q+\sum_{q=t-\mathcal{T}_t+1}^{t-1}\eta_q\rp}{\sum_{t=1}^{T}\eta_t}+\frac{\ls\mw^*\rs^2+4LD\sum_{t=1}^{K_H-1}\eta_t}{2\sum_{t=1}^{T}\eta_t}\\
			&+\frac{2D\beta L\sum_{t=1}^{T}\eta_{t}\sum_{q=1}^{t-1}\eta_q\lambda^{t-q-1}+\sum_{t=K_H}^{T}\frac{L\eta_t}{2t}}{\sum_{t=1}^{T}\eta_t}+\frac{L^2\sum_{t=1}^{T}\eta_{t}^2}{2\sum_{t=1}^{T}\eta_t}.
		\end{align*}
		Furthermore, choosing $\eta_t=\eta=\frac{1}{\sqrt{T\log(T)}}$ and noting that $D=\mathcal{O}\lp\frac{\eta\sqrt{T}}{\sqrt{1-\lambda}}\rp$, we can get
		\begin{align}
			\mathbb{E}_{\A}\lk R_S(\bar{\mw}^{T})-R_S(\mw^*)\rk=&\mathcal{O}\lp\frac{\sum_{t=1}^{T}\mathcal{T}_t\eta^2+\frac{\eta^3T^{3/2}}{(1-\lambda)^{3/2}}+T\eta^2+\frac{K_H\eta^2\sqrt{T}}{\sqrt{1-\lambda}}}{T\eta}\rp\nonumber\\
			=&\mathcal{O}\lp\frac{\sqrt{\log T}\lp1+\sum_{t=1}^{T}\mathcal{T}_t\eta^2\rp}{\sqrt{T}}+\frac{\eta^2\sqrt{T}}{(1-\lambda)^{3/2}}+\frac{K_H\eta}{\sqrt{1-\lambda}\sqrt{T}}\rp\label{level}.
		\end{align}
		Define the truncated variable $J=\frac{1}{2C_H Dn\lambda(H)^{K_H}}$, if $t\leq J$, we have
		\begin{align*}
			\log (2C_HDnt)\leq\log (2C_HDnJ)=\log\lp\frac{1}{\lambda(H)^{K_H}}\rp=K_H\log\lp1/\lambda(H)\rp.
		\end{align*}
		It implies that
		\begin{align*}
			\frac{\log (2C_HDnt)}{\log\lp1/\lambda(H)\rp}\leq K_H,~~\mathcal{T}_t\leq K_H.
		\end{align*}
		We can get that
		\begin{align*}
			\sum_{t=1}^{J}\mathcal{T}_t\eta^2\leq JK_H\eta^2=\frac{K_H\sqrt{1-\lambda}}{T\sqrt{\log T}2C_H n\lambda(H)^{K_H}}.
		\end{align*}
		If $t>J$, we have $\mathcal{T}_t\leq\lceil\frac{\log\lp2C_HDnt\rp}{\log\lp1/\lambda(H)\rp}\rceil$ and
		\begin{align*}
			\sum_{t=J+1}^{T}\mathcal{T}_t\eta^2\leq&\sum_{t=J+1}^{T}\lp\frac{\log\lp2C_HD\rp}{\log\lp1/\lambda(H)\rp}+\frac{\log\lp n\rp}{\log\lp1/\lambda(H)\rp}+\frac{\log\lp t\rp}{\log\lp1/\lambda(H)\rp}+1\rp\eta^2\\
			\leq&\frac{1}{\log\lp1/\lambda(H)\rp}\lk(T-J)\eta^2\lp\log\lp2C_HD\rp+\log\lp n\rp\rp+\sum_{t=J+1}^{T}\log(t)\eta^2\rk+T\eta^2\\
			\leq&\mathcal{O}\lp\frac{1}{\log\lp1/\lambda(H)\rp}\lp1+\frac{\log n}{\log T}\rp\rp.
		\end{align*}
		With the assumption $T=\mathcal{O}\lp mn\rp$, we get
		\begin{align*}
			\sum_{t=1}^{T}\mathcal{T}_t\eta^2=\mathcal{O}\lp\frac{K_H\sqrt{1-\lambda}}{T\sqrt{\log T}C_H n\lambda(H)^{K_H}}+\frac{1}{\log(1/\lambda(H))}\rp.
		\end{align*}
		Back to the \cref{level}, 
		\begin{align*}
			\mathbb{E}_{\A}\lk R_S(\bar{\mw}^{T})-R_S(\mw^*)\rk
			=\mathcal{O}\lp\frac{\sqrt{\log T}}{\sqrt{T}\log(1/\lambda(H))}+\frac{1}{(1-\lambda)\sqrt{T}\log T}+\frac{K_H\sqrt{1-\lambda}}{T^{\frac{3}{2}}C_Hn\lambda(H)^{K_H}}+\frac{K_H}{\sqrt{1-\lambda}T\sqrt{\log T}}\rp.
		\end{align*}
		If $K_H=0$, then
		\begin{align*}
			\mathbb{E}_{\A}\lk R_S(\bar{\mw}^{T})-R_S(\mw^*)\rk=\mathcal{O}\lp\frac{\sqrt{\log T}}{\sqrt{T}\log(1/\lambda(H))}+\frac{1}{(1-\lambda)\sqrt{T}\log T}\rp.
		\end{align*}
	\end{proof}
	\begin{theorem}[Restate, Excess Risk for Smooth Case] Suppose that the loss function $f(w;Z)$ is convex, L-Lipschitz and $\beta$-Smooth, suppose that Assumption 3-4 hold. Let $\A$ denote the DMc-SGD with $T$ iterations, producing the sequence $\{w^t\}_{t=1}^{T}$. The algorithm is initialized at $\mw^0=0$ and $\eta_t=\eta\leq2/\beta$. If we select $T=\mathcal{O}\lp mn\rp$ and $\eta=\frac{1}{\sqrt{T\log T}}$, then
		\begin{align*}
			\mathbb{E}_{S,\A}\lk R(\bar{\mw}^{T})-R(\mw^*)\rk=\mathcal{O}\lp\frac{1}{(1-\lambda)\log mn }+\frac{\sqrt{\log mn}}{\sqrt{mn}\log(1/\lambda(H))}\rp.
		\end{align*}
	\end{theorem}
	\begin{proof}
		If the stepsize $\eta_t=\eta=\frac{1}{\sqrt{T\log T}}$ and $T=\mathcal{O}\lp mn\rp$, we have
		\begin{align*}
			&\mathbb{E}_{S,\A}\lk R(\bar{\mw}^{T})-R(\mw^*)\rk\\
			=&\mathbb{E}_{S,\A}\lk R(\bar{\mw}^{T})- R_S(\bar{\mw}^{T})\rk+\mathbb{E}_{S,\A}\lk R_S(\bar{\mw}^{T})-R_S(\mw_S^*)\rk\\
			=&\mathcal{O}\lp\frac{\eta^2 T}{1-\lambda}+\frac{\eta T}{mn}+\frac{\sqrt{\log T}\lp1+\sum_{t=1}^{T}\mathcal{T}_t\eta^2\rp}{\sqrt{T}}+\frac{\eta^2\sqrt{T}}{(1-\lambda)^{3/2}}+\frac{K_H\eta}{\sqrt{1-\lambda}\sqrt{T}}\rp\\
			=&\mathcal{O}\lp\frac{1}{(1-\lambda)\log T}+\frac{1}{\sqrt{T\log T}}+\frac{\sqrt{\log T}}{\sqrt{T}\log(1/\lambda(H))}+\frac{1}{(1-\lambda)\sqrt{T}\log T}+\frac{K_H}{\sqrt{1-\lambda}T\sqrt{\log T}}+\frac{K_H\sqrt{1-\lambda}}{T^{\frac{3}{2}}C_Hn\lambda(H)^{K_H}}\rp.
		\end{align*}
		If $K_H=0$, we can get
		\begin{align*}
			\mathbb{E}_{S,\A}\lk R(\bar{\mw}^{T})-R(\mw^*)\rk=&\mathcal{O}\lp\frac{1}{(1-\lambda)\log (mn) }+\frac{1}{\sqrt{mn\log (mn)}}+\frac{\sqrt{\log mn}}{\sqrt{mn}\log(1/\lambda(H))}+\frac{1}{(1-\lambda)\sqrt{mn}\log (mn)}\rp\\
			=&\mathcal{O}\lp\frac{1}{(1-\lambda)\log (mn)}+\frac{\sqrt{\log mn}}{\sqrt{mn}\log(1/\lambda(H))}\rp.
		\end{align*}
	\end{proof}
	\begin{theorem}[Convex Case (Non-smooth)]
		Suppose that $f(w;Z)$ is convex and L-Lipschitz, and that Assumption 3 holds. Let $\A$denote the DMc-SGD run for $T$ iterations, producing the sequence $\{\mw^t\}_{t=1}^{T}$. The algorithm is initialized at $\mw^0=0$, with a constant stepsize $\eta_t=\eta\leq2/L$. $D$ is defined as before.
		Then we can obtain that
		\begin{align*}
			\mathbb{E}_{\A}\lk R_S(\bar{\mw}^{T})-R_S(\mw_S^*)\rk\leq&\frac{L^2\sum_{t=1}^{T}\eta_t\lp\sum_{q=t-\mathcal{T}_t+1}^{t}\eta_q+\sum_{q=t-\mathcal{T}_t+1}^{t-1}\eta_q\rp}{\sum_{t=1}^{T}\eta_t}+\frac{\ls\mw^*\rs^2+4LD\sum_{t=1}^{K_H-1}\eta_t}{2\sum_{t=1}^{T}\eta_t}\\
			&+\frac{2DL\sum_{t=1}^{T}\eta_{t}+\sum_{t=K_H}^{T}\frac{L\eta_t}{2t}}{\sum_{t=1}^{T}\eta_t}+\frac{L^2\sum_{t=1}^{T}\eta_{t}^2}{2\sum_{t=1}^{T}\eta_t}.
		\end{align*}
		In addition, If Assumption 4 holds and $\eta_t=\eta=1/\sqrt{T\log T}$, we get
		\begin{align*}
			\mathbb{E}_{\A}\lk R_S(\bar{\mw}^{T})-R_S(\mw_S^*)\rk=\mathcal{O}\lp\frac{\sqrt{\log T}}{\sqrt{T}\log(1/\lambda(H))}+\frac{1}{(1-\lambda)\sqrt{\log T}}\rp.
		\end{align*}
	\end{theorem}
	\begin{remark}
		The proof of this theorem is omitted, since it closely parallels the analysis of the optimization error in the smooth setting. The only distinction arises in bounding the consensus error, where smoothness is no longer applicable and is instead replaced by Lipschitz-based arguments.
	\end{remark}
	\begin{theorem}[Restate, Excess Risk for Nonsmooth Case] Assume that the loss function $f(w;Z)$ is convex, L-Lipschitz , suppose that Assumptions 4 holds. Let $\A$ be DMc-SGD with $T$ iterations, and $\{w^t\}_{t=1}^{T}$ be produced by $\A$ with $\mw^0=0$ and $\eta_t=\eta$. If we select $T=\mathcal{O}\lp\frac{m^2n^2}{1-\lambda}\rp$ and $\eta=\frac{\sqrt{1-\lambda}}{T^{3/4}}$, then
		\begin{align*}
			\mathbb{E}_{S,\A}\lk R(\bar{\mw}^{T})-R(\mw^*)\rk=\mathcal{O}\lp\frac{\log(mn)}{(mn)^{3/2}(1-\lambda)^{1/4}\log(1/\lambda(H))}\rp.
		\end{align*}
	\end{theorem}
	\begin{proof}
		Similar with the smooth case, we get
		\begin{align*}
			\mathbb{E}_{S,\A}\lk R(\bar{\mw}^{T})-R(\mw^*)\rk
			=&\mathbb{E}_{S,\A}\lk R(\bar{\mw}^{T})- R_S(\bar{\mw}^{T})\rk+\mathbb{E}_{S,\A}\lk_S(\bar{\mw}^{T})-R_S(\mw_S^*)\rk\\
			=&\mathcal{O}\lp\eta \sqrt{T}+\frac{\eta \sqrt{T}}{\sqrt{1-\lambda}}+\frac{\eta T}{mn}+\frac{\lp T+\sum_{t=1}^{T}\mathcal{T}_t\rp\eta^2}{T\eta}+\frac{\eta\sqrt{T}}{\sqrt{1-\lambda}}+\frac{K_H\eta}{\sqrt{1-\lambda}\sqrt{T}}\rp.
		\end{align*}
		If $\eta=\frac{\sqrt{1-\lambda}}{T^{3/4}}$ and $T=\mathcal{O}\lp\frac{m^2n^2}{1-\lambda}\rp$, we can also derive that
		\begin{align*}
			\sum_{t=1}^{T}\mathcal{T}_t\eta^2=\mathcal{O}\lp\frac{K_H(1-\lambda)}{T^{5/4}2C_H n\lambda(H)^{K_H}}+\frac{\log T(1-\lambda)}{\sqrt{T}\log(1/\lambda(H))}\rp.
		\end{align*}
		Back to the excess risk bound, if $K_H=0$, we have
		\begin{align*}
			\mathbb{E}_{S,\A}\lk R(\bar{\mw}^{T})-R(\mw^*)\rk=\mathcal{O}\lp\frac{\log(mn)}{(mn)^{3/2}(1-\lambda)^{1/4}\log(1/\lambda(H))}\rp.
		\end{align*}
	\end{proof}
	\begin{theorem}[High-probability Bound]
		Assume that $f(\mw;Z)$ is convex and Assumptions 1,2 and 4 holds. Suppose that $\sup_{Z\in\mathcal{Z}}f(\mw;Z)\leq M$ for some $M>0$. Let $\alpha\in(0,1)$, then with probability at least $1-\alpha$,
		\begin{align*}
			R_S(\bar{\mw}^T)-R_S(\mw^*)=\mathcal{O}\lp\frac{\sqrt{\log T}}{\sqrt{T}\log(1/\lambda(H))}+\frac{M\sqrt{\log (1/\alpha)}}{\sqrt{T}}+\frac{1}{(1-\lambda)\sqrt{T\log T}}\rp.
		\end{align*}
	\end{theorem}
	\begin{proof}
		Suppose that $\xi_t=\eta_{t}\lk \frac{1}{m}\sum_{r=1}^{m}f(\mw^{t-\mathcal{T}_t};Z_{j_t(r)})-f(\mw_S^{*};Z_{j_t(r)})\rk$. Observe that $\mid\xi_t-\mathbb{E}_{j_t}[\xi_t]\mid\leq2M\eta_{t}$. Then applying \cref{mag}, with probability at least $1-\alpha$,
		\begin{align}\label{difff}
			\sum_{t=1}^{T}\mathbb{E}_{j_t}[\xi_t]-\sum_{t=1}^{T}\xi_t\leq2M\lp2\sum_{t=1}^{T}\eta_{t}^2\log(1/\alpha)\rp^{\frac{1}{2}}.
		\end{align}
		Then we consider that
		\begin{align*}
			&\frac{1}{m}\sum_{t=1}^{T}\sum_{k=1}^{n}\lk\sum_{r=1}^{m}\eta_{t}\lp f\lp\mw^{t-\mathcal{T}_t};Z_{j_{t}(r)}\rp-f\lp\mw_S^{*};Z_{j_{t}(r)}\rp\rp\lp\lk H^{\mathcal{T}_t}\rk_{j_{t-\mathcal{T}_t}(r),k}-\frac{1}{n}\rp\rk+\sum_{t=1}^{T}\eta_{t}\lp R_S\lp \mw^{t-\mathcal{T}_t}\rp-R_S(\mw_S^{*})\rp\\
			=&\sum_{t=1}^{T}\mathbb{E}_{j_t}\lk\frac{1}{m}\sum_{r=1}^{m}\eta_{t}\lp f\lp\mw^{t-\mathcal{T}_t};Z_{j_{t}(r)}\rp-f\lp\mw_S^{*};Z_{j_{t}(r)}\rp\rp\mid\mw^0,\cdots,\mw^{t-\mathcal{T}_t},Z_{j_{1}(r)},\cdots,Z_{j_{t-\mathcal{T}_t}(r)}\rk.
		\end{align*}
		Combine with \cref{difff},
		\begin{align*}
			&\frac{1}{m}\sum_{t=1}^{T}\sum_{k=1}^{n}\lk\sum_{r=1}^{m}\eta_{t}\lp f\lp\mw^{t-\mathcal{T}_t};Z_{j_{t}(r)}\rp-f\lp\mw_S^{*};Z_{j_{t}(r)}\rp\rp\lp\lk H^{\mathcal{T}_t}\rk_{j_{t-\mathcal{T}_t}(r),k}-\frac{1}{n}\rp\rk+\sum_{t=1}^{T}\eta_{t}\lp R_S\lp \mw^{t-\mathcal{T}_t}\rp-R_S(\mw_S^{*})\rp\\
			\leq&\sum_{t=1}^{T}\lk\frac{1}{m}\sum_{r=1}^{m}\eta_{t}\lp f\lp\mw^{t-\mathcal{T}_t};Z_{j_{t}(r)}\rp-f\lp\mw_S^{*};Z_{j_{t}(r)}\rp\rp\rk+2M\lp2\sum_{t=1}^{T}\eta_{t}^2\log(1/\alpha)\rp^{\frac{1}{2}}.
		\end{align*}
		Back to the initial term,
		\begin{align*}
			\sum_{t=1}^{T}\eta_t\lk R_S(\mw^{t-\mathcal{T}_t})-R_S(\mw_S^{*})\rk\leq&\frac{1}{m}\sum_{t=1}^{T}\sum_{k=1}^{n}\lk\sum_{r=1}^{m}\eta_{t}\lp f\lp\mw^{t-\mathcal{T}_t};Z_{j_{t}(r)}\rp-f\lp\mw_S^{*};Z_{j_{t}(r)}\rp\rp\lp\frac{1}{n}-\lk H^{\mathcal{T}_t}\rk_{j_{t-\mathcal{T}_t}(r),k}\rp\rk\\
			&+\sum_{t=1}^{T}\lk\frac{1}{m}\sum_{r=1}^{m}\eta_{t}\lp f\lp\mw^{t-\mathcal{T}_t};Z_{j_{t}(r)}\rp-f\lp\mw_S^{*};Z_{j_{t}(r)}\rp\rp\rk+2M\lp2\sum_{t=1}^{T}\eta_{t}^2\log(1/\alpha)\rp^{\frac{1}{2}}\\
			\leq&2LD\sum_{t=1}^{K_H-1}\eta_t+\sum_{t=K_H}^{T}\frac{L\eta_t}{2t}+\frac{1}{2}\ls\mw_S^*\rs^2+4D\beta L\sum_{t=1}^{T}\eta_{t}\sum_{q=1}^{t-1}\eta_q\lambda^{t-q-1}\\
			&+L^2\sum_{t=1}^{T}\eta_{t}\sum_{q=t-\mathcal{T}_t+1}^{t-1}\eta_q+\frac{L^2}{2}\sum_{t=1}^{T}\eta_{t}^2+2M\lp2\sum_{t=1}^{T}\eta_{t}^2\log(1/\alpha)\rp^{\frac{1}{2}}.
		\end{align*}
		With probability at least $1-\alpha$, there holds
		\begin{align*}
			\sum_{t=1}^{T}\eta_t\lk R_S(\bar{\mw}^T)-R_S(\mw^*)\rk     
			\leq& L^2\sum_{t=1}^{T}\eta_t\sum_{q=t-\mathcal{T}_t+1}^{t}\eta_q+2LD\sum_{t=1}^{K_H-1}\eta_t+\sum_{t=K_H}^{T}\frac{L\eta_t}{2t}+\frac{1}{2}\ls\mw_S^*\rs^2+4D\beta L\sum_{t=1}^{T}\eta_{t}\sum_{q=1}^{t-1}\eta_q\lambda^{t-q-1}\\
			&+L^2\sum_{t=1}^{T}\eta_{t}\sum_{q=t-\mathcal{T}_t+1}^{t-1}\eta_q+\frac{L^2}{2}\sum_{t=1}^{T}\eta_{t}^2+2M\lp2\sum_{t=1}^{T}\eta_{t}^2\log(1/\alpha)\rp^{\frac{1}{2}}.
		\end{align*}
		By Jensen’s inequality,
		\begin{align*}
			R_S(\bar{\mw}^T)-R_S(\mw^*)\leq&\frac{L^2\sum_{t=1}^{T}\eta_t\lp\sum_{q=t-\mathcal{T}_t+1}^{t}\eta_q+\sum_{q=t-\mathcal{T}_t+1}^{t-1}\eta_q\rp}{\sum_{t=1}^{T}\eta_t}+\frac{\ls\mw_S^*\rs^2+4LD\sum_{t=1}^{K_H-1}\eta_t}{2\sum_{t=1}^{T}\eta_t}\\
			&+\frac{4D\beta L\sum_{t=1}^{T}\eta_{t}\sum_{q=1}^{t-1}\eta_q\lambda^{t-q-1}+\sum_{t=K_H}^{T}\frac{L\eta_t}{2t}}{\sum_{t=1}^{T}\eta_t}+\frac{L^2\sum_{t=1}^{T}\eta_{t}^2}{2\sum_{t=1}^{T}\eta_t}+\frac{2M\lp2\sum_{t=1}^{T}\eta_{t}^2\log(1/\alpha)\rp^{\frac{1}{2}}}{\sum_{t=1}^{T}\eta_{t}}.
		\end{align*}
		Furthermore, choosing $\eta_t=\eta=\frac{1}{\sqrt{T\log(T)}}$, we can get
		\begin{align*}
			&R_S(\bar{\mw}^T)-R_S(\mw_S^*)\\
			=&\mathcal{O}\lp\frac{\sqrt{\log T}}{\sqrt{T}\log(1/\lambda(H))}+\frac{M\sqrt{\log (1/\alpha)}}{\sqrt{T}}+\frac{1}{(1-\lambda)\sqrt{T\log T}}+\frac{K_H\sqrt{1-\lambda}}{T^{\frac{3}{2}}C_Hn\lambda(H)^{K_H}}+\frac{K_H}{\sqrt{1-\lambda}T\sqrt{\log T}}\rp.
		\end{align*}
		If $K_H=0$, then
		\begin{align*}
			R_S(\bar{\mw}^T)-R_S(\mw_S^*)=\mathcal{O}\lp\frac{\sqrt{\log T}}{\sqrt{T}\log(1/\lambda(H))}+\frac{M\sqrt{\log (1/\alpha)}}{\sqrt{T}}+\frac{1}{(1-\lambda)\sqrt{T\log T}}\rp.
		\end{align*}
	\end{proof}
	\begin{theorem}[Non-convex Case]
		Assume that Assumption 1, 2 and 3 hold. Let $\mathcal{A}$ be MC-SGD with $T$ iterations and the weights $\{\mw^{t}\}_{t=1}^{T}$ be produced by $\A$. Let $r$ be the diameter of $\mathcal{W}$, we get
		\begin{align*}
			\min_{1\leq t\leq T}\mathbb{E}_{\A}\lk\ls\partial R_S\lp\mw^t\rp\rs^2\rk=\mathcal{O}\lp\frac{K_H}{T}+\frac{1}{\lp1-\lambda\rp\sqrt{T}\log T}+\frac{\log T}{\sqrt{T}}\lp\frac{K_p^2}{T\log^2 TC_H n\lambda(H)^{K_H}}+\frac{1}{\log^2(1/\lambda(H))}\rp\rp
		\end{align*}
		Furthermore, suppose that Assumption 4 holds, if the stepsize $\eta_{t}=\eta=\frac{1}{\log T\sqrt{T}}$, we have
		\begin{align*}
			\min_{1\leq t\leq T}\mathbb{E}_{\A}\lk\ls\partial R_S\lp\mw^t\rp\rs^2\rk=\mathcal{O}\lp\frac{1}{\lp1-\lambda\rp\sqrt{T}\log T}+\frac{\log T}{\sqrt{T}\log^2(1/\lambda(H))}\rp.
		\end{align*}
	\end{theorem}
	\begin{proof}
		We decompose the expression as follows:
		\begin{align}
			\sum_{t=1}^{T}\eta_{t}\mathbb{E}_{\A}\lk\ls\partial R_S\lp\mw^t\rp\rs^2\rk
			=&\sum_{t=1}^{T}\eta_{t}\mathbb{E}_{\A}\lk\ls\partial R_S\lp\mw^t\rp\rs^2-\ls\partial R_S\lp\mw^{t-\mathcal{T}_t}\rp\rs^2\rk+	\sum_{t=1}^{T}\eta_{t}\mathbb{E}_{\A}\lk\ls\partial R_S\lp\mw^{t-\mathcal{T}_t}\rp\rs^2\rk\nonumber\\
			\leq&\sum_{t=1}^{T}\eta_{t}\mathbb{E}_{\A}\lk\lp\ls\partial R_S\lp\mw^t\rp\rs+\ls\partial R_S\lp\mw^{t-\mathcal{T}_t}\rp\rs\rp\lp\ls\partial R_S\lp\mw^t\rp\rs-\ls\partial R_S\lp\mw^{t-\mathcal{T}_t}\rp\rs\rp\rk\nonumber\\
			&+\sum_{t=1}^{T}\eta_{t}\mathbb{E}_{\A}\lk\ls\partial R_S\lp\mw^{t-\mathcal{T}_t}\rp\rs^2\rk\nonumber\\
			\leq&2L\sum_{t=1}^{T}\eta_{t}\mathbb{E}_{\A}\lk\lp\ls\partial R_S\lp\mw^t\rp-\partial R_S\lp\mw^{t-\mathcal{T}_t}\rp\rs\rp\rk+\sum_{t=1}^{T}\eta_{t}\mathbb{E}_{\A}\lk\ls\partial R_S\lp\mw^{t-\mathcal{T}_t}\rp\rs^2\rk\nonumber\\
			\leq&2\beta L\sum_{t=1}^{T}\eta_{t}\mathbb{E}_{\A}\lk\ls\mw^t-\mw^{t-\mathcal{T}_t}\rs\rk+\sum_{t=1}^{T}\eta_{t}\mathbb{E}_{\A}\lk\ls\partial R_S\lp\mw^{t-\mathcal{T}_t}\rp\rs^2\rk\nonumber\\
			\leq&2\beta L^2\sum_{t=1}^{T}\eta_{t}\sum_{q=t-\mathcal{T}_t+1}^{t}\eta_q+\sum_{t=1}^{T}\eta_{t}\mathbb{E}_{\A}\lk\ls\partial R_S\lp\mw^{t-\mathcal{T}_t}\rp\rs^2\rk.\label{no-deco}
		\end{align}
		Similar to the previous case, consider the probabilistic event of the gradient term,
		\begin{align*}
			&\mathbb{E}_{j_t}\lk\left\langle\frac{1}{m}\sum_{r=1}^{m}\partial f(\mw^{t-\mathcal{T}_t};Z_{j_t(r)}),\partial R_S(\mw^{t-\mathcal{T}_t})\right\rangle\mid\mw^0,\cdots,\mw^{t-\mathcal{T}_t},Z_{j_{1}(r)},\cdots,Z_{j_{t-\mathcal{T}_t}(r)}\rk\\
			=&\frac{1}{m}\sum_{k=1}^{n}\lk\sum_{r=1}^{m}\left\langle\partial f(\mw^{t-\mathcal{T}_t};Z_{j_t(r)}),\partial R_S(\mw^{t-\mathcal{T}_t})\right\rangle\Pr\lp j_{t}(r)=k\mid j_{t-\mathcal{T}_t}(r)=k\rp\rk\\
			=&\frac{1}{m}\sum_{k=1}^{n}\lk\sum_{r=1}^{m}\left\langle\partial f(\mw^{t-\mathcal{T}_t};Z_{j_t(r)}),\partial R_S(\mw^{t-\mathcal{T}_t})\right\rangle\lk H^{\mathcal{T}_t}\rk_{j_{t-\mathcal{T}_t}(r),k}\rk\\
			=&\frac{1}{m}\sum_{k=1}^{n}\lk\sum_{r=1}^{m}\left\langle\partial f(\mw^{t-\mathcal{T}_t};Z_{j_t(r)}),\partial R_S(\mw^{t-\mathcal{T}_t})\right\rangle\lp\lk H^{\mathcal{T}_t}\rk_{j_{t-\mathcal{T}_t}(r),k}-\frac{1}{n}\rp\rk\\
			&+\frac{1}{mn}\sum_{k=1}^{n}\lk\sum_{r=1}^{m}\left\langle\partial f(\mw^{t-\mathcal{T}_t};Z_{j_t(r)}),\partial R_S(\mw^{t-\mathcal{T}_t})\right\rangle\rk\\
			=&\frac{1}{m}\sum_{k=1}^{n}\lk\sum_{r=1}^{m}\left\langle\partial f(\mw^{t-\mathcal{T}_t};Z_{j_t(r)}),\partial R_S(\mw^{t-\mathcal{T}_t})\right\rangle\lp\lk H^{\mathcal{T}_t}\rk_{j_{t-\mathcal{T}_t}(r),k}-\frac{1}{n}\rp\rk+\ls R_S\lp \mw^{t-\mathcal{T}_t}\rp\rs^2.
		\end{align*}
		Taking both sides expectations,
		\begin{align}
			\sum_{t=1}^{T}\eta_{t}\mathbb{E}_{\A}\lk\ls R_S\lp \mw^{t-\mathcal{T}_t}\rp\rs^2\rk=&\underbrace{\sum_{t=1}^{T}\eta_{t}\mathbb{E}_{\A}\lk\frac{1}{m}\sum_{k=1}^{n}\lk\sum_{r=1}^{m}\left\langle\partial f(\mw^{t-\mathcal{T}_t};Z_{j_t(r)}),\partial R_S(\mw^{t-\mathcal{T}_t})\right\rangle\lp\frac{1}{n}-\lk H^{\mathcal{T}_t}\rk_{j_{t-\mathcal{T}_t}(r),k}\rp\rk\rk}_{\spadesuit}\nonumber\\
			&+\underbrace{\sum_{t=1}^{T}\eta_{t}\mathbb{E}_{\A}\lk\left\langle\frac{1}{m}\sum_{r=1}^{m}\partial f(\mw^{t-\mathcal{T}_t};Z_{j_t(r)}),\partial R_S(\mw^{t-\mathcal{T}_t})\right\rangle\rk}_{\clubsuit}.
		\end{align}
		Using the smoothness property of $f$, we can bound the $\clubsuit$ term as
		\begin{align*}
			R_S(\mw^t)\leq&R_S(\mw^{t-1})+\left\langle\mw^t-\mw^{t-1},\partial R_S(\mw^{t-1})\right\rangle+\frac{\beta}{2}\ls\mw^t-\mw^{t-1}\rs^2\\
			\leq&R_S(\mw^{t-1})+\left\langle\mw^t-\mw^{t-1},\partial R_S(\mw^{t-1})-\partial R_S(\mw^{t-\mathcal{T}_t})\right\rangle+\left\langle\mw^t-\mw^{t-1},\partial R_S(\mw^{t-\mathcal{T}_t})\right\rangle+\frac{\beta}{2}\ls\mw^t-\mw^{t-1}\rs^2\\
			\leq&R_S(\mw^{t-1})+\left\langle\mw^t-\mw^{t-1},\partial R_S(\mw^{t-\mathcal{T}_t})\right\rangle+\frac{\beta+1}{2}\ls\mw^t-\mw^{t-1}\rs^2+\frac{\beta^2}{2}\ls\mw^{t-1}-\mw^{t-\mathcal{T}_t}\rs^2\\
			\leq&R_S(\mw^{t-1})+\left\langle\mw^t-\mw^{t-1},\partial R_S(\mw^{t-\mathcal{T}_t})\right\rangle+\frac{(\beta+1)\eta_{t}^2}{2m^2}\ls\sum_{i=1}^{m}\nabla f(\mw^{t-1}(i);Z_{j_t(i)})\rs^2+\frac{\beta^2}{2}\ls\mw^{t-1}-\mw^{t-\mathcal{T}_t}\rs^2\\
			\leq&R_S(\mw^{t-1})+\left\langle\mw^t-\mw^{t-1},\partial R_S(\mw^{t-\mathcal{T}_t})\right\rangle+\frac{(\beta+1)\eta_{t}^2L^2}{2}+\frac{\beta^2}{2}\ls\mw^{t-1}-\mw^{t-\mathcal{T}_t}\rs^2\\
			\leq&R_S(\mw^{t-1})+\left\langle\mw^t-\mw^{t-1},\partial R_S(\mw^{t-\mathcal{T}_t})\right\rangle+\frac{(\beta+1)\eta_{t}^2L^2}{2}+\frac{\beta^2L^2\mathcal{T}_t\sum_{q=t-\mathcal{T}_t+1}^{t-1}\eta_q^2}{2}.
		\end{align*}
		By rearranging the above inequality, we have
		\begin{align}\label{no}
			\mathbb{E}_{\A}\lk\left\langle\mw^t-\mw^{t-1},\partial R_S(\mw^{t-\mathcal{T}_t})\right\rangle\rk\leq\mathbb{E}_{\A}\lk R_S(\mw^{t-1})-R_S(\mw^t)\rk+\frac{(\beta+1)\eta_{t}^2L^2}{2}+\frac{\beta^2L^2\mathcal{T}_t\sum_{q=t-\mathcal{T}_t+1}^{t-1}\eta_q^2}{2}.
		\end{align}
		It should be noted that
		\begin{align*}
			\mathbb{E}_{\A}\lk\left\langle\mw^t-\mw^{t-1},\partial R_S(\mw^{t-\mathcal{T}_t})\right\rangle\rk
			=&\frac{\eta_t}{m}\mathbb{E}_{\A}\lk\left\langle\sum_{i=1}^{m}\nabla f(\mw^{t-1}(i);Z_{j_t(i)}),\partial R_S(\mw^{t-\mathcal{T}_t})\right\rangle\rk\\
			=&\frac{\eta_t}{m}\mathbb{E}_{\A}\lk\left\langle\sum_{i=1}^{m}\lp\nabla f(\mw^{t-1}(i);Z_{j_t(i)})-\nabla f(\mw^{t-1};Z_{j_t(i)})\rp,\partial R_S(\mw^{t-\mathcal{T}_t})\right\rangle\rk\\
			&+\frac{\eta_t}{m}\mathbb{E}_{\A}\lk\left\langle\sum_{i=1}^{m}\nabla f(\mw^{t-1};Z_{j_t(i)}),\partial R_S(\mw^{t-\mathcal{T}_t})\right\rangle\rk\\
			=&\frac{\eta_t}{m}\mathbb{E}_{\A}\lk\left\langle\sum_{i=1}^{m}\lp\nabla f(\mw^{t-1}(i);Z_{j_t(i)})-\nabla f(\mw^{t-1};Z_{j_t(i)})\rp,\partial R_S(\mw^{t-\mathcal{T}_t})\right\rangle\rk\\
			&+\frac{\eta_t}{m}\mathbb{E}_{\A}\lk\left\langle\sum_{i=1}^{m}\lp\nabla f(\mw^{t-1};Z_{j_t(i)})-\nabla f(\mw^{t-\mathcal{T}_t};Z_{j_t(i)})\rp,\partial R_S(\mw^{t-\mathcal{T}_t})\right\rangle\rk\\
			&+\frac{\eta_t}{m}\mathbb{E}_{\A}\lk\left\langle\sum_{i=1}^{m}\nabla f(\mw^{t-\mathcal{T}_t};Z_{j_t(i)}),\partial R_S(\mw^{t-\mathcal{T}_t})\right\rangle\rk\\
			\geq&\frac{\eta_t}{m}\mathbb{E}_{\A}\lk\left\langle\sum_{i=1}^{m}\nabla f(\mw^{t-\mathcal{T}_t};Z_{j_t(i)}),\partial R_S(\mw^{t-\mathcal{T}_t})\right\rangle\rk-\frac{\eta_t \beta L}{m}\mathbb{E}_{\A}\sum_{i=1}^{m}\ls\mw^{t-1}(i)-\mw^{t-1}\rs\\
			&-{\eta_t \beta L}\mathbb{E}_{\A}\ls\mw^{t-1}-\mw^{t-\mathcal{T}_t}
			\rs\\
			\geq&\frac{\eta_t}{m}\mathbb{E}_{\A}\lk\left\langle\sum_{i=1}^{m}\nabla f(\mw^{t-\mathcal{T}_t};Z_{j_t(i)}),\partial R_S(\mw^{t-\mathcal{T}_t})\right\rangle\rk-\eta_t \beta L^2\sum_{q=1}^{t-1}\eta_q\lambda^{t-q-1}-\eta_t\beta L^2\sum_{q=t-\mathcal{T}_t+1}^{t-1}\eta_q.
		\end{align*}
		Combining with the \cref{no}, we have
		\begin{align*}
			&\frac{\eta_t}{m}\mathbb{E}_{\A}\lk\left\langle\sum_{i=1}^{m}\nabla f(\mw^{t-\mathcal{T}_t};Z_{j_t(i)}),\partial R_S(\mw^{t-\mathcal{T}_t})\right\rangle\rk\\
			\leq&\mathbb{E}_{\A}\lk R_S(\mw^{t-1})-R_S(\mw^t)\rk+\frac{(\beta+1)\eta_{t}^2L^2}{2}+\frac{\beta^2L^2\mathcal{T}_t\sum_{q=t-\mathcal{T}_t+1}^{t-1}\eta_q^2}{2}+\eta_t \beta L^2\sum_{q=1}^{t-1}\eta_q\lambda^{t-q-1}+\eta_t\beta L^2\sum_{q=t-\mathcal{T}_t+1}^{t-1}\eta_q.
		\end{align*}
		It then follows that
		\begin{align*}
			&\sum_{t=1}^{T}\frac{\eta_t}{m}\mathbb{E}_{\A}\lk\left\langle\sum_{i=1}^{m}\nabla f(\mw^{t-\mathcal{T}_t};Z_{j_t(i)}),\partial R_S(\mw^{t-\mathcal{T}_t})\right\rangle\rk\\
			\leq& R_S(\mw^{0})+\sum_{t=1}^{T}\frac{(\beta+1)\eta_{t}^2L^2}{2}+\sum_{t=1}^{T}\frac{\beta^2L^2\mathcal{T}_t\sum_{q=t-\mathcal{T}_t+1}^{t-1}\eta_q^2}{2}+\sum_{t=1}^{T}\eta_t \beta L^2\sum_{q=1}^{t-1}\eta_q\lambda^{t-q-1}+\sum_{t=1}^{T}\eta_t\beta L^2\sum_{q=t-\mathcal{T}_t+1}^{t-1}\eta_q.
		\end{align*}
		It can be derived that
		\begin{align*}
			\sum_{t=1}^{T}\eta_{t}\mathbb{E}_{\A}\lk\frac{1}{m}\sum_{k=1}^{n}\lk\sum_{r=1}^{m}\left\langle\partial f(\mw^{t-\mathcal{T}_t};Z_{j_t(r)}),\partial R_S(\mw^{t-\mathcal{T}_t})\right\rangle\lp\frac{1}{n}-\lk H^{\mathcal{T}_t}\rk_{j_{t-\mathcal{T}_t}(r),k}\rp\rk\rk
			\leq2L^2\sum_{t=1}^{K_H-1}\eta_t+\sum_{t=K_H}^{T}\frac{L^2\eta_t}{2t}.
		\end{align*}
		Then we can show that
		\begin{align*}
			\sum_{t=1}^{T}\eta_{t}\mathbb{E}_{\A}\lk\ls R_S\lp w^{t-\mathcal{T}_t}\rp\rs^2\rk
			\leq&2L^2\sum_{t=1}^{K_p-1}\eta_t+\sum_{t=K_p}^{T}\frac{L^2\eta_t}{2t}+R_S(\mw^{0})+\sum_{t=1}^{T}\frac{(\beta+1)\eta_{t}^2L^2}{2}+\sum_{t=1}^{T}\frac{\beta^2L^2\mathcal{T}_t\sum_{q=t-\mathcal{T}_t+1}^{t-1}\eta_q^2}{2}\\
			&+\sum_{t=1}^{T}\eta_t \beta L^2\sum_{q=1}^{t-1}\eta_q\lambda^{t-q-1}+\sum_{t=1}^{T}\eta_t\beta L^2\sum_{q=t-\mathcal{T}_t+1}^{t-1}\eta_q.
		\end{align*}
		Going back to \cref{no-deco}, we can get
		\begin{align*}
			\sum_{t=1}^{T}\eta_{t}\mathbb{E}_{\A}\lk\ls\partial R_S\lp\mw^t\rp\rs^2\rk
			\leq&2\beta L^2\sum_{t=1}^{T}\eta_{t}\sum_{q=t-\mathcal{T}_t+1}^{t}\eta_q+2L^2\sum_{t=1}^{K_H-1}\eta_t+\sum_{t=K_H}^{T}\frac{L^2\eta_t}{2t}+R_S(\mw^{0})+\frac{(\beta+1)L^2}{2}\sum_{t=1}^{T}{\eta_{t}^2}\\
			&+\frac{\beta^2L^2}{2}\sum_{t=1}^{T}{\mathcal{T}_t\sum_{q=t-\mathcal{T}_t+1}^{t-1}\eta_q^2}+\beta L^2\sum_{t=1}^{T}\eta_t \sum_{q=1}^{t-1}\eta_q\lambda^{t-q-1}+\beta L^2\sum_{t=1}^{T}\eta_t\sum_{q=t-\mathcal{T}_t+1}^{t-1}\eta_q.
		\end{align*}
		Furthermore, based on 
		\begin{align*}
			\sum_{t=1}^{T}\eta_{t}\min_{1\leq t\leq T}\mathbb{E}_{\A}\lk\ls\partial R_S\lp\mw^t\rp\rs^2\rk\leq\sum_{t=1}^{T}\eta_{t}\mathbb{E}_{\A}\lk\ls\partial R_S\lp\mw^t\rp\rs^2\rk,
		\end{align*}
		we have
		\begin{align*}
			\min_{1\leq t\leq T}\mathbb{E}_{\A}\lk\ls\partial R_S\lp\mw^t\rp\rs^2\rk\leq&\frac{2L^2\sum_{t=1}^{K_H-1}\eta_t+\sum_{t=K_H}^{T}\frac{L^2\eta_t}{2t}+R_S(\mw^{0})+\frac{(\beta+1)L^2}{2}\sum_{t=1}^{T}{\eta_{t}^2}}{\sum_{t=1}^{T}\eta_{t}}\\
			&+\frac{\beta L^2\sum_{t=1}^{T}\eta_{t}\lp2\sum_{q=t-\mathcal{T}_t+1}^{t}\eta_q+\sum_{q=1}^{t-1}\eta_q\lambda^{t-q-1}+\sum_{q=t-\mathcal{T}_t+1}^{t-1}\eta_q\rp}{\sum_{t=1}^{T}\eta_{t}}\\
			&+\frac{{\beta^2L^2}\sum_{t=1}^{T}{\mathcal{T}_t\sum_{q=t-\mathcal{T}_t+1}^{t-1}\eta_q^2}}{2\sum_{t=1}^{T}\eta_{t}}.
		\end{align*}
		If we choose the stepsizes $\eta_{t}=\eta=\frac{1}{\sqrt{T}\log T}$, then
		\begin{align*}
			\min_{1\leq t\leq T}\mathbb{E}_{\A}\lk\ls\partial R_S\lp\mw^t\rp\rs^2\rk=&\mathcal{O}\lp\frac{K_H\eta+1+T\eta^2+\frac{\eta^2 T}{1-\lambda}+\sum_{t=1}^{T}\mathcal{T}_t^2\eta^2}{T\eta}\rp\\
			=&\mathcal{O}\lp\frac{K_H}{T}+\frac{\eta}{1-\lambda}+\frac{\log T(1+\sum_{t=1}^{T}\mathcal{T}_t^2\eta^2)}{\sqrt{T}}\rp.
		\end{align*}
		Now we estimate the term $\sum_{t=1}^{T}\mathcal{T}_t^2\eta^2$. We also define that $J=\frac{1}{2C_H Dn\lambda(H)^{K_H}}$, if $t\leq J$, we have
		\begin{align*}
			\log (2C_HDnt)\leq\log (2C_HDnJ)=\log\lp\frac{1}{\lambda(H)^{K_H}}\rp=K_H\log\lp1/\lambda(H)\rp.
		\end{align*}
		We can get that
		\begin{align*}
			\sum_{t=1}^{J}\mathcal{T}_t^2\eta^2\leq JK_H^2\eta^2=\frac{K_H^2}{T\log^2 T2C_H D n\lambda(H)^{K_H}}.
		\end{align*}
		If $t>J$, we have $\mathcal{T}_t\leq\lceil\frac{\log\lp2C_HDnt\rp}{\log\lp1/\lambda(H)\rp}\rceil$ and
		\begin{align*}
			\sum_{t=J+1}^{T}\mathcal{T}_t^2\eta^2\leq&\sum_{t=J+1}^{T}\lk\lp\frac{\log\lp2C_HD\rp}{\log\lp1/\lambda(H)\rp}+\frac{\log\lp n\rp}{\log\lp1/\lambda(H)\rp}+\frac{\log\lp t\rp}{\log\lp1/\lambda(H)\rp}\rp+1\rk^2\eta^2\\
			\leq&\frac{6}{\log^2\lp1/\lambda(H)\rp}\lk(T-J)\eta^2\lp\log^2\lp2C_HD\rp+\log^2\lp n\rp\rp+\sum_{t=J+1}^{T}\log(t)^2\eta^2\rk+2T\eta^2\\
			\leq&\mathcal{O}\lp\frac{1}{\log^2\lp1/\lambda(H)\rp}\rp,
		\end{align*}
		where the second inequality uses the basic inequality.
		Then we obtain
		\begin{align*}
			\sum_{t=1}^{T}\mathcal{T}_t^2\eta^2=\mathcal{O}\lp\frac{K_H^2}{T\log^2 TC_H Dn\lambda(H)^{K_H}}+\frac{1}{\log^2(1/\lambda(H))}\rp.
		\end{align*}
		Then
		\begin{align*}
			\min_{1\leq t\leq T}\mathbb{E}_{\A}\lk\ls\partial R_S\lp\mw^t\rp\rs^2\rk=\mathcal{O}\lp\frac{K_H}{T}+\frac{1}{\lp1-\lambda\rp\sqrt{T}\log T}+\frac{\log T}{\sqrt{T}}\lp\frac{K_H^2}{T\log^2 TC_H n\lambda(H)^{K_H}}+\frac{1}{\log^2(1/\lambda(H))}\rp\rp.
		\end{align*}
		If $K_H=0$, we have
		\begin{align*}
			\min_{1\leq t\leq T}\mathbb{E}_{\A}\lk\ls\partial R_S\lp\mw^t\rp\rs^2\rk=\mathcal{O}\lp\frac{1}{\lp1-\lambda\rp\sqrt{T}\log T}+\frac{\log T}{\sqrt{T}\log^2(1/\lambda(H))}\rp.
		\end{align*}
	\end{proof}
	\section{Decentralized Markov Chain SGDA}
	\begin{lemma}[\cite{lei2021stability}, Lemma C.1.]
		Let $f(\mathbf{w}, \mathbf{v})$ be $\rho$-SC-SC with $\rho\geq0$ and $\beta$-smooth. 
		\begin{align*}
			\left\|\binom{\mathbf{w}-\eta \nabla_{\mathbf{w}} f(\mathbf{w}, \mathbf{v})}{\mathbf{v}+\eta \nabla_{\mathbf{v}} f(\mathbf{w}, \mathbf{v})}-\binom{\tilde{\mw}-\eta \nabla_{\mathbf{w}} f\left(\tilde{\mw}, \tilde{\mv}\right)}{\tilde{\mv}+\eta \nabla_{\mathbf{v}} f\left(\tilde{\mw}, \tilde{\mv}\right)}\right\|_2^2 \leq\left(1-2 \rho \eta+\beta^2 \eta^2\right)\left\|\binom{\mathbf{w}-\tilde{\mw}}{\mathbf{v}-\tilde{\mv}}\right\|_2^2 .
		\end{align*}
	\end{lemma}
	\begin{lemma}[\textnormal{[Rockafellar,1976]}]
		Let $f(\mathbf{w}, \mathbf{v})$ be $\rho$-SC-SC with $\rho\geq0$.
		\begin{align*}
			\left\langle\binom{\mw-\tilde{\mw}}{\mv-\tilde{\mv}},\binom{\nabla_{\mathbf{w}} f(\mathbf{w}, \mathbf{v})-\nabla_{\mathbf{w}} f\left(\tilde{\mw}, \tilde{\mv}\right)}{ \nabla_{\mathbf{v}} f\left(\tilde{\mw}, \tilde{\mv}\right)-\nabla_{\mathbf{v}} f(\mathbf{w}, \mathbf{v})}\right\rangle\geq\rho\ls\binom{\mw-\tilde{\mw}}{\mv-\tilde{\mv}}\rs^2.
		\end{align*}
	\end{lemma}
	\begin{lemma}
		Suppose that Assumption \ref{Lipschitz} holds, and $\{\mw^{t}(i),\mv^{t}(i),{\mw}^{t},\mv^{t}\}$ are generated by D-SGDA during the $t$-th iteration. In this case, the difference between the average models ${\mw}^{t},{\mv}^{t}$ and each local model $\mw^{t}(i),\mv^{t}(i)$ can be bounded as follows:
		\begin{equation*}
			\lk\sum_{i=1}^{m}\ls{\mw}^{t}-\mw^{t}(i)\rs^2+\ls{\mv}^{t}-\mv^{t}(i)\rs^2\rk^\frac{1}{2}\leq2\sqrt{m}L\sum_{q=1}^{t}\eta_q\lambda^{t-q}.
		\end{equation*}
	\end{lemma}
	\subsection{Generalization Error via Stability}
	\begin{theorem}[Generalization via Argument Stability]
		Let $\mathcal{A}$ be a randomized algorithm with on-average $\epsilon$-argument-stable.
		\begin{itemize}
			\item If Assumption 5 holds, then
			\begin{align*}
				\Delta^{\mw}\left(\mathcal{A}_{\mathrm{w}}, \mathcal{A}_{\mathrm{v}}\right)-\Delta_{e m p}^w\left(\mathcal{A}_{\mathrm{w}}, \mathcal{A}_{\mathrm{v}}\right) \leq L\epsilon.
			\end{align*}
			\item \cite{lei2021stability} If the mapping $\mathbf{v} \mapsto F(\mathbf{w}, \mathbf{v})$ is $\rho$-strongly-concave and Assumptions 5-6 hold, then
			\begin{align*}
				\mathbb{E}_{S, \mathcal{A}}\left[F\left(\mathcal{A}_{\mathrm{w}}(S)\right)-F_S\left(\mathcal{A}_{\mathrm{w}}(S)\right)\right] \leq(1+\beta /\rho)L\epsilon.
			\end{align*}
		\end{itemize}
	\end{theorem}
	\begin{proof}
		We can follow \citep{wang-markov-2022}'s step, 
		\begin{align*}
			&\Delta^{\mw}\left(\mathcal{A}_{\mathbf{w}}(S), \mathcal{A}_{\mathbf{v}}(S)\right)-\Delta_{\mathrm{em p}}^{\mw}\left(\mathcal{A}_{\mathbf{w}}(S), \mathcal{A}_{\mathbf{v}}(S)\right)\\
			\leq& \sup _{\mathbf{v}^{\prime} \in \mathcal{V}} \mathbb{E}\Big[R\left(\mathcal{A}_{\mathbf{w}}(S), \mathbf{v}^{\prime}\right) 
			-R_S\left(\mathcal{A}_{\mathbf{w}}(S), \mathbf{v}^{\prime}\right)\Big]
			+\sup_{\mathbf{w}^{\prime} \in \mathcal{W}} \mathbb{E}\left[R_S\left(\mathbf{w}^{\prime}, \mathcal{A}_{\mathbf{v}}(S)\right)-R\left(\mathbf{w}^{\prime}, \mathcal{A}_{\mathbf{v}}(S)\right)\right].
		\end{align*}
		Recall that $S=\left\{Z_{1(1)}, \ldots, Z_{n(m)}\right\}$ and $S_{rk}=\left\{Z_{1(1)}, \ldots, \tilde{Z}_{k(r)}, \ldots, Z_{n(m)}\right\}$. According to the symmetry, we conclude that
		\begin{align*}
			\mathbb{E}\left[R\left(\mathcal{A}_{\mw}(S), \mv^{\prime}\right)-R_S\left(\mathcal{A}_{\mw}(S), \mv^{\prime}\right)\right] & =\frac{1}{mn}\sum_{r=1}^{m}\sum_{k=1}^n \mathbb{E}\left[R\left(\mathcal{A}_{\mw}\left(S_{rk}\right), \mv^{\prime}\right)\right]-\mathbb{E}\left[R
			_S\left(\mathcal{A}_{\mw}(S), \mv^{\prime}\right)\right] \\
			& =\frac{1}{mn}\sum_{r=1}^{m}\sum_{k=1}^n \mathbb{E}\left[f\left(\mathcal{A}_{\mw}\left(S_{rk}\right), \mv^{\prime} ; Z_{k(r)}\right)-f\left(\mathcal{A}_{\mw}(S), \mv^{\prime} ; Z_{k(r)}\right)\right] \\
			& \leq \frac{L}{mn} \sum_{r=1}^{m}\sum_{k=1}^n \mathbb{E}\left[\left\|\mathcal{A}_{\mw}\left(S_{rk}\right)-\mathcal{A}_{\mw}(S)\right\|_2\right].
		\end{align*}
		In a similar way, we can prove
		\begin{align*}
			\mathbb{E}\left[R_S\left(\mathbf{w}^{\prime}, \mathcal{A}_{\mathbf{v}}(S)\right)-R\left(\mathbf{w}^{\prime}, \mathcal{A}_{\mathbf{v}}(S)\right)\right] \leq \frac{L}{mn}\sum_{r=1}^{m}\sum_{i=1}^n \mathbb{E}\left[\left\|\mathcal{A}_{\mathbf{v}}\left(S_{rk}\right)-\mathcal{A}_{\mathbf{v}}(S)\right\|_2\right].
		\end{align*}
		As a combination of the above three inequalities we get
		\begin{align*}
			\Delta^{\mw}\left(\mathcal{A}_{\mw}(S), \mathcal{A}_{\mv}(S)\right)-\Delta_{\mathrm{emp}}^{\mw}\left(\mathcal{A}_{\mw}(S), \mathcal{A}_{\mv}(S)\right) \leq \frac{L}{mn} \sum_{r=1}^{m}\sum_{k=1}^n \mathbb{E}\left[\left\|\mathcal{A}_{\mw}\left(S_{rk}\right)-\mathcal{A}_{\mw}(S)\right\|_2+\left\|\mathcal{A}_{\mv}\left(S_{rk}\right)-\mathcal{A}_{\mv}(S)\right\|_2\right].
		\end{align*}
	\end{proof}
	\subsection{Stability Bounds}
	We first prove the smooth case.
	\begin{proof}
		We consider two cases at the $t$-th iteration.
		If $j_t\neq k$, we have
		\begin{align*}
			\ls\binom{\mw^{t}-\mw_{(rk)}^t}{\mv^{t}-\mv_{(rk)}^t}\rs\leq&\ls\binom{\mw^{t-1}-\frac{\eta_{t}}{m}\sum_{i=1}^{m}\nabla_{\mw}f(\mw^{t-1}(i),\mv^{t-1}(i);Z_{j_t(i)})-\mw_{(rk)}^{t-1}+\frac{\eta_{t}}{m}\sum_{i=1}^{m}\nabla_{\mw}f(\mw_{(rk)}^{t-1}(i),\mv_{(rk)}^{t-1}(i);{Z}_{j_t(i)})}{\mv^{t-1}+\frac{\eta_{t}}{m}\sum_{i=1}^{m}\nabla_{\mv}f(\mw^{t-1}(i),\mv^{t-1}(i);Z_{j_t(i)})-\mv_{(rk)}^{t-1}-\frac{\eta_{t}}{m}\sum_{i=1}^{m}\nabla_{\mv}f(\mw_{(rk)}^{t-1}(i),\mv_{(rk)}^{t-1}(i);{Z}_{j_t(i)})}\rs\\
			\leq&\ls\binom{\mw^{t-1}-\frac{\eta_{t}}{m}\sum_{i=1}^{m}\nabla_{\mw}f(\mw^{t-1},\mv^{t-1};Z_{j_t(i)})-\mw_{(rk)}^{t-1}+\frac{\eta_{t}}{m}\sum_{i=1}^{m}\nabla_{\mw}f(\mw_{(rk)}^{t-1},\mv_{(rk)}^{t-1};{Z}_{j_t(i)})}{\mv^{t-1}+\frac{\eta_{t}}{m}\sum_{i=1}^{m}\nabla_{\mv}f(\mw^{t-1},\mv^{t-1};Z_{j_t(i)})-\mv_{(rk)}^{t-1}-\frac{\eta_{t}}{m}\sum_{i=1}^{m}\nabla_{\mv}f(\mw_{(rk)}^{t-1},\mv_{(rk)}^{t-1};{Z}_{j_t(i)})}\rs\\
			&+\ls\binom{\frac{\eta_{t}}{m}\sum_{i=1}^{m}\nabla_{\mw}f(\mw^{t-1},\mv^{t-1};Z_{j_t(i)})-\nabla_{\mw}f(\mw^{t-1}(i),\mv^{t-1}(i);Z_{j_t(i)})}{\frac{\eta_{t}}{m}\sum_{i=1}^{m}\nabla_{\mv}f(\mw^{t-1}(i),\mv^{t-1}(i);Z_{j_t(i)})-\nabla_{\mv}f(\mw^{t-1},\mv^{t-1};Z_{j_t(i)})}\rs\\
			&+\ls\binom{\frac{\eta_{t}}{m}\sum_{i=1}^{m}\nabla_{\mw}f(\mw_{(rk)}^{t-1}(i),\mv_{(rk)}^{t-1}(i);{Z}_{j_t(i)})-\nabla_{\mw}f(\mw_{(rk)}^{t-1},\mv_{(rk)}^{t-1};{Z}_{j_t(i)})}{\frac{\eta_{t}}{m}\sum_{i=1}^{m}\nabla_{\mv}f(\mw_{(rk)}^{t-1},\mv_{(rk)}^{t-1};{Z}_{j_t(i)})-\nabla_{\mv}f(\mw_{(rk)}^{t-1}(i),\mv_{(rk)}^{t-1}(i);{Z}_{j_t(i)})}\rs\\
			\leq&\sqrt{1+\beta^2\eta_t^2}	\ls\binom{\mw^{t-1}-\mw_{(rk)}^{t-1}}{\mv^{t-1}-\mv_{(rk)}^{t-1}}\rs+4\sqrt{2}\eta_{t}\beta L\sum_{q=1}^{t-1}\eta_{q}\lambda^{t-q-1},
		\end{align*}
		where the last inequality uses Lemma 8 and Lemma 10.
		
		If $j_t=k$, we have
		\begin{align*}
			\ls\binom{\mw^{t}-\mw_{(rk)}^t}{\mv^{t}-\mv_{(rk)}^t}\rs\leq&\ls\binom{\mw^{t-1}-\frac{\eta_{t}}{m}\sum_{i=1}^{m}\nabla_{\mw}f(\mw^{t-1}(i),\mv^{t-1}(i);Z_{j_t(i)})-\mw_{(rk)}^{t-1}+\frac{\eta_{t}}{m}\sum_{i=1}^{m}\nabla_{\mw}f(\mw_{(rk)}^{t-1}(i),\mv_{(rk)}^{t-1}(i);\tilde{Z}_{j_t(i)})}{\mv^{t-1}+\frac{\eta_{t}}{m}\sum_{i=1}^{m}\nabla_{\mv}f(\mw^{t-1}(i),\mv^{t-1}(i);Z_{j_t(i)})-\mv_{(rk)}^{t-1}-\frac{\eta_{t}}{m}\sum_{i=1}^{m}\nabla_{\mv}f(\mw_{(rk)}^{t-1}(i),\mv_{(rk)}^{t-1}(i);\tilde{Z}_{j_t(i)})}\rs\\
			\leq&\ls\binom{\mw^{t-1}-\frac{\eta_{t}}{m}\sum_{i=1,i\neq r}^{m}\nabla_{\mw}f(\mw^{t-1},\mv^{t-1};Z_{j_t(i)})-\mw_{(rk)}^{t-1}+\frac{\eta_{t}}{m}\sum_{i=1,i\neq r}^{m}\nabla_{\mw}f(\mw_{(rk)}^{t-1},\mv_{(rk)}^{t-1};\tilde{Z}_{j_t(i)})}{\mv^{t-1}+\frac{\eta_{t}}{m}\sum_{i=1,i\neq r}^{m}\nabla_{\mv}f(\mw^{t-1},\mv^{t-1};Z_{j_t(i)})-\mv_{(rk)}^{t-1}-\frac{\eta_{t}}{m}\sum_{i=1,i\neq r}^{m}\nabla_{\mv}f(\mw_{(rk)}^{t-1},\mv_{(rk)}^{t-1};\tilde{Z}_{j_t(i)})}\rs\\
			&+\frac{\eta}{m}\ls\binom{\nabla_{\mw}f(\mw^{t-1},\mv^{t-1};Z_{k(r)})-\nabla_{\mw}f(\mw_{(rk)}^{t-1},\mv_{(rk)}^{t-1};\tilde{Z}_{k(r)})}{\nabla_{\mv}f(\mw^{t-1},\mv^{t-1};Z_{k(r)})-\nabla_{\mv}f(\mw_{(rk)}^{t-1},\mv_{(rk)}^{t-1};\tilde{Z}_{k(r)})}\rs\\
			&+\ls\binom{\frac{\eta_{t}}{m}\sum_{i=1}^{m}\nabla_{\mw}f(\mw^{t-1},\mv^{t-1};Z_{j_t(i)})-\nabla_{\mw}f(\mw^{t-1}(i),\mv^{t-1}(i);Z_{j_t(i)})}{\frac{\eta_{t}}{m}\sum_{i=1}^{m}\nabla_{\mv}f(\mw^{t-1}(i),\mv^{t-1}(i);Z_{j_t(i)})-\nabla_{\mv}f(\mw^{t-1},\mv^{t-1};Z_{j_t(i)})}\rs\\
			&+\ls\binom{\frac{\eta_{t}}{m}\sum_{i=1}^{m}\nabla_{\mw}f(\mw_{(rk)}^{t-1}(i),\mv_{(rk)}^{t-1}(i);\tilde{Z}_{j_t(i)})-\nabla_{\mw}f(\mw_{(rk)}^{t-1},\mv_{(rk)}^{t-1};\tilde{Z}_{j_t(i)})}{\frac{\eta_{t}}{m}\sum_{i=1}^{m}\nabla_{\mv}f(\mw_{(rk)}^{t-1},\mv_{(rk)}^{t-1};\tilde{Z}_{j_t(i)})-\nabla_{\mv}f(\mw_{(rk)}^{t-1}(i),\mv_{(rk)}^{t-1}(i);\tilde{Z}_{j_t(i)})}\rs\\
			\leq&\sqrt{1+\beta^2\eta_t^2}	\ls\binom{\mw^{t-1}-\mw_{(rk)}^{t-1}}{\mv^{t-1}-\mv_{(rk)}^{t-1}}\rs+4\sqrt{2}\eta_{t}\beta L\sum_{q=1}^{t-1}\eta_{q}\lambda^{t-q-1}+\frac{2\sqrt{2}\eta_{t}L}{m}.
		\end{align*}
		Combine the above two case,
		\begin{align*}
			\ls\binom{\mw^{t}-\mw_{(rk)}^t}{\mv^{t}-\mv_{(rk)}^t}\rs\leq&\sqrt{1+\beta^2\eta_t^2}	\ls\binom{\mw^{t-1}-\mw_{(rk)}^{t-1}}{\mv^{t-1}-\mv_{(rk)}^{t-1}}\rs+4\sqrt{2}\eta_{t}\beta L\sum_{q=1}^{t-1}\eta_{q}\lambda^{t-q-1}+\frac{2\sqrt{2}\eta_{t}L}{m}\mathbb{I}_{\lk j_t=k\rk}\\
			\leq&(1+\beta\eta_t)	\ls\binom{\mw^{t-1}-\mw_{(rk)}^{t-1}}{\mv^{t-1}-\mv_{(rk)}^{t-1}}\rs+4\sqrt{2}\eta_{t}\beta L\sum_{q=1}^{t-1}\eta_{q}\lambda^{t-q-1}+\frac{2\sqrt{2}\eta_{t}L}{m}\mathbb{I}_{\lk j_t=k\rk}.
		\end{align*}
		We can apply the above inequality recursively and get
		\begin{align*}
			\ls\binom{\mw^{t}-\mw_{(rk)}^t}{\mv^{t}-\mv_{(rk)}^t}\rs\leq\beta\sum_{s=1}^{t}\eta_s\ls\binom{\mw^{s}-\mw_{(rk)}^s}{\mv^{s}-\mv_{(rk)}^s}\rs+4\sqrt{2}\sum_{s=1}^{t}\eta_{s}\beta L\sum_{q=1}^{s-1}\eta_{q}\lambda^{s-q-1}+\frac{2\sqrt{2}L}{m}\sum_{s=1}^{t}\eta_{s}\mathbb{I}_{\lk j_s=k\rk}.
		\end{align*}
		Let 
		\begin{align*}
			\delta_t^{(rk)}=\max_{s\in[t]}\ls\binom{\mw^{s}-\mw_{(rk)}^s}{\mv^{s}-\mv_{(rk)}^s}\rs.
		\end{align*}
		If $\sum_{s=1}^{t}\eta_{s}\leq1/2\beta$, we can get that
		\begin{align*}
			\delta_t^{(rk)}\leq&\beta\delta_t^{(rk)}\sum_{s=1}^{t}\eta_{s}+4\sqrt{2}\sum_{s=1}^{t}\eta_{s}\beta L\sum_{q=1}^{s-1}\eta_{q}\lambda^{s-q-1}+\frac{2\sqrt{2}L}{m}\sum_{s=1}^{t}\eta_{s}\mathbb{I}_{\lk j_s=k\rk}\\
			\leq&\frac{1}{2}\delta_t^{(rk)}+4\sqrt{2}\sum_{s=1}^{t}\eta_{s}\beta L\sum_{q=1}^{s-1}\eta_{q}\lambda^{s-q-1}+\frac{2\sqrt{2}L}{m}\sum_{s=1}^{t}\eta_{s}\mathbb{I}_{\lk j_s=k\rk}.
		\end{align*}
		It then follows that
		\begin{align*}
			\delta_t^{(rk)}\leq8\sqrt{2}\beta L\sum_{s=1}^{t}\eta_{s}\sum_{q=1}^{s-1}\eta_{q}\lambda^{s-q-1}+\frac{4\sqrt{2}L}{m}\sum_{s=1}^{t}\eta_{s}\mathbb{I}_{\lk j_s=k\rk}.
		\end{align*}
		Taking average about $k$ and $r$,
		\begin{align}\label{sgda-smooth}
			\frac{1}{mn}\sum_{r=1}^{m}\sum_{k=1}^{n}\ls\binom{\mw^{T}-\mw_{(rk)}^T}{\mv^{T}-\mv_{(rk)}^T}\rs
			\leq&8\sqrt{2}\beta L\sum_{t=1}^{T}\eta_{t}\sum_{q=1}^{t-1}\eta_{q}\lambda^{t-q-1}+\frac{4\sqrt{2}L}{mn}\sum_{t=1}^{T}\eta_{t}.
		\end{align}
		Taking expectation about the algorithm $\A$, we have
		\begin{align*}
			\mathbb{E}_{\A}\lk\frac{1}{mn}\sum_{r=1}^{m}\sum_{k=1}^{n}\ls\binom{\mw^{T}-\mw_{(rk)}^T}{\mv^{t}-\mv_{(rk)}^t}\rs\rk
			\leq&8\sqrt{2}\beta L\sum_{t=1}^{T}\eta_{t}\sum_{q=1}^{t-1}\eta_{q}\lambda^{t-q-1}+\frac{4\sqrt{2}L}{mn}\sum_{t=1}^{T}\eta_{t}.
		\end{align*}
	\end{proof}
	\noindent Now we turn into non-smooth case.
	\begin{proof}
		We consider two cases at the $t$-th iteration.
		\begin{align*}
			\ls\binom{\mw^{t}-\mw_{(rk)}^t}{\mv^{t}-\mv_{(rk)}^t}\rs^2\leq&\ls\binom{\mw^{t-1}-\frac{\eta_{t}}{m}\sum_{i=1}^{m}\nabla_{\mw}f(\mw^{t-1}(i),\mv^{t-1}(i);Z_{j_t(i)})-\mw_{(rk)}^{t-1}+\frac{\eta_{t}}{m}\sum_{i=1}^{m}\nabla_{\mw}f(\mw_{(rk)}^{t-1}(i),\mv_{(rk)}^{t-1}(i);\tilde{Z}_{j_t(i)})}{\mv^{t-1}+\frac{\eta_{t}}{m}\sum_{i=1}^{m}\nabla_{\mv}f(\mw^{t-1}(i),\mv^{t-1}(i);Z_{j_t(i)})-\mv_{(rk)}^{t-1}-\frac{\eta_{t}}{m}\sum_{i=1}^{m}\nabla_{\mv}f(\mw_{(rk)}^{t-1}(i),\mv_{(rk)}^{t-1}(i);\tilde{Z}_{j_t(i)})}\rs^2\\
			=&\ls\binom{\mw^{t-1}-\mw_{(rk)}^{t-1}}{\mv^{t-1}-\mv_{(rk)}^{t-1}}\rs^2+\frac{\eta_t^2}{m^2}\ls\binom{\sum_{i=1}^{m}\lp\nabla_{\mw}f(\mw^{t-1}(i),\mv^{t-1}(i);Z_{j_t(i)})-\nabla_{\mw}f(\mw_{(rk)}^{t-1}(i),\mv_{(rk)}^{t-1}(i);\tilde{Z}_{j_t(i)})\rp}{\sum_{i=1}^{m}\lp\nabla_{\mv}f(\mw^{t-1}(i),\mv^{t-1}(i);Z_{j_t(i)})-\nabla_{\mv}f(\mw_{(rk)}^{t-1}(i),\mv_{(rk)}^{t-1}(i);\tilde{Z}_{j_t(i)})\rp}\rs^2\\
			&\underbrace{-\frac{2\eta_{t}}{m}\left\langle\binom{\mw^{t-1}-\mw_{(rk)}^{t-1}}{\mv^{t-1}-\mv_{(rk)}^{t-1}},\binom{\sum_{i=1}^{m}\lp\nabla_{\mw}f(\mw^{t-1}(i),\mv^{t-1}(i);Z_{j_t(i)})-\nabla_{\mw}f(\mw_{(rk)}^{t-1}(i),\mv_{(rk)}^{t-1}(i);\tilde{Z}_{j_t(i)})\rp}{\sum_{i=1}^{m}\lp\nabla_{\mv}f(\mw_{(rk)}^{t-1}(i),\mv_{(rk)}^{t-1}(i);\tilde{Z}_{j_t(i)})-\nabla_{\mv}f(\mw^{t-1}(i),\mv^{t-1}(i);Z_{j_t(i)})\rp}\right\rangle}_{\Im}.
		\end{align*}
		If $i=r$ and $j_t=k$, we obtain
		\begin{align*}
			&\left\langle\binom{\mw^{t-1}-\mw_{(rk)}^{t-1}}{\mv^{t-1}-\mv_{(rk)}^{t-1}},\binom{\nabla_{\mw}f(\mw^{t-1}(i),\mv^{t-1}(i);Z_{j_t(i)})-\nabla_{\mw}f(\mw_{(rk)}^{t-1}(i),\mv_{(rk)}^{t-1}(i);\tilde{Z}_{j_t(i)})}{\nabla_{\mv}f(\mw_{(rk)}^{t-1}(i),\mv_{(rk)}^{t-1}(i);\tilde{Z}_{j_t(i)})-\nabla_{\mv}f(\mw^{t-1}(i),\mv^{t-1}(i);Z_{j_t(i)})}\right\rangle\\
			\leq&\ls\binom{\mw^{t-1}-\mw_{(rk)}^{t-1}}{\mv^{t-1}-\mv_{(rk)}^{t-1}}\rs\ls\binom{\nabla_{\mw}f(\mw^{t-1}(i),\mv^{t-1}(i);Z_{j_t(i)})-\nabla_{\mw}f(\mw_{(rk)}^{t-1}(i),\mv_{(rk)}^{t-1}(i);\tilde{Z}_{j_t(i)})}{\nabla_{\mv}f(\mw_{(rk)}^{t-1}(i),\mv_{(rk)}^{t-1}(i);\tilde{Z}_{j_t(i)})-\nabla_{\mv}f(\mw^{t-1}(i),\mv^{t-1}(i);Z_{j_t(i)})}\rs\\
			\leq&4\sqrt{2}L\eta_{t}\ls\binom{\mw^{t-1}-\mw_{(rk)}^{t-1}}{\mv^{t-1}-\mv_{(rk)}^{t-1}}\rs.
		\end{align*}
		
		If ($i=r$ and $j_t\neq k$) or $i\neq r$, we have
		\begin{align*}
			\ls\binom{\mw^{t}-\mw_{(rk)}^t}{\mv^{t}-\mv_{(rk)}^t}\rs^2\leq&\ls\binom{\mw^{t-1}-\mw_{(rk)}^{t-1}}{\mv^{t-1}-\mv_{(rk)}^{t-1}}\rs^2+\frac{\eta_t^2}{m^2}\ls\binom{\sum_{i=1}^{m}\lp\nabla_{\mw}f(\mw^{t-1}(i),\mv^{t-1}(i);Z_{j_t(i)})-\nabla_{\mw}f(\mw_{(rk)}^{t-1}(i),\mv_{(rk)}^{t-1}(i);{Z}_{j_t(i)})\rp}{\sum_{i=1}^{m}\lp\nabla_{\mv}f(\mw^{t-1}(i),\mv^{t-1}(i);Z_{j_t(i)})-\nabla_{\mv}f(\mw_{(rk)}^{t-1}(i),\mv_{(rk)}^{t-1}(i);{Z}_{j_t(i)})\rp}\rs^2\\
			&{-\frac{2\eta_{t}}{m}\left\langle\binom{\mw^{t-1}-\mw_{(rk)}^{t-1}}{\mv^{t-1}-\mv_{(rk)}^{t-1}},\binom{\sum_{i=1}^{m}\lp\nabla_{\mw}f(\mw^{t-1}(i),\mv^{t-1}(i);Z_{j_t(i)})-\nabla_{\mw}f(\mw_{(rk)}^{t-1}(i),\mv_{(rk)}^{t-1}(i);{Z}_{j_t(i)})\rp}{\sum_{i=1}^{m}\lp\nabla_{\mv}f(\mw_{(rk)}^{t-1}(i),\mv_{(rk)}^{t-1}(i);{Z}_{j_t(i)})-\nabla_{\mv}f(\mw^{t-1}(i),\mv^{t-1}(i);Z_{j_t(i)})\rp}\right\rangle}.
		\end{align*}
		We consider the inner term,
		\begin{align*}
			&\left\langle\binom{\mw^{t-1}-\mw_{(rk)}^{t-1}}{\mv^{t-1}-\mv_{(rk)}^{t-1}},\binom{\nabla_{\mw}f(\mw^{t-1}(i),\mv^{t-1}(i);Z_{j_t(i)})-\nabla_{\mw}f(\mw_{(rk)}^{t-1}(i),\mv_{(rk)}^{t-1}(i);{Z}_{j_t(i)})}{\nabla_{\mv}f(\mw_{(rk)}^{t-1}(i),\mv_{(rk)}^{t-1}(i);{Z}_{j_t(i)})-\nabla_{\mv}f(\mw^{t-1}(i),\mv^{t-1}(i);Z_{j_t(i)})}\right\rangle\\
			=&\left\langle\binom{\mw^{t-1}(i)-\mw_{(rk)}^{t-1}(i)}{\mv^{t-1}(i)-\mv_{(rk)}^{t-1}(i)},\binom{\nabla_{\mw}f(\mw^{t-1}(i),\mv^{t-1}(i);Z_{j_t(i)})-\nabla_{\mw}f(\mw_{(rk)}^{t-1}(i),\mv_{(rk)}^{t-1}(i);{Z}_{j_t(i)})}{\nabla_{\mv}f(\mw_{(rk)}^{t-1}(i),\mv_{(rk)}^{t-1}(i);{Z}_{j_t(i)})-\nabla_{\mv}f(\mw^{t-1}(i),\mv^{t-1}(i);Z_{j_t(i)})}\right\rangle\\
			&+\left\langle\binom{\mw^{t-1}-\mw^{t-1}(i)+\mw_{(rk)}^{t-1}(i)-\mw_{(rk)}^{t-1}}{\mv^{t-1}-\mv^{t-1}(i)+\mv_{(rk)}^{t-1}(i)-\mv_{(rk)}^{t-1}},\binom{\nabla_{\mw}f(\mw^{t-1}(i),\mv^{t-1}(i);Z_{j_t(i)})-\nabla_{\mw}f(\mw_{(rk)}^{t-1}(i),\mv_{(rk)}^{t-1}(i);{Z}_{j_t(i)})}{\nabla_{\mv}f(\mw_{(rk)}^{t-1}(i),\mv_{(rk)}^{t-1}(i);{Z}_{j_t(i)})-\nabla_{\mv}f(\mw^{t-1}(i),\mv^{t-1}(i);Z_{j_t(i)})}\right\rangle\\
			\geq&-2\sqrt{2}L\ls\binom{\mw^{t-1}-\mw^{t-1}(i)}{\mv^{t-1}-\mv^{t-1}(i)}\rs-2\sqrt{2}L\ls\binom{\mw_{(rk)}^{t-1}-\mw_{(rk)}^{t-1}(i)}{\mv_{(rk)}^{t-1}-\mv_{(rk)}^{t-1}(i)}\rs\\
			\geq&-4\sqrt{2}L\ls\binom{\mw^{t-1}-\mw^{t-1}(i)}{\mv^{t-1}-\mv^{t-1}(i)}\rs.
		\end{align*}
		Back to the sum term,
		\begin{align*}
			\Im\leq\frac{8\sqrt{2}L\eta_{t}}{m}\sum_{i=1}^{m}\ls\binom{\mw^{t-1}-\mw^{t-1}(i)}{\mv^{t-1}-\mv^{t-1}(i)}\rs\leq&\frac{8\sqrt{2}L\eta_{t}}{\sqrt{m}}\lp\sum_{i=1}^{m}\ls\mw^{t-1}-\mw^{t-1}(i)\rs^2+\ls\mv^{t-1}-\mv^{t-1}(i)\rs^2\rp^{1/2}\\
			\leq&16\sqrt{2}L^2\eta_{t}\sum_{q=1}^{t-1}\eta_q\lambda^{t-q-1}.
		\end{align*}
		Furtheremore, we get
		\begin{align*}
			\ls\binom{\mw^{t}-\mw_{(rk)}^t}{\mv^{t}-\mv_{(rk)}^t}\rs^2\leq\ls\binom{\mw^{t-1}-\mw_{(rk)}^{t-1}}{\mv^{t-1}-\mv_{(rk)}^{t-1}}\rs^2+8\eta_t^2L^2+16\sqrt{2}L^2\eta_{t}\sum_{q=1}^{t-1}\eta_q\lambda^{t-q-1}.
		\end{align*}
		If $j_t=k$, we have
		\begin{align*}
			\Im\leq&\frac{8\sqrt{2}L\eta_{t}}{m}\sum_{i=1}^{m}\ls\binom{\mw^{t-1}-\mw^{t-1}(i)}{\mv^{t-1}-\mv^{t-1}(i)}\rs+\frac{4\sqrt{2}L\eta_t}{m}\ls\binom{\mw^{t-1}-\mw_{(rk)}^{t-1}}{\mv^{t-1}-\mv_{(rk)}^{t-1}}\rs\\
			\leq&16\sqrt{2}L^2\eta_{t}\sum_{q=1}^{t-1}\eta_q\lambda^{t-q-1}+\frac{4\sqrt{2}L\eta_t}{m}\ls\binom{\mw^{t-1}-\mw_{(rk)}^{t-1}}{\mv^{t-1}-\mv_{(rk)}^{t-1}}\rs.
		\end{align*}
		Moreover, it then follows that
		\begin{align*}
			\ls\binom{\mw^{t}-\mw_{(rk)}^t}{\mv^{t}-\mv_{(rk)}^t}\rs^2\leq\ls\binom{\mw^{t-1}-\mw_{(rk)}^{t-1}}{\mv^{t-1}-\mv_{(rk)}^{t-1}}\rs^2+8\eta_t^2L^2+16\sqrt{2}L^2\eta_{t}\sum_{q=1}^{t-1}\eta_q\lambda^{t-q-1}+\frac{4\sqrt{2}L\eta_t}{m}\ls\binom{\mw^{t-1}-\mw_{(rk)}^{t-1}}{\mv^{t-1}-\mv_{(rk)}^{t-1}}\rs.
		\end{align*}
		Combine the above two case,
		\begin{align*}
			\ls\binom{\mw^{t}-\mw_{(rk)}^t}{\mv^{t}-\mv_{(rk)}^t}\rs^2\leq\ls\binom{\mw^{t-1}-\mw_{(rk)}^{t-1}}{\mv^{t-1}-\mv_{(rk)}^{t-1}}\rs^2+8\eta_t^2L^2+16L^2\eta_{t}\sum_{q=1}^{t-1}\eta_q\lambda^{t-q-1}+\frac{4\sqrt{2}L\eta_t}{m}\ls\binom{\mw^{t-1}-\mw_{(rk)}^{t-1}}{\mv^{t-1}-\mv_{(rk)}^{t-1}}\rs\mathbb{I}_{\lk j_t=k\rk}.
		\end{align*}
		Recursiving the above inequality, we get the following result
		\begin{align*}
			\ls\binom{\mw^{t}-\mw_{(rk)}^t}{\mv^{t}-\mv_{(rk)}^t}\rs^2\leq8L^2\sum_{s=1}^{t}\eta_s^2+16L^2\sum_{s=1}^{t}\eta_s\sum_{q=1}^{s-1}\eta_q\lambda^{s-q-1}+\frac{4\sqrt{2}L}{m}\sum_{s=1}^{t}\eta_s\ls\binom{\mw^{s-1}-\mw_{(rk)}^{s-1}}{\mv^{s-1}-\mv_{(rk)}^{s-1}}\rs\mathbb{I}_{\lk j_s=k\rk}.
		\end{align*}
		Applying to Lemma \ref{sequel}, it obtain
		\begin{align*}
			\ls\binom{\mw^{t}-\mw_{(rk)}^t}{\mv^{t}-\mv_{(rk)}^t}\rs^2\leq&\sqrt{8L^2\sum_{s=1}^{t}\eta_s^2+16L^2\sum_{s=1}^{t}\eta_s\sum_{q=1}^{s-1}\eta_q\lambda^{s-q-1}}+\frac{4\sqrt{2}L}{m}\sum_{s=1}^{t}\eta_s\mathbb{I}_{\lk j_s=k\rk}\\
			\leq&2\sqrt{2}L\sqrt{\sum_{s=1}^{t}\eta_s^2}+4L\sqrt{\sum_{s=1}^{t}\eta_s\sum_{q=1}^{s-1}\eta_q\lambda^{s-q-1}}+\frac{4\sqrt{2}L}{m}\sum_{s=1}^{t}\eta_s\mathbb{I}_{\lk j_s=k\rk}.
		\end{align*}
		Taking average about $k$ and $r$,
		\begin{align}\label{sgda-nonsmooth}
			\frac{1}{mn}\sum_{r=1}^{m}\sum_{k=1}^{n}\ls\binom{\mw^{T}-\mw_{(rk)}^T}{\mv^{t}-\mv_{(rk)}^t}\rs^2
			\leq2\sqrt{2}L\sqrt{\sum_{t=1}^{T}\eta_t^2}+4L\sqrt{\sum_{t=1}^{T}\eta_t\sum_{q=1}^{t-1}\eta_q\lambda^{t-q-1}}+\frac{4\sqrt{2}L}{mn}\sum_{t=1}^{T}\eta_t.
		\end{align}
		Taking expectation about the algorithm $\A$, we have
		\begin{align*}
			\mathbb{E}_{\A}\lk\frac{1}{mn}\sum_{r=1}^{m}\sum_{k=1}^{n}\ls\binom{\mw^{T}-\mw_{(rk)}^T}{\mv^{t}-\mv_{(rk)}^t}\rs^2\rk
			\leq&2\sqrt{2}L\sqrt{\sum_{t=1}^{T}\eta_t^2}+4L\sqrt{\sum_{t=1}^{T}\eta_t\sum_{q=1}^{t-1}\eta_q\lambda^{t-q-1}}+\frac{4\sqrt{2}L}{mn}\sum_{t=1}^{T}\eta_t.
		\end{align*}
	\end{proof}
	\subsection{Average weight (Proof of Theorem 7 and 8)}
	It is easy to know that
	\begin{align*}
		\frac{1}{mn}\sum_{r=1}^{m}\sum_{k=1}^{n}\ls\binom{\bar{\mw}^{T}-\bar{\mw}_{(rk)}^{T}}{\bar{\mv}^{T}-\bar{\mv}_{(rk)}^{T}}\rs=\frac{1}{mn}\sum_{r=1}^{m}\sum_{k=1}^{n}\ls\binom{\frac{\sum_{t=1}^{T}\eta_t\lp\mw^{t}-\mw_{(rk)}^{t}\rp}{\sum_{t=1}^{T}\eta_t}}{\frac{\sum_{t=1}^{T}\eta_t\lp\mv^{t}-\mv_{(rk)}^{t}\rp}{\sum_{t=1}^{T}\eta_t}}\rs
		\leq\frac{1}{mn}\sum_{r=1}^{m}\sum_{k=1}^{n}\frac{\sum_{t=1}^{T}\eta_t\ls\binom{\mw^{t}-\mw_{(rk)}^{t}}{\mv^{t}-\mv_{(rk)}^{t}}\rs}{\sum_{t=1}^{T}\eta_t}.
	\end{align*}
	\textbf{Smooth Case:}
	According to the eq. \eqref{sgda-smooth}, we can obtain
	\begin{align*}
		\frac{1}{mn}\sum_{r=1}^{m}\sum_{k=1}^{n}\ls\binom{\mw^{t}-\mw_{(rk)}^t}{\mv^{t}-\mv_{(rk)}^t}\rs
		\leq8\sqrt{2}\beta L\sum_{s=1}^{t}\eta_{s}\sum_{q=1}^{s-1}\eta_{q}\lambda^{s-q-1}+\frac{4\sqrt{2}L}{mn}\sum_{s=1}^{t}\eta_{s}.
	\end{align*}
	If $\eta_t=\eta$, we have
	\begin{align*}
		\mathbb{E}_{\A}\lk\frac{1}{mn}\sum_{r=1}^{m}\sum_{k=1}^{n}\ls\binom{\mw^{t}-\mw_{(rk)}^t}{\mv^{t}-\mv_{(rk)}^t}\rs\rk\leq&\frac{\eta\sum_{t=1}^{T}\lp\frac{8\sqrt{2}\eta^2\beta Lt}{1-\lambda}+\frac{4\sqrt{2}\eta Lt}{mn}\rp}{T\eta}\\
		\leq&\frac{4\sqrt{2}\eta^2\beta LT}{1-\lambda}+\frac{2\sqrt{2}\eta LT}{mn}.
	\end{align*}
	Furthermore, we get the weak PD risk bound
	\begin{align*}
		\epsilon^{\mw}_{\mathrm{gen}}\leq\frac{4\sqrt{2}\eta^2\beta L^2T}{1-\lambda}+\frac{2\sqrt{2}\eta L^2T}{mn}.
	\end{align*}
	We can also obtain the primal risk bound
	\begin{align*}
		\epsilon^{\mathrm{P}}_{\mathrm{gen}}\leq 2\sqrt{2}L^2(1+\beta/\rho)\lp\frac{2\eta^2\beta T}{1-\lambda}+\frac{\eta T}{mn}\rp.
	\end{align*}
	\textbf{Non-smooth Case:}
	According to the eq. \eqref{sgda-nonsmooth}, we derive that
	\begin{align*}
		\frac{1}{mn}\sum_{r=1}^{m}\sum_{k=1}^{n}\ls\binom{\mw^{t}-\mw_{(rk)}^t}{\mv^{t}-\mv_{(rk)}^t}\rs
		\leq2\sqrt{2}L\sqrt{\sum_{s=1}^{t}\eta_s^2}+4L\sqrt{\sum_{s=1}^{t}\eta_s\sum_{q=1}^{s-1}\eta_q\lambda^{s-q-1}}+\frac{4\sqrt{2}L}{mn}\sum_{s=1}^{t}\eta_s.
	\end{align*}
	If $\eta_t=\eta$, we have
	\begin{align*}
		\mathbb{E}_{\A}\lk\frac{1}{mn}\sum_{r=1}^{m}\sum_{k=1}^{n}\ls\binom{\mw^{t}-\mw_{(rk)}^t}{\mv^{t}-\mv_{(rk)}^t}\rs\rk\leq&\frac{\eta\sum_{t=1}^{T}\lp2\sqrt{2}L\eta\sqrt{t}+\frac{4\eta L\sqrt{t}}{\sqrt{1-\lambda}}+\frac{4\sqrt{2}\eta Lt}{mn}\rp}{T\eta}\\
		\leq&2\sqrt{2}L\eta\sqrt{T}+\frac{4\eta L\sqrt{T}}{\sqrt{1-\lambda}}+\frac{4\sqrt{2}\eta LT}{mn}.
	\end{align*}
	Furthermore, we get the weak PD risk bound
	\begin{align*}
		\epsilon^{\mw}_{\mathrm{gen}}\leq2\sqrt{2}L^2\eta\sqrt{T}+\frac{4\eta L^2\sqrt{T}}{\sqrt{1-\lambda}}+\frac{4\sqrt{2}\eta L^2T}{mn}.
	\end{align*}
	We can also obtain the primal risk bound
	\begin{align*}
		\epsilon^{\mathrm{P}}_{\mathrm{gen}}\leq 2\sqrt{2}L^2(1+\beta/\rho)\lp\eta\sqrt{T}+\eta\sqrt{\frac{2T}{1-\lambda}}+\frac{2\eta T}{mn}\rp.
	\end{align*}
	\subsection{Optimization Error of DMc-SGDA}
	\begin{theorem}[Convex Case]
		Suppose that $f(\mw,\mv;Z)$ is convex-convave, assume that Assumptions 3,5,6 hold. Let $\A$ denote the DMc-SGDA algorithm executed for $T$ iterations, generating the iterates $\{\mw^t,\mv^t\}_{t=1}^{T}$, with initialization $\mw^0=0$ and $\eta_t=\eta$. Let $D_{\mw}$ and $D_{\mv}$ be the diameter of the domains $\mathcal{W}$ and $\mathcal{V}$, and define $D=D_{\mw}+D_{\mv}$. For any $t\in[T]$, define 
		\begin{align*}
			\mathcal{T}_t=\min\left\{\max\left\{\lceil\frac{\log\lp2C_HDnt^2\rp}{\log\lp1/\lambda(H)\rp}\rceil,K_H\right\},t\right\}.
		\end{align*}
		Then, the expected primal–dual gap satisfies
		\begin{align*}
			\mathbb{E}_{\A}\lk\max_{\mv\in\mathcal{V}}R_S(\bar{\mw}^{T},\mv)-\min_{\mw\in\mathcal{W}}R_S(\mw,\bar{\mv}^{T})\rk\leq&\frac{2L(D_{\mw}+D_{\mv})K_H+\sum_{t=K_H}^{T}\frac{L}{t^2}}{T}+\frac{4(D_{\mw}^2+D_{\mv}^2)+\frac{16(D_{\mw}+D_{\mv})\beta L\eta^2T}{1-\lambda}}{2T\eta}\\
			&+\frac{12L^2\eta^2\sum_{t=1}^{T}\mathcal{T}_t+{T\eta^2}L^2}{T\eta}.
		\end{align*}
		Furthermore, If Assumption 4 also holds and the stepsize is chosen as $\eta_t=\eta=1/\sqrt{T\log T}$, then 
		\begin{align*}
			\mathbb{E}_{\A}\lk\max_{\mv\in\mathcal{V}}R_S(\bar{\mw}^{T},\mv)-\min_{\mw\in\mathcal{W}}R_S(\mw,\bar{\mv}^{T})\rk
			=\mathcal{O}\lp\frac{1}{\sqrt{T\log T}(1-\lambda)}+\frac{\sqrt{\log T}}{\sqrt{T}\log(1/\lambda(H))}\rp.
		\end{align*}
	\end{theorem}
	\begin{proof}
		We focus on the average term,
		\begin{align}
			&\mathbb{E}_{\A}\lk\max_{\mv\in\mathcal{V}} R_S(\bar{\mw}^{T},\mv)-\min_{\mw\in\mathcal{W}}R_S(\mw,\bar{\mv}^{T})\rk\nonumber\\
			\leq&\underbrace{\mathbb{E}_{\A}\lk\frac{1}{T}\sum_{t=1}^{T}R_S({\mw}^t,\mv^t)-\min_{\mw\in\mathcal{W}}R_S(\mw,\bar{\mv}^{T})\rk}_{\dagger}+\underbrace{\mathbb{E}_{\A}\lk\max_{\mv\in\mathcal{V}} R_S(\bar{\mw}^{T},\mv)-\frac{1}{T}\sum_{t=1}^{T}R_S({\mw}^t,\mv^t)\rk}_{\ddagger}.
		\end{align}
		Let us begin to consider $\dagger$ term, and we abbrev that $\mw_{S}^{*,\mathrm{DMc-SGDA}}=\mw_{S}^{*,\mathrm{G}}=\arg\min_{\mw\in\mathcal{W}} R_S(\mw,\bar{\mw}^{T})$
		\begin{align*}
			&\mathbb{E}_{\A}\lk\frac{1}{T}\sum_{t=1}^{T}R_S({\mw}^t,\mv^t)-R_S(\mw_{S}^{*,\mathrm{G}},\bar{\mv}^{T})\rk\\
			\leq&\mathbb{E}_{\A}\lk\frac{1}{T}\sum_{t=1}^{T}\lp R_S({\mw}^t,\mv^t)-R_S(\mw_{S}^{*,\mathrm{G}},{\mv}^{t})\rp\rk\\
			\leq&\mathbb{E}_{\A}\lk\frac{1}{T}\sum_{t=1}^{T}\lp R_S({\mw}^t,\mv^t)-R_S(\mw^{t-\mathcal{T}_t},{\mv}^{t})\rp\rk+\mathbb{E}_{\A}\lk\frac{1}{T}\sum_{t=1}^{T}\lp R_S({\mw}^{t-\mathcal{T}_t},\mv^t)-R_S(\mw^{t-\mathcal{T}_t},{\mv}^{t-\mathcal{T}_t})\rp\rk\\
			&+\mathbb{E}_{\A}\lk\frac{1}{T}\sum_{t=1}^{T}\lp R_S({\mw}^{t-\mathcal{T}_t},\mv^{t-\mathcal{T}_t})-R_S(\mw_{S}^{*,\mathrm{G}},{\mv}^{t-\mathcal{T}_t})\rp\rk+\mathbb{E}_{\A}\lk\frac{1}{T}\sum_{t=1}^{T}\lp R_S(\mw_{S}^{*,\mathrm{G}},{\mv}^{t-\mathcal{T}_t})-R_S(\mw_{S}^{*,\mathrm{G}},{\mv}^{t})\rp\rk\\
			\leq&\frac{3L^2\eta}{T}\sum_{t=1}^{T}\mathcal{T}_t+\mathbb{E}_{\A}\lk\frac{1}{T}\sum_{t=1}^{T}\lp R_S({\mw}^{t-\mathcal{T}_t},\mv^{t-\mathcal{T}_t})-R_S(\mw_{S}^{*,\mathrm{G}},{\mv}^{t-\mathcal{T}_t})\rp\rk.
		\end{align*}
		Similar to the proof of DMc-SGD, we can estimate that
		\begin{align*}
			&\mathbb{E}_{j_t}\lk\frac{1}{m}\sum_{r=1}^{m}\lp f\lp\mw^{t-\mathcal{T}_t},\mv^{t-\mathcal{T}_t};Z_{j_{t}(r)}\rp-f\lp\mw_S^{*,\mathrm{G}},\mv^{t-\mathcal{T}_t};Z_{j_{t}(r)}\rp\rp\mid(\mw^0,\mv^0),\cdots,(\mw^{t-\mathcal{T}_t},\mv^{t-\mathcal{T}_t}),Z_{j_{1}(r)},\cdots,Z_{j_{t-\mathcal{T}_t}(r)}\rk\\
			=&\frac{1}{m}\sum_{k=1}^{n}\lk\sum_{r=1}^{m}\lp f\lp\mw^{t-\mathcal{T}_t},\mv^{t-\mathcal{T}_t};Z_{j_{t}(r)}\rp-f\lp\mw_S^{*,\mathrm{G}},\mv^{t-\mathcal{T}_t};Z_{j_{t}(r)}\rp\rp\Pr\lp j_{t}(r)=k\mid j_{t-\mathcal{T}_t}(r)=k\rp\rk\\
			=&\frac{1}{m}\sum_{k=1}^{n}\lk\sum_{r=1}^{m}\lp f\lp\mw^{t-\mathcal{T}_t},\mv^{t-\mathcal{T}_t};Z_{j_{t}(r)}\rp-f\lp\mw_S^{*,\mathrm{G}},\mv^{t-\mathcal{T}_t};Z_{j_{t}(r)}\rp\rp\lk H^{\mathcal{T}_t}\rk_{j_{t-\mathcal{T}_t}(r),k}\rk\\
			=&\frac{1}{m}\sum_{k=1}^{n}\lk\sum_{r=1}^{m}\lp f\lp\mw^{t-\mathcal{T}_t},\mv^{t-\mathcal{T}_t};Z_{j_{t}(r)}\rp-f\lp\mw_S^{*,\mathrm{G}},\mv^{t-\mathcal{T}_t};Z_{j_{t}(r)}\rp\rp\lp\lk H^{\mathcal{T}_t}\rk_{j_{t-\mathcal{T}_t}(r),k}-\frac{1}{n}\rp\rk\\
			&+\frac{1}{mn}\sum_{k=1}^{n}\lk\sum_{r=1}^{m}\lp f\lp\mw^{t-\mathcal{T}_t},\mv^{t-\mathcal{T}_t};Z_{j_{t}(r)}\rp-f\lp\mw_S^{*,\mathrm{G}},\mv^{t-\mathcal{T}_t};Z_{j_{t}(r)}\rp\rp\rk\\
			=&\frac{1}{m}\sum_{k=1}^{n}\lk\sum_{r=1}^{m}\lp f\lp\mw^{t-\mathcal{T}_t},\mv^{t-\mathcal{T}_t};Z_{j_{t}(r)}\rp-f\lp\mw_S^{*,\mathrm{G}},\mv^{t-\mathcal{T}_t};Z_{j_{t}(r)}\rp\rp\lp\lk H^{\mathcal{T}_t}\rk_{j_{t-\mathcal{T}_t}(r),k}-\frac{1}{n}\rp\rk\\
			&+R_S\lp \mw^{t-\mathcal{T}_t},\mv^{t-\mathcal{T}_t}\rp-R_S(\mw_S^{*,\mathrm{G}},\mv^{t-\mathcal{T}_t}).
		\end{align*}
		Rearranging the above equality,
		\begin{align*}
			&\mathbb{E}_{\A}\lk R_S\lp \mw^{t-\mathcal{T}_t},\mv^{t-\mathcal{T}_t}\rp-R_S(\mw_S^{*,\mathrm{G}},\mv^{t-\mathcal{T}_t})\rk\\
			=&\mathbb{E}_{\A}\lk\frac{1}{m}\sum_{r=1}^{m}\lp f\lp\mw^{t-\mathcal{T}_t},\mv^{t-\mathcal{T}_t};Z_{j_{t}(r)}\rp-f\lp\mw_S^{*,\mathrm{G}},\mv^{t-\mathcal{T}_t};Z_{j_{t}(r)}\rp\rp\rk\\
			&+\mathbb{E}_{\A}\lk\frac{1}{m}\sum_{k=1}^{n}\lk\sum_{r=1}^{m}\lp f\lp\mw^{t-\mathcal{T}_t},\mv^{t-\mathcal{T}_t};Z_{j_{t}(r)}\rp-f\lp\mw_S^{*,\mathrm{G}},\mv^{t-\mathcal{T}_t};Z_{j_{t}(r)}\rp\rp\lp\frac{1}{n}-\lk H^{\mathcal{T}_t}\rk_{j_{t-\mathcal{T}_t}(r),k}\rp\rk\rk.
		\end{align*}
		Summing the above inequality over $t$ gives
		\begin{align*}
			&\sum_{t=1}^{T}\mathbb{E}_{\A}\lk R_S\lp \mw^{t-\mathcal{T}_t},\mv^{t-\mathcal{T}_t}\rp-R_S(\mw_S^{*,\mathrm{G}},\mv^{t-\mathcal{T}_t})\rk\\
			=&\underbrace{\sum_{t=1}^{T}\mathbb{E}_{\A}\lk\frac{1}{m}\sum_{r=1}^{m}\lp f\lp\mw^{t-\mathcal{T}_t},\mv^{t-\mathcal{T}_t};Z_{j_{t}(r)}\rp-f\lp\mw_S^{*,\mathrm{G}},\mv^{t-\mathcal{T}_t};Z_{j_{t}(r)}\rp\rp\rk}_{\Im}\\
			&+\underbrace{\sum_{t=1}^{T}\mathbb{E}_{\A}\lk\frac{1}{m}\sum_{k=1}^{n}\lk\sum_{r=1}^{m}\lp f\lp\mw^{t-\mathcal{T}_t},\mv^{t-\mathcal{T}_t};Z_{j_{t}(r)}\rp-f\lp\mw_S^{*,\mathrm{G}},\mv^{t-\mathcal{T}_t};Z_{j_{t}(r)}\rp\rp\lp\frac{1}{n}-\lk H^{\mathcal{T}_t}\rk_{j_{t-\mathcal{T}_t}(r),k}\rp\rk\rk}_{\wp}.
		\end{align*}
		Estimate the $\Im$ term,
		\begin{align*}
			&\ls\mw^t-\mw_S^{*,\mathrm{G}}\rs^2\\
			\leq&\ls\mw^{t-1}-\frac{\eta}{m}\sum_{i=1}^{m}\nabla_{\mw} f(\mw^{t-1}(i),\mv^{t-1}(i);Z_{j_{t}(i)})-\mw_S^{*,\mathrm{G}}\rs^2\\
			\leq&\ls\mw^{t-1}-\mw_S^{*,\mathrm{G}}\rs^2-\frac{2\eta}{m}\sum_{i=1}^{m}\left\langle\mw^{t-1}-\mw_S^{*,\mathrm{G}},\nabla f(\mw^{t-1}(i),\mv^{t-1}(i);Z_{j_{t}(i)})\right\rangle+\frac{\eta^2}{m^2}\ls\sum_{i=1}^{m}\nabla f(\mw^{t-1}(i),\mv^{t-1}(i);Z_{j_{t}(i)})\rs^2\\
			\leq&\ls\mw^{t-1}-\mw_S^{*,\mathrm{G}}\rs^2-\frac{2\eta}{m}\sum_{i=1}^{m}\left\langle\mw^{t-1}-\mw_S^{*,\mathrm{G}},\nabla f(\mw^{t-1}(i),\mv^{t-1}(i);Z_{j_{t}(i)})-\nabla f\lp\mw^{t-1},\mv^{t-1}(i);Z_{j_{t}(i)}\rp\right\rangle\\
			&-\frac{2\eta}{m}\sum_{i=1}^{m}\left\langle\mw^{t-1}-\mw_S^{*,\mathrm{G}},\nabla f\lp\mw^{t-1},\mv^{t-1}(i);Z_{j_{t}(i)}\rp-\nabla f\lp\mw^{t-1},\mv^{t-1};Z_{j_{t}(i)}\rp\right\rangle+{\eta^2L^2}\\
			&-\frac{2\eta}{m}\sum_{i=1}^{m}\left\langle\mw^{t-1}-\mw_S^{*,\mathrm{G}},\nabla f\lp\mw^{t-1},\mv^{t-1};Z_{j_{t}(i)}\rp\right\rangle\\
			\leq&\ls\mw^{t-1}-\mw_S^{*,\mathrm{G}}\rs^2+\frac{4D_{\mw}\beta\eta}{m}\sum_{i=1}^{m}\ls\mw^{t-1}(i)-\mw^{t-1}\rs+\frac{4D_{\mw}\beta\eta}{m}\sum_{i=1}^{m}\ls\mv^{t-1}(i)-\mv^{t-1}\rs\\
			&-\frac{2\eta}{m}\sum_{i=1}^{m}\lp f\lp \mw^{t-1},\mv^{t-1};Z_{j_{t}(i)}\rp-f\lp \mw_S^{*,\mathrm{G}},\mv^{t-1};Z_{j_{t}(i)}\rp\rp
			+{\eta^2}L^2\\
			\leq&\ls\mw^{t-1}-\mw_S^{*,\mathrm{G}}\rs^2+\frac{4D_{\mw}\beta\eta}{\sqrt{m}}\sum_{i=1}^{m}\lk\ls\mw^{t-1}(i)-\mw^{t-1}\rs^2\rk^{\frac{1}{2}}+\frac{4D_{\mw}\beta\eta}{\sqrt{m}}\sum_{i=1}^{m}\lk\ls\mv^{t-1}(i)-\mv^{t-1}\rs^2\rk^{\frac{1}{2}}\\
			&-\frac{2\eta}{m}\sum_{i=1}^{m}\lp f\lp \mw^{t-\mathcal{T}_t},\mv^{t-\mathcal{T}_t};Z_{j_{t}(i)}\rp-f\lp \mw_S^{*,\mathrm{G}},\mv^{t-\mathcal{T}_t};Z_{j_{t}(i)}\rp\rp+\frac{2\eta}{m}\sum_{i=1}^{m}\lp f\lp \mw^{t-\mathcal{T}_t},\mv^{t-\mathcal{T}_t};Z_{j_{t}(i)}\rp-f\lp \mw^{t-\mathcal{T}_t},\mv^{t-1};Z_{j_{t}(i)}\rp\rp\\
			&+\frac{2\eta}{m}\sum_{i=1}^{m}\lp f\lp \mw^{t-\mathcal{T}_t},\mv^{t-1};Z_{j_{t}(i)}\rp-f\lp \mw^{t-1},\mv^{t-1};Z_{j_{t}(i)}\rp\rp+\frac{2\eta}{m}\sum_{i=1}^{m}\lp f\lp \mw_S^{*,\mathrm{G}},\mv^{t-1};Z_{j_{t}(i)}\rp-f\lp \mw_S^{*,\mathrm{G}},\mv^{t-\mathcal{T}_t};Z_{j_{t}(i)}\rp\rp\\
			&+{\eta^2}L^2\\
			\leq&\ls\mw^{t-1}-\mw_S^{*,\mathrm{G}}\rs^2+\frac{16D_{\mw}\beta L\eta^2}{1-\lambda}-\frac{2\eta}{m}\sum_{i=1}^{m}\lp f\lp \mw^{t-\mathcal{T}_t},\mv^{t-\mathcal{T}_t};Z_{j_{t}(i)}\rp-f\lp \mw_S^{*,\mathrm{G}},\mv^{t-\mathcal{T}_t};Z_{j_{t}(i)}\rp\rp\\
			&+{2L\eta}\ls\mw^{t-\mathcal{T}_t}-\mw^{t-1}\rs+{\eta^2}L^2\\
			\leq&\ls\mw^{t-1}-\mw_S^{*,\mathrm{G}}\rs^2+\frac{16D_{\mw}\beta L\eta^2}{1-\lambda}-\frac{2\eta}{m}\sum_{i=1}^{m}\lp f\lp \mw^{t-\mathcal{T}_t},\mv^{t-\mathcal{T}_t};Z_{j_{t}(i)}\rp-f\lp \mw_S^{*,\mathrm{G}},\mv^{t-\mathcal{T}_t};Z_{j_{t}(i)}\rp\rp
			+6L^2\eta^2\mathcal{T}_t+{\eta^2}L^2.
		\end{align*}
		Taking a summation of the both sides over $t$, we can obtain that
		\begin{align*}
			\frac{2\eta}{m}\sum_{t=1}^{T}\lp f\lp \mw^{t-\mathcal{T}_t},\mv^{t-\mathcal{T}_t};Z_{j_{t}(i)}\rp-f\lp \mw_S^{*,\mathrm{G}},\mv^{t-\mathcal{T}_t};Z_{j_{t}(i)}\rp\rp\leq\frac{4D_{\mw}^2+\frac{16D_{\mw}\beta L\eta^2T}{1-\lambda}+6L^2\eta^2\sum_{t=1}^{T}\mathcal{T}_t+T\eta^2L^2}{2\eta}.
		\end{align*}
		Back to the matrix difference term, according to the lemma 5,
		\begin{align*}
			\Bigg|\frac{1}{n}-\lk H^{\mathcal{T}_t}\rk_{k,k^{\prime}}\Bigg|\leq\frac{1}{2(D_w+D_v)nt^2}.
		\end{align*}
		Furthermore, we get
		\begin{align*}
			&\sum_{t=K_p}^{T}\lk\frac{1}{m}\sum_{k=1}^{n}\lk\sum_{r=1}^{m}\lp f\lp\mw^{t-\mathcal{T}_t},\mv^{t-\mathcal{T}_t};Z_{j_{t}(r)}\rp-f\lp\mw_S^{*,\mathrm{G}},\mv^{t-\mathcal{T}_t};Z_{j_{t}(r)}\rp\rp\lp\frac{1}{n}-\lk H^{\mathcal{T}_t}\rk_{j_{t-\mathcal{T}_t}(r),k}\rp\rk\rk\\
			\leq&LD_{\mw}\sum_{t=K_{H}}^{T}\sum_{k=1}^{n}\lp\frac{1}{n}-\lk H^{\mathcal{T}_t}\rk_{j_{t-\mathcal{T}_t}(r),k}\rp
			\leq\sum_{t=K_{H}}^{T}\frac{L}{2t^2}.
		\end{align*}
		In addition, we obtain
		\begin{align*}
			\sum_{t=1}^{K_H-1}\lk\frac{1}{m}\sum_{k=1}^{n}\lk\sum_{r=1}^{m}\lp f\lp\mw^{t-\mathcal{T}_t},\mv^{t-\mathcal{T}_t};Z_{j_{t}(r)}\rp-f\lp\mw_S^{*,\mathrm{G}},\mv^{t-\mathcal{T}_t};Z_{j_{t}(r)}\rp\rp\lp\frac{1}{n}-\lk H^{\mathcal{T}_t}\rk_{j_{t-\mathcal{T}_t}(r),k}\rp\rk\rk
			\leq2LD_{\mw}K_H.
		\end{align*}
		Combine the above inequality, we have
		\begin{align*}
			&\sum_{t=1}^{T}\lk\frac{1}{m}\sum_{k=1}^{n}\lk\sum_{r=1}^{m}\lp f\lp\mw^{t-\mathcal{T}_t},\mv^{t-\mathcal{T}_t};Z_{j_{t}(r)}\rp-f\lp\mw_S^{*,\mathrm{G}},\mv^{t-\mathcal{T}_t};Z_{j_{t}(r)}\rp\rp\lp\frac{1}{n}-\lk H^{\mathcal{T}_t}\rk_{j_{t-\mathcal{T}_t}(r),k}\rp\rk\rk\\
			\leq&2LD_{\mw}K_H+\sum_{t=K_H}^{T}\frac{L}{2t^2}.
		\end{align*}
		Back to the initial term,
		\begin{align*}
			\mathbb{E}_{\A}\lk\frac{1}{T}\sum_{t=1}^{T}R_S({\mw}^t,\mv^t)-\min_{\mw\in\mathcal{W}}R_S(\mw,\bar{\mv}^{T})\rk
			\leq\frac{2LD_{\mw}K_H+\sum_{t=K_H}^{T}\frac{L}{2t^2}}{T}+\frac{4D_{\mw}^2+\frac{16D_{\mw}\beta L\eta^2T}{1-\lambda}+12L^2\eta^2\sum_{t=1}^{T}\mathcal{T}_t+T{\eta^2}L^2}{2T\eta}.
		\end{align*}
		In a similar way, we show that
		\begin{align*}
			\mathbb{E}_{\A}\lk\max_{\mv\in\mathcal{V}}R_S(\bar{\mw}^{T},\mv)-\frac{1}{T}\sum_{t=1}^{T}R_S({\mw}^t,\mv^t)\rk
			\leq\frac{2LD_{\mv}K_H+\sum_{t=K_H}^{T}\frac{L}{2t^2}}{T}+\frac{4D_{\mv}^2+\frac{16D_{\mv}\beta L\eta^2T}{1-\lambda}+12L^2\eta^2\sum_{t=1}^{T}\mathcal{T}_t+T{\eta^2}L^2}{2T\eta}.
		\end{align*}
		Combine the above two inequality, 
		\begin{align*}
			\mathbb{E}_{\A}\lk\max_{\mv\in\mathcal{V}}R_S(\bar{\mw}^{T},\mv)-\min_{\mw\in\mathcal{W}}R_S(\mw,\bar{\mv}^{T})\rk\leq&\frac{2L(D_{\mw}+D_{\mv})K_H+\sum_{t=K_H}^{T}\frac{L}{t^2}}{T}+\frac{4(D_{\mw}^2+D_{\mv}^2)+\frac{16(D_{\mw}+D_{\mv})\beta L\eta^2T}{1-\lambda}}{2T\eta}\\
			&+\frac{12L^2\eta^2\sum_{t=1}^{T}\mathcal{T}_t+{T\eta^2}L^2}{T\eta}.
		\end{align*}
		Similiar to DMc-SGD case, let $J=\frac{1}{\sqrt{2C_HDn\lambda(H)^{K_H}}}$, if $t\leq J$, we can estimate
		\begin{align*}
			\sum_{t=1}^{J}\mathcal{T}_t\eta^2\leq JK_H\eta^2=\frac{K_H}{T\log T\sqrt{2C_HDn\lambda(H)^{K_H}}}.
		\end{align*}
		Furthermore, we have
		\begin{align*}
			\sum_{t=1}^{J}\mathcal{T}_t\eta^2+\sum_{t=J+1}^{T}\mathcal{T}_t\eta^2=\mathcal{O}\lp\frac{K_H}{T\log T\sqrt{2C_Hn\lambda(H)^{K_H}}}+\frac{1}{\log(1/\lambda(H))}\rp.
		\end{align*}
		Then we get
		\begin{align*}
			&\mathbb{E}_{\A}\lk\max_{\mv\in\mathcal{V}}R_S(\bar{\mw}^{T},\mv)-\min_{\mw\in\mathcal{W}}R_S(\mw,\bar{\mv}^{T})\rk\\
			=&\mathcal{O}\lp\frac{K_H}{T}+\frac{\eta}{(1-\lambda)}+\frac{1+\sum_{t=1}^{T}\mathcal{T}_t\eta^2}{T\eta}+\eta\rp\\
			=&\mathcal{O}\lp\frac{K_H}{T}+\frac{1}{\sqrt{T\log T}(1-\lambda)}+\frac{\sqrt{\log T}}{\sqrt{T}\log(1/\lambda(H))}+\frac{K_H}{T^{3/2}\sqrt{\log T}\sqrt{2C_Hn\lambda(H)^{K_H}}}\rp.
		\end{align*}
		If $K_H=0$, we have
		\begin{align*}
			\mathbb{E}_{\A}\lk\max_{\mv\in\mathcal{V}}R_S(\bar{\mw}^{T},\mv)-\min_{\mw\in\mathcal{W}}R_S(\mw,\bar{\mv}^{T})\rk
			=\mathcal{O}\lp\frac{1}{\sqrt{T\log T}(1-\lambda)}+\frac{\sqrt{\log T}}{\sqrt{T}\log(1/\lambda(H))}\rp.
		\end{align*}
	\end{proof}
	\subsection{PD Population Risk}
	\begin{proof}
		When all step sizes are equal, we have
		\begin{align*}
			\Delta^{\mw}(\bar{\mw}^T,\bar{\mv}^T)-\Delta_{\mathrm{emp}}^{\mw}(\bar{\mw}^T,\bar{\mv}^T)\leq\frac{4\sqrt{2}\eta^2\beta LT}{1-\lambda}+\frac{2\sqrt{2}\eta LT}{mn}.
		\end{align*} 
		Combining with the PD empirical risk (Theorem 16), we obtain that
		\begin{align*}
			\Delta^{\mw}(\bar{\mw}^T,\bar{\mv}^T)=&\Delta^{\mw}(\bar{\mw}^T,\bar{\mv}^T)-\Delta_{\mathrm{emp}}^{\mw}(\bar{\mw}^T,\bar{\mv}^T)+\Delta_{\mathrm{emp}}^{\mw}(\bar{\mw}^T,\bar{\mv}^T)\\
			\leq&\frac{4\sqrt{2}\eta^2\beta LT}{1-\lambda}+\frac{2\sqrt{2}\eta LT}{mn}+\frac{2L(D_{\mw}+D_{\mv})K_H+\sum_{t=K_H}^{T}\frac{L}{t^2}}{T}+\frac{4(D_{\mw}^2+D_{\mv}^2)+\frac{16(D_{\mw}+D_{\mv})\beta L\eta^2T}{1-\lambda}}{2T\eta}\\
			&+\frac{12L^2\eta^2\sum_{t=1}^{T}\mathcal{T}_t+{T\eta^2}L^2}{T\eta}.
		\end{align*}
		If we choose $T=mn$, $\eta=\frac{1}{\sqrt{T\log T}}$ and $K_H=0$, we get
		\begin{align*}
			\Delta^{\mw}(\bar{\mw}^T,\bar{\mv}^T)=\mathcal{O}\lp\frac{1}{(1-\lambda)\log (mn)}+\frac{1}{\sqrt{mn\log (mn)}(1-\lambda)}+\frac{\sqrt{\log (mn)}}{\sqrt{mn}\log(1/\lambda(H))}\rp.
		\end{align*}
	\end{proof}

	\subsection{Excess Primal Population Risk}
	\begin{theorem}[Excess primal population risk]
	   Suppose that Assumptions 3-6 hold. Assume for all $Z$, the function $(\mathbf{w}, \mathbf{v}) \mapsto f(\mathbf{w}, \mathbf{v} ; Z)$ is convex-concave. Assume $\mathbf{v} \mapsto R(\mathbf{w}, \mathbf{v})$ is $\rho$-strongly-concave. Let $\left\{\mathbf{w}_t, \mathbf{v}_t\right\}_{t=1}^T$ be produced by DMc-SGDA with $\eta_t \equiv \eta$. Let $\mathcal{A}$ be defined by $\mathcal{A}_{\mathbf{w}}(S)=\overline{\mathbf{w}}_T$ and $\mathcal{A}_{\mathbf{v}}(S)=\overline{\mathbf{v}}_T$. If we choose $T=mn, \eta=(T\log(T))^{-1/2}$, then
		\begin{align*}
		    \mathbb{E}_{S, \mathcal{A}}\left[F\left(\overline{\mathbf{w}}_T\right)\right]-\min_{\mathbf{w} \in \mathcal{W}} F(\mathbf{w})
			=\mathcal{O}({1}/{\sqrt{mn\log (mn)}(1-\lambda)}+\sqrt{\log (mn)} /(\sqrt{mn} \log (1 / \lambda(P)))) .
	    \end{align*}
	\end{theorem}
	\begin{proof}
		The following decomposition is applied:
		\begin{align*}
			\mathbb{E}\lk F(\bar{\mw}^T)-F(\mw^{*,\mathrm{G}})\rk=&\mathbb{E}\lk F(\bar{\mw}^T)-F_S(\bar{\mw}^T)\rk+\mathbb{E}\lk R_S(\bar{\mw}^T)-F_S(\mw^{*,\mathrm{G}},\bar{\mv}^T)\rk\\
			&+\mathbb{E}\lk R_S(\mw^{*,\mathrm{G}},\bar{\mv}^T)-R(\mw^{*,\mathrm{G}},\bar{\mv}^T)\rk
			+\mathbb{E}\lk R(\mw^{*,\mathrm{G}},\bar{\mv}^T)-F(\mw^{*,\mathrm{G}})\rk.
		\end{align*}
		Due to $R(\mw^{*,\mathrm{G}},\bar{\mv}^T)\leq F(\mw^{*,\mathrm{G}})$, we only need to consider the first three items. For the first term, we can obtain
		\begin{align*}
			\mathbb{E}\lk F(\bar{\mw}^T)-F_S(\bar{\mw}^T)\rk\leq2\sqrt{2}L^2(1+\beta/\rho)\left(\frac{2\eta^2\beta T}{1-\lambda}+\frac{\eta T}{mn}\right).
		\end{align*}
		The third term, similar to the first, forms a duality and yields a similar result
		\begin{align*}
			\mathbb{E}\lk R_S(\mw^{*,\mathrm{G}},\bar{\mv}^T)-R(\mw^{*,\mathrm{G}},\bar{\mv}^T)\rk\leq2\sqrt{2}L^2(1+\beta/\rho)\left(\frac{2\eta^2\beta T}{1-\lambda}+\frac{\eta T}{mn}\right).
		\end{align*}
		For the second term, it is easy to verify that
		\begin{align*}
			\mathbb{E}\lk R_S(\bar{\mw}^T)-F_S(\mw^{*,\mathrm{G}},\bar{\mv}^T)\rk\leq&\frac{2L(D_{\mw}+D_{\mv})K_H+\sum_{t=K_H}^{T}\frac{L}{t^2}}{T}+\frac{4(D_{\mw}^2+D_{\mv}^2)+\frac{16(D_{\mw}+D_{\mv})\beta L\eta^2T}{1-\lambda}}{2T\eta}\\
			&+\frac{12L^2\eta^2\sum_{t=1}^{T}\mathcal{T}_t+{T\eta^2}L^2}{T\eta}.
		\end{align*}
		If we choose $T=mn$, $\eta=\frac{1}{\sqrt{T\log T}}$ and $K_H=0$, we get
		\begin{align*}
			\mathbb{E}\lk F(\bar{\mw}^T)-F(\mw^{*,\mathrm{G}})\rk=\mathcal{O}\lp\frac{(1+\beta/\rho)}{(1-\lambda)\log (mn)}+\frac{1}{\sqrt{mn\log (mn)}(1-\lambda)}+\frac{\sqrt{\log (mn)}}{\sqrt{mn}\log(1/\lambda(H))}\rp.
		\end{align*}
	\end{proof}
	\section{Additional discussions}
	This section provides several clarifications and minor corrections to existing results that are closely related to our analysis. We emphasize that the issues discussed below do not affect the qualitative convergence behavior or the main conclusions of the corresponding works, but addressing them helps ensure technical correctness and consistency.
	
	\noindent\textbf{1. Clarifications on \cite{wang-markov-2022}}
	\begin{itemize}
		\item Page 16–17.
		The inequality stated as “by the convexity of $\ls\cdot\rs$" cannot be directly justified from convexity alone. A valid argument requires a more careful decomposition similar to the technique developed in Section~C.2 of this paper, where the bound is derived via an explicit expansion and regrouping of terms rather than a direct convexity claim. Importantly, correcting this step does not change the resulting convergence order.
		\item Page 26.
		After substituting the stepsize, the term $T\eta/n$ does not simplify to $\sqrt{T}/n$. While this algebraic inconsistency affects intermediate expressions, the final convergence rate remains unchanged once the correct substitution is applied.
		\item Page 27.
		When substituting back into Eq. (B.25), the correct expression should be
		$
			\log T/T^{3/2}\eta\log (1/\lambda(P))=2\log n/n^{3/2}\log (1/\lambda(P))
		$
		rather than a rate scaling with $\sqrt{n}$.
		\item Page 36. The term $\Delta^{\mw}$ should scale as $\Delta^{\mw}=\frac{\sqrt{\log n}}{\sqrt{n}\log (1/\lambda(P))}$.
		\item Page 37. The factor $L/\rho$ does not influence the term as stated. After proper simplification, the correct convergence rate should again scale as
		$\frac{\sqrt{\log n}}{\sqrt{n}\log (1/\lambda(P))}$.
	\end{itemize}
	\noindent\textbf{2. Clarifications on \cite{zhu2024stability}}
	\begin{itemize}
		\item Page 15 (Lemma 1, Case b).
		The proof implicitly assumes that “$x$ strongly convex + 
		$y$ strongly concave” implies strong monotonicity of the joint operator in $(x,y)$. This implication does not generally hold. A correct argument should instead follow the operator-theoretic approach used in \cite{farnia2021train,lei2021stability}, which combines strong monotonicity with Lipschitz continuity. Due to the absence of a clean non-expansive property, this approach typically requires a more restrictive stepsize condition, leading nonetheless to a comparable convergence rate.
		\item Stability notion in the decentralized setting.
		The stability analysis in \cite{zhu2024stability} considers perturbations where each worker contains an anomalous sample. This differs from the standard definition of stability in decentralized learning, where a single global sample is perturbed across the distributed dataset \cite{sun2021stability,zhu2022topology,bars2023improved,zeng2025stability}. The latter notion is more consistent with classical algorithmic stability and is adopted in the present work.
	\end{itemize}
	
\end{document}